\documentclass[twocolumn]{article}

\usepackage{url}
\usepackage[hidelinks]{hyperref}
\usepackage[utf8]{inputenc}
\usepackage[small]{caption}
\usepackage{graphicx}
\usepackage{amsmath}
\usepackage{amsthm}
\usepackage{amsfonts}
\usepackage{amssymb}
\usepackage{booktabs}
\usepackage{cleveref}
\usepackage[ruled,vlined]{algorithm2e}
\usepackage[switch]{lineno}
\usepackage[square,semicolon]{natbib}
\usepackage{subcaption}
\usepackage{parskip}

\theoremstyle{plain}
\newtheorem{theorem}{Theorem}[section]

\newtheorem{lemma}[theorem]{Lemma}
\newtheorem{corollary}[theorem]{Corollary}

\theoremstyle{definition}
\newtheorem{definition}[theorem]{Definition}

\theoremstyle{remark}
\newtheorem{remark}[theorem]{Remark}

\newcommand{\R}{\mathbb{R}}
\newcommand{\N}{\mathbb{N}}
\newcommand{\Prob}{\mathbb{P}}
\newcommand{\E}{\mathbb{E}}

\newcommand{\UQ}{\mathcal{U}_\beta}
\newcommand{\Bhat}{\beta}
\newcommand{\bern}{\mathcal{B}}

\newcommand{\muvtrue}{\boldsymbol\mu^{\rm true}}
\newcommand{\muvhat}[1]{\widehat{\boldsymbol\mu}^{(#1)}}
\newcommand{\muvtil}[1]{\widetilde{\boldsymbol\mu}^{(#1)}}
\newcommand{\Pareto}{\operatorname{Pareto}}
\newcommand{\Unif}{\operatorname{Unif}}
\newcommand{\Err}{\mathsf{Err}}
\newcommand{\ttpTrueFront}{\mathcal S^\star}

\newcommand{\ttpEmpFront}[1]{\widehat{\mathcal S}^{\star}_{#1}}
\newcommand{\ttpPosteriorFront}[2]{\widetilde{\mathcal S}_{#1}^{(#2)}}
\newcommand{\ttpHistory}[1]{\mathcal F_{#1}}

\newcommand{\ttpErrEvent}[1]{\mathsf{Err}_{#1}}
\newcommand{\ttpErrProb}[2]{e_{#1}(#2)}
\newcommand{\ttpStackMean}[3]{\overline X_{#1,#2}^{\,#3}}
\newcommand{\ttpGapMin}{\mathrm m}
\newcommand{\ttpGapMax}{\mathrm M}
\newcommand{\ttpInd}{\mathbf 1}
\newcommand{\ttpNormInf}[1]{\left\lVert#1\right\rVert_\infty}

\let\cite\citep

\title{Bayesian Anytime Pareto Set Identification for\\ Multi-Objective Multi-Armed Bandits}
\author{%
  \parbox{\linewidth}{\centering
    Lennert Saerens$^{1,2,}$\footnote{Corresponding author: Lennert.Saerens@vub.be}, Bram Silue$^{1}$, Eleni Litsa$^{2}$, Peter Vrancx$^{1,2}$ and Pieter Libin$^{1,3}$ \\[0.4em]
    $^1$Artificial Intelligence Lab, Departement of Computer Science, Vrije Universiteit Brussel, Brussels, Belgium \quad $^2$imec, Leuven, Belgium \quad $^3$Data Science Institute, Interuniversity Institute of Biostatistics and Statistical Bioinformatics, UHasselt, Hasselt, Belgium
  }%
}
\date{}

\begin{document}

\maketitle

\begin{abstract}
    Identifying Pareto optimal solutions is critical to support multi-objective decision-making. In this work, we present \textit{Top-Two Pareto Front Thompson Sampling} (TTPFTS), the first \textit{anytime} algorithm for Pareto Set Identification in Multi-Objective Multi-Armed Bandits (MOMABs), taking a Bayesian approach. We benchmark TTPFTS against state-of-the-art fixed-confidence and fixed-budget Pareto Set Identification algorithms on synthetic environments. Next, we demonstrate its practical utility in a challenging real-world multi-objective multi-agent molecular discovery setting by efficiently exploring an ultra-large synthesis-on-demand molecular library, outperforming industry-standard and state-of-the-art baselines. Furthermore, we introduce a novel uncertainty quantification metric that estimates our algorithm’s confidence in the predicted Pareto set. We demonstrate that this metric effectively proxies true performance, yielding a robust methodology for monitoring learning progress in complex settings. Finally, we complement these empirical findings with the first finite-time theoretical analysis of an anytime MOMAB Pareto Set Identification algorithm.
\end{abstract}

\section{Introduction}
The \textit{Multi-Armed Bandit} (MAB) problem has been widely examined in the literature, primarily as a single-objective stochastic optimization task. In this setting, an agent sequentially draws samples from a collection of $K$ unknown probability distributions, referred to as arms. Two prevalent problem variants are either to maximize the cumulative rewards, commonly framed as \textit{regret minimization}~\cite{Auer2002FinitetimeAO,lattimore2020bandit}, or to identify the arm with the highest expected value, known as \textit{best-arm identification}~\cite{Audibert2010BestAI,pmlr-v49-russo16}. Yet, in many real-world settings, multiple and potentially conflicting objectives must be optimized simultaneously, and no single arm may achieve the optimal score across all criteria~\cite{Hayes2021APG}. Consequently, the rewards become multi-dimensional and several arms can be Pareto optimal. This gives rise to the \textit{Multi-Objective Multi-Armed Bandit} (MOMAB) framework first proposed by \citeauthor{Drugan2013DesigningMM}, \citeyear{Drugan2013DesigningMM}. A significant body of work focuses on minimizing Pareto regret by prioritizing the selection of optimal arms during the agent's sampling process~\cite{Drugan2013DesigningMM,Yahyaa2015ThompsonSF,Cheng2024HierarchizePD}. We instead focus on the pure-exploration formulation introduced by \citeauthor{Zuluaga2013ActiveLF}, \citeyear{Zuluaga2013ActiveLF} and \citeauthor{Auer2016ParetoFI}, \citeyear{Auer2016ParetoFI} in their seminal works. This problem, known as \textit{Pareto Set Identification} (PSI), aims to identify the set of Pareto optimal arms $\ttpTrueFront$ as efficiently and accurately as possible. An arm is included in the Pareto optimal set $\ttpTrueFront$ if it is not dominated by any other arm, meaning that no alternative arm has strictly better expected values for all of the objectives. Hence, the arms present in the Pareto set provide optimal trade-offs between objectives. This multi-objective exploration problem of identifying $\ttpTrueFront$ arises in numerous applications, including materials discovery~\cite{Gopakumar2018}, clinical trials~\cite{Zhao03062018}, epidemic mitigation policies~\cite{REYMOND2024123686}, molecular design~\cite{LiuMolecules2025}, and drug discovery~\cite{FromerMolecules2024}.

\paragraph{Related work.} \citeauthor{Auer2016ParetoFI}, \citeyear{Auer2016ParetoFI} proposed a fixed-confidence algorithm that identifies Pareto optimal arms, potentially including suboptimal arms that become Pareto optimal when increased coordinate-wise by $\epsilon_1$. Concurrently, \citeauthor{Zuluaga2013ActiveLF}, \citeyear{Zuluaga2013ActiveLF} studied a structured variant of fixed-confidence PSI where means are regular functions of arm descriptors. Utilizing Gaussian process modeling, they derived worst-case sample complexity bounds~\cite{Zuluaga2013ActiveLF,zuluaga2016epal}. Building on these foundations, the \textit{Adaptive Pareto Exploration} (APE) algorithm was proposed in~\cite{Kone2023AdaptiveAF}, which allows for various stopping rules to accommodate different relaxations of the PSI problem. More recently, \cite{pmlr-v238-crepon24a} established tight guarantees in the asymptotic regime by iteratively solving an optimization problem characterizing the optimal problem-dependent sample complexity. Also within the fixed-confidence setting, \cite{Kone2025ParetoSI} introduced a Bayesian approach relying on posterior distributions to guide sampling and stopping decisions. Addressing the fixed-budget PSI setting, \cite{Kone2023BanditPS} developed the \textit{Empirical Gap Elimination} (EGE) algorithms, which extend the successive-halving and successive-rejects strategies to multi-objective environments. Finally, numerous variants of the standard PSI setting have been explored, including (top) feasible-arm identification~\cite{KatzSamuels2018FeasibleAI,KatzSamuels2019TopFA,Kone2025ConstrainedPS} and PSI under partial ordering~\cite{Ararat2021VectorOW}.

\paragraph{Contributions.} While significant strides have been made in the fixed-confidence and fixed-budget MOMAB PSI settings, a notable gap remains regarding the \textit{anytime} setting. In this work, we address this gap by introducing the first anytime algorithm designed for the MOMAB PSI problem: \textit{Top-Two Pareto Front Thompson Sampling} (TTPFTS). Our work adopts a \textit{preference-unaware} approach \cite{Osika2023BeyondParetoFront}, aiming to learn the \textit{entire} Pareto set without making assumptions about the decision-maker preferences. We evaluate the algorithm's empirical performance on synthetic benchmarks against the state-of-the-art fixed-budget EGE \cite{Kone2023BanditPS} and fixed-confidence APE \cite{Kone2023AdaptiveAF} algorithms. Next, we validate the algorithm in a real-world decision-making scenario: exploring synthesis-on-demand molecular libraries to identify molecules with desirable properties. In this setting, we show that our algorithm achieves significant gains in computational efficiency over exhaustive virtual screening, the industry-standard approach \cite{Lyu2019UltraLarge,WaltersMols2024}, and over a state-of-the-art active learning method \cite{FromerMolecules2024}. Moreover, we introduce a novel uncertainty quantification metric specific to Bayesian MOMAB PSI algorithms. This metric leverages information contained in the posteriors to estimate confidence in current Pareto set estimate without ground-truth knowledge, offering a practical and interpretable method for guiding decision-making in complex multi-objective settings. Finally, we present the first finite-time analysis of an anytime MOMAB PSI algorithm, retrieving a general polynomially decreasing upper bound on the error probability, as well as instance-specific exponentially decaying error probability bounds.

\section{MOMAB PSI Setting}
In this section, the \textit{Multi-objective Multi-armed Bandit} (MOMAB) \textit{Pareto Set Identification} (PSI) setting is formalized. A MOMAB PSI problem instance consists of $K$ arms, each arm corresponding to a reward distribution $\nu_1,\cdots,\nu_K$ over $\R^d$ with respective expectations $\boldsymbol{\mu}_1,\cdots,\boldsymbol{\mu}_K$, where $d>1$ represents the number of objectives. The bandit instance is defined as $\nu := (\nu_1,\cdots,\nu_K)$. At each time $t=1,2,\cdots$ the agent pulls an arm $a_t \in [K]$ and observes an independent sample $\boldsymbol{X}_t\sim \nu_{a_t}$ with $\E\left[\boldsymbol{X}_t\right]=\boldsymbol{\mu}_{a_t}$. The goal of an anytime MOMAB PSI algorithm is to estimate the Pareto set after \textit{every} observation. These are precisely the arms that are not Pareto dominated by any other arm.

\begin{definition}[Pareto Domination]
    Let $\boldsymbol{\mu}_i$ and $\boldsymbol{\mu}_j$ be the expected reward vectors of arms $i$ and $j$. Arm $i$ is (Pareto) dominated by $j$ ($\boldsymbol{\mu}_i \preceq \boldsymbol{\mu}_j$) if $\mu_i^c \le \mu_j^c$ for all objectives $c \in \{1,\dots,d\}$ and there exists at least one coordinate $c$ where $\mu_i^c < \mu_j^c$. It is strongly dominated ($\boldsymbol{\mu}_i \prec \boldsymbol{\mu}_j$) if $\mu_i^c < \mu_j^c$ strictly for all $c$. 
\end{definition}

\begin{definition}[Pareto Set]
    The Pareto set, denoted $\ttpTrueFront(\nu)$ or simply $\ttpTrueFront$, contains all arms not dominated by any other arm: \[\ttpTrueFront(\nu) := \{ i \in [K] \mid \nexists j \in [K] : \boldsymbol{\mu}_i \preceq \boldsymbol{\mu}_j \}\] Arms within $\ttpTrueFront$ are termed Pareto optimal, while the remaining arms comprise the suboptimal set $[K] \setminus \ttpTrueFront$.
\end{definition}

The complexity of a MOMAB PSI problem instance can be characterized by the \textit{hardness} term: \[H(\nu):=\sum_{a=1}^{K}\frac{1}{\Delta_a^2}, \quad\Delta:=\min_{a\in[K]}\Delta_a>0,\] as introduced in \cite{Auer2016ParetoFI}, where $\Delta_a$ represents the suboptimality gap for arm $a$. This quantity is explained further in \Cref{subsec:gaps}.

While the anytime MOMAB PSI setting requires the algorithm to maintain and update its estimate of the Pareto set after every observation~\cite{Bubeck2009PureEI}, other formulations of the problem consider stopping conditions that depend on a predefined accuracy or budget constraint. In the \textit{fixed-confidence} MOMAB PSI setting, the learner aims to identify the true Pareto set $\ttpTrueFront$ with a probability of at least $1-\delta$ for some confidence level $\delta \in (0,1)$, while minimizing the total number of samples collected. Formally, the learner must design a stopping time $\tau$ and an estimator $\ttpEmpFront{\tau}$ such that $\Prob(\ttpEmpFront{\tau}=\ttpTrueFront)\ge1-\delta$~\cite{NIPS2012_8b0d2689,Kone2023AdaptiveAF,pmlr-v238-crepon24a,Kone2025ParetoSI}. In contrast, the \textit{fixed-budget} MOMAB PSI setting assumes a total sampling budget $B$ is fixed a priori. The learner must then use these $B$ samples to produce an estimator $\ttpEmpFront{B}$ of the Pareto set that maximizes the probability of correct identification $\Prob(\ttpEmpFront{B}=\ttpTrueFront)$~\cite{NIPS2012_8b0d2689,Kone2023BanditPS}. The anytime formulation, as considered here, is more flexible: the algorithm continuously refines its Pareto set estimate $\ttpEmpFront{T}$ after every observation without requiring a predefined stopping rule~\cite{AnytimeHype}.

\section{Algorithm}
In this section, we present the \textit{Top-Two Pareto Front Thompson Sampling} (TTPFTS) algorithm. TTPFTS constitutes the first algorithmic exploration of the anytime MOMAB PSI setting. 

The proposed TTPFTS algorithm extends the top-two principle, popularized by algorithms like \textit{Top-Two Thompson Sampling} (TTTS) \cite{pmlr-v49-russo16}, to the multi-objective domain. In single-objective best-arm identification, TTTS selects a candidate best arm and then repeatedly resamples to find a different arm that outperforms the candidate, effectively allocating samples between these two arms. TTPFTS generalizes this mechanism to the MOMAB setting without requiring iterative re-sampling. Without a preference vector to scalarize objectives, there typically exists no single best arm. Instead, there exists a set of Pareto optimal arms representing the trade-off surface. Consequently, TTPFTS generalizes the top-two notion from arms to \textit{Pareto fronts}: it adaptively distributes samples between the current estimated Pareto front, the best set, and the second Pareto front, the challenger set, according to a probability $\rho$ based on posterior samples.

\begin{figure}[ht]
    \centering
    \includegraphics[width=0.9\linewidth]{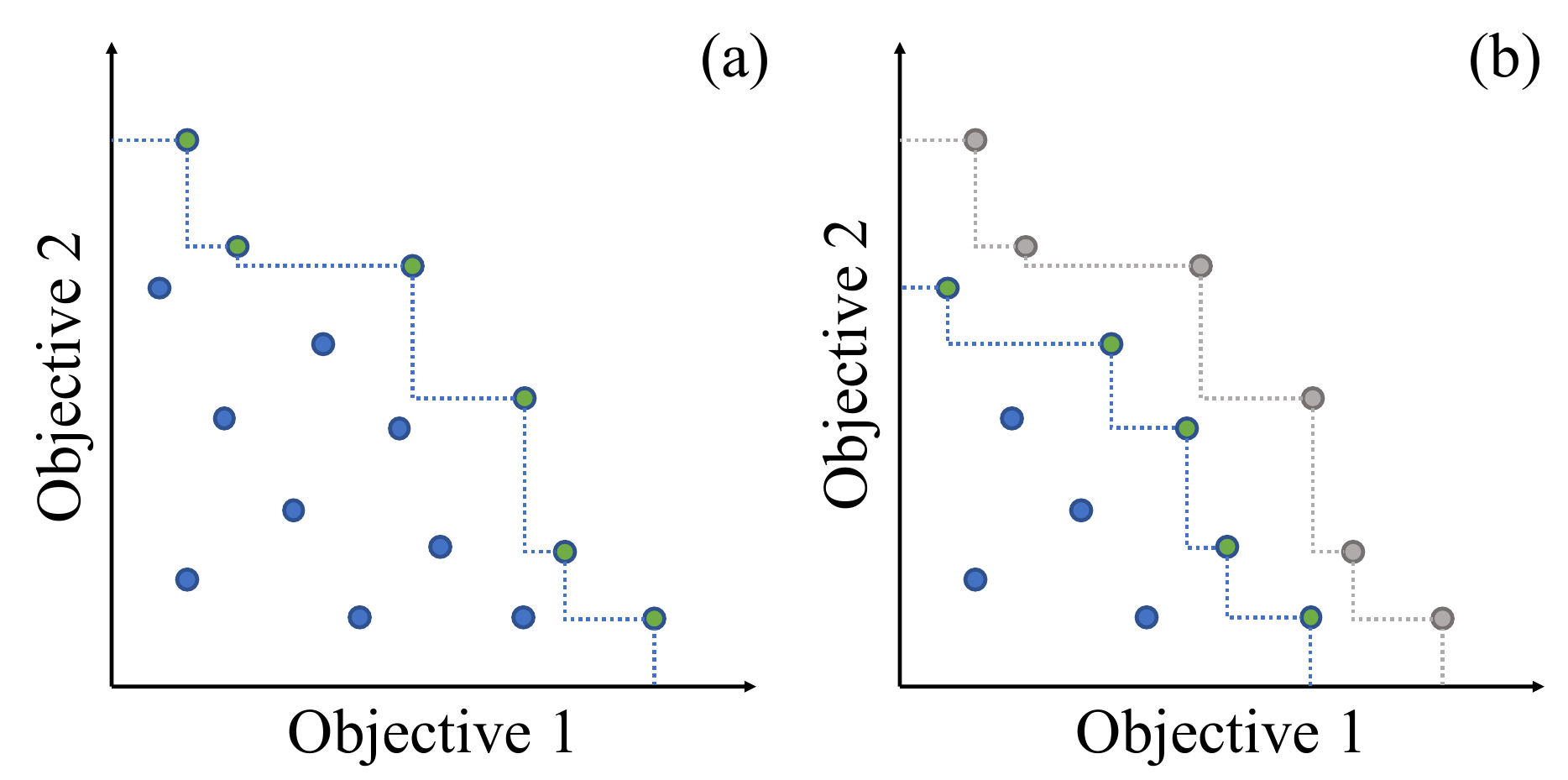}
    \caption{The TTPFTS strategy in a bi-objective maximization setting. The samples obtained at time $t$ by sampling from the posterior distributions $ \pi(\cdot | \mathcal{H}^{(t-1)})$ are shown as dots. With probability $\rho$, the optimal samples $\ttpPosteriorFront{1}{t}$ are computed and one of the corresponding arms is chosen uniformly at random (a). With probability $1-\rho$, the optimal samples are removed and one of the arms corresponding to optimal samples of the remaining set $\ttpPosteriorFront{2}{t}$ is chosen uniformly at random (b).}
    \label{fig:TopTwoFronts}
\end{figure}

The detailed procedure is outlined in Algorithm 1. At each time step $t$, the algorithm begins by sampling from the $d$-variate posterior distributions $\pi(\cdot \mid \mathcal{H}^{(t-1)})$ for each arm $a \in [K]$. For each arm, this produces an estimate $\muvtil{t}_a$ of its expected value, as illustrated by the spots in \Cref{fig:TopTwoFronts}. Once these samples are obtained, TTPFTS performs a Bernoulli trial with success probability $\rho$. The outcome of this trial decides which front to explore. If the trial yields 1, the algorithm identifies the set of Pareto optimal arms $\ttpPosteriorFront{1}{t}$ based on the current sampled means. As shown in \Cref{fig:TopTwoFronts}(a), an arm $a_t$ is then selected uniformly at random from this set. This step focuses on verifying arms that are currently believed to be optimal. If the trial yields 0 (probability $1-\rho$), the algorithm removes the optimal set $\ttpPosteriorFront{1}{t}$ and considers the subset of remaining arms $\mathcal{A}' = [K] \setminus \ttpPosteriorFront{1}{t}$. Subsequently, it computes the Pareto front of this subset, denoted as $\ttpPosteriorFront{2}{t}$. We refer to $\ttpPosteriorFront{2}{t}$ as the \textit{second front}, visualized in \Cref{fig:TopTwoFronts}(b). An arm $a_t$ is selected uniformly at random from this second front. This step promotes exploration among the most promising suboptimal arms, effectively testing if they have been underestimated or solidifying them as suboptimal. In the event that all arms are mutually non-dominating and the remaining subset $\mathcal{A}'$ is empty, TTPFTS selects an arm uniformly at random from $\ttpPosteriorFront{1}{t}$.

\begin{algorithm}[ht]
    \DontPrintSemicolon
    \caption{Top-Two Pareto Front Thompson Sampling}
    \KwIn{MOMAB instance $\nu$ with $K$ arms, parameter $\rho \in (0,1)$, prior $\pi$, history $\mathcal{H}^{(0)} = \emptyset$}
    
    \For{$t \leftarrow 1$ \KwTo $\infty$}{
        \For{$a \leftarrow 1$ \KwTo $K$}{
            Sample $\muvtil{t}_a \sim \pi(\cdot \mid \mathcal{H}^{(t-1)})$\;
        }
        
        Compute the first front: $\ttpPosteriorFront{1}{t}=\{ i \in [K] \mid \nexists j \in [K] : \muvtil{t}_i \preceq \muvtil{t}_j \}$\;       
        Sample $b \sim \mathcal{B}(\rho)$\;
        \If{$b=1$}{
            Select arm $a_t$ uniformly at random from $\ttpPosteriorFront{1}{t}$\;
        } \Else{
            Compute remaining arms $\mathcal{A}' = [K] \setminus \ttpPosteriorFront{1}{t}$\;
            Compute the second front: $ \ttpPosteriorFront{2}{t} = \{ i \in \mathcal{A}' \mid \nexists j \in \mathcal{A}' : \muvtil{t}_i \preceq \muvtil{t}_j \}$\;
            Select arm $a_t$ uniformly at random from $\ttpPosteriorFront{2}{t}$\;
        }

        Pull arm $a_t$ and observe reward $\boldsymbol{X}_t \sim \nu_{a_t}$\;
        Update history $\mathcal{H}^{(t)} \leftarrow \mathcal{H}^{(t-1)} \cup \{(a_t, \boldsymbol{X}_t)\}$\;
        Compute and recommend $\ttpEmpFront{t}$, the Pareto set of the posterior means\;
    }
\end{algorithm}

Finally, regardless of which front was selected, the learner pulls the chosen arm $a_t$, observes the reward vector $\boldsymbol{X}_t \sim \nu_{a_t}$, and updates the history $\mathcal{H}^{(t)}$. As this is an anytime algorithm, the learner maintains a recommendation set $\ttpEmpFront{t}$, which can be queried at any point during the process. While distinct, this strategy is similar in spirit to \textit{Boundary Focused Thompson Sampling} \cite{LibinBFTS2019}, which addresses the single-objective top-$m$ problem by targeting the decision boundary between the $m$-th and $(m+1)$-th arms. TTPFTS focuses exploration on the MOMAB PSI decision boundary: the divide between the first and second Pareto fronts.

\section{Theoretical Guarantees}
\label{sec:theory}
Define the error probability as
\begin{equation}
  \Err_T:=
  \{\widehat{\mathcal S}^{\star}_T\ne\mathcal S^\star\},
  \qquad
  e_T(\nu):=\Prob_\nu(\Err_T).
  \label{eq:error-event}
\end{equation}
In this section, the upper bound for the error probability under the TTPFTS sampling rule is derived. Throughout this analysis, we assume the arms' marginal reward distributions are $\sigma$-subgaussian, a common assumption in literature \cite{Kone2023BanditPS}. We distinguish \textit{mixed-front separated}, \textit{all-Pareto}, and $K=2$ instances.

\begin{definition}[Mixed-front Separated Instance]
\label{def:poly-mixed-instance}
An instance is \emph{mixed-front separated} when $K\geq3$, $d\geq1$,
$\sigma>0$, and
\begin{equation}
  \varnothing\ne \ttpTrueFront\subsetneq[K],
  \qquad
  \Delta_a>0\quad(a\in[K]).
  \label{eq:poly-mixed-scope}
\end{equation}
\end{definition}

This condition characterizes virtually all non-trivial MOMAB PSI problems, requiring only $K\geq3$ and the presence of at least one suboptimal arm alongside the Pareto set.

\begin{theorem}[Mixed-front Bound]
\label{thm:poly-mixed-finite}
Let $\nu$ be a mixed-front separated instance. Define
$\beta(\nu)$, $c_{\rm sel}(\nu)$, and $T_0^{\rm poly}(\nu)$ as in \eqref{eq:poly-constants}--\eqref{eq:poly-horizon}. Then every integer
$T\geq T_0^{\rm poly}(\nu)$ satisfies
\begin{equation}
  \ttpErrProb{T}{\nu}
  \leq
  \left(K+\tfrac12\right)T^{-\beta(\nu)}
  +K\exp\!\left\{-\frac{c_{\rm sel}(\nu)}8T^{1/4}\right\}.
  \label{eq:poly-mixed-bound}
\end{equation}
\end{theorem}

\begin{theorem}[All-Pareto Bound]
\label{thm:all-pareto}
Assume $K\geq2$, $d\geq1$, marginal subgaussianity, and $\mathcal S^\star=[K]$ with the positive gap $\Delta$. Let
\begin{equation}
  q:=\min\{\rho,1-\rho\},\qquad
  \alpha:=\frac{q}{16K^2},\qquad
  \gamma:=\frac{q}{2K},
  \label{eq:ap-constants}
\end{equation}
and, for an integer formal horizon $T$, set
\begin{equation}
  n_0(T):=\lceil\alpha T\rceil,
  \quad
  m_\star:=
  \left\lceil
    \frac{128\sigma^2}{\Delta^2}\log(4Kd)
  \right\rceil,
  \quad
  T_0:=
  \left\lceil\frac{m_\star}{\alpha}\right\rceil.
  \label{eq:ap-horizon}
\end{equation}
Then, for every integer $T\geq T_0$,
\begin{equation}
  e_T(\nu)
  \leq
  2Kd\exp\!\left(-\frac{\alpha T\Delta^2}{32\sigma^2}\right)
  +\exp\!\left(-\frac{\gamma T}{8}\right).
  \label{eq:ap-bound}
\end{equation}
\end{theorem}

\begin{theorem}[$K=2$ Bound]
\label{thm:two-arm}
Assume $K=2$, $d\geq1$, marginal subgaussianity, and
\begin{equation}
  \Delta:=\min\{\Delta_1,\Delta_2\}>0.
  \label{eq:two-arm-gap}
\end{equation}
Let $q=\min\{\rho,1-\rho\}$. For every formal horizon $T\in\N_0$,
\begin{equation}
  e_T(\nu)
  \leq
  4d\exp\!\left(-\frac{qT\Delta^2}{64\sigma^2}\right)
  +2\exp\!\left(-\frac{qT}{8}\right).
  \label{eq:two-arm-bound}
\end{equation}
\end{theorem}

Under the stated assumptions, \Cref{thm:poly-mixed-finite} provides a general polynomial-rate upper bound on $e_T(\nu)$, while the instance-specific \Cref{thm:all-pareto,thm:two-arm} admit exponential bounds. The formal model and assumptions, together with the full proofs of these results, are provided in \Cref{sec:model-contract,sec:exponential-bounds,sec:universal-polynomial}.

\section{Empirical Evaluation}
\label{sec:emp_eval}
Next, we empirically evaluate the performance of the proposed TTPFTS algorithm on the PSI task. Our evaluation is conducted using the eight synthetic MOMAB PSI benchmarking environments introduced by \citeauthor{Kone2023BanditPS}, \citeyear{Kone2023BanditPS}, further specified in \Cref{sec:supp:env_specifications}. These environments are designed to capture a wide range of Pareto set geometries, sub-optimality gap profiles, and objective space dimensionalities commonly encountered in MOMAB problems. For each environment, we use Gaussian reward distributions with diagonal covariance matrix $\sigma^2 I_d$ where $I_d$ is the identity matrix of order $d$ and variance $\sigma^2=0.25$.
 
\begin{figure*}[ht]
    \centering
    \includegraphics[width=\linewidth]{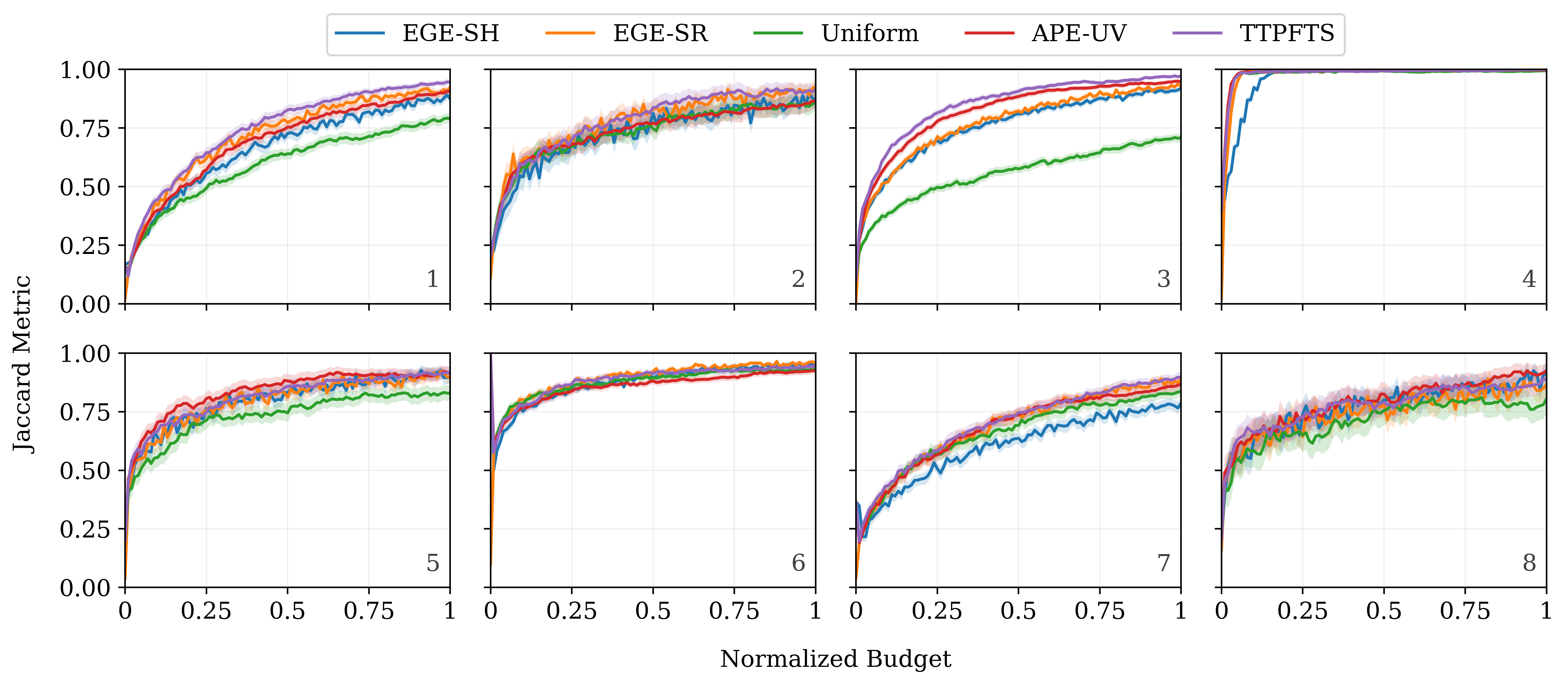}
    \caption{Performance of the MOMAB PSI algorithms with respect to the Jaccard metric in the eight synthetic benchmarking environments. We use Gaussian reward distributions with diagonal covariance matrix $\sigma^2 I_d$ with $\sigma^2=0.25$. Results are averaged over 100 experimental trajectories, with shaded areas indicating the 95\% confidence interval around the mean performance.}
    \label{fig:jaccard_all_envs}
\end{figure*}

We compare TTPFTS with the fixed-confidence APE algorithm \cite{Kone2023AdaptiveAF} and both variants of the fixed-budget EGE algorithm \cite{Kone2023BanditPS}: \textit{Successive Halving} (EGE-SH) and \textit{Successive Rejects} (EGE-SR). These represent the state-of-the-art for fixed-confidence and fixed-budget MOMAB PSI algorithms. In addition, we include a Uniform sampling strategy as a baseline. As a Bayesian method, TTPFTS requires specifying a prior distribution. To demonstrate robustness, we assume weak prior information by employing an uninformative prior \cite{Honda2013OptimalityOT} over a Gaussian likelihood with unknown mean and unknown variance, yielding a $t$-distributed posterior. To ensure an information-matched comparison under unknown reward variance, we adapt APE \cite{Kone2023AdaptiveAF}: whereas standard APE relies on a user-supplied noise parameter, our adapted variant (APE-UV) receives no variance input. Instead, it tracks running sample variances online via Welford's algorithm, with both TTPFTS and APE-UV initialized with four warm-up pulls per arm. In \Cref{sec:supp:rhobustness}, TTPFTS is shown to be robust to its single hyperparameter $\rho$, which is set to 0.5 across all experiments.

To quantify the quality of the estimated Pareto sets produced by each algorithm, we use the Jaccard metric $J$~\cite{Lu2019MultiObjectiveGL} between the estimated Pareto set $\ttpEmpFront{}$ and the ground-truth Pareto set $\ttpTrueFront$:
\begin{equation}
    J(\ttpTrueFront,\ttpEmpFront{}):=\frac{|\ttpTrueFront\cap\ttpEmpFront{}|}{|\ttpTrueFront\cup\ttpEmpFront{}|}
\end{equation}
For each environment, the sampling budget is set equal to its hardness $H(\nu)$, reported in \Cref{tab:environment-hardness-budgets}, and we report mean performance over 100 experimental trajectories. Due to the inherently distinct nature of the algorithms under comparison, fixed-budget and fixed-confidence versus anytime, we clarify how a single experimental trajectory is generated for each method. The fixed-budget EGE algorithms are executed independently for every budget $T \in \left(0, H(\nu)\right]$, producing an estimate at the end of each run. The APE algorithm's fixed-confidence stopping rule is modified so that it outputs its estimated Pareto set after a predefined budget, leaving its sampling rule unchanged. In contrast, TTPFTS is executed only once per trajectory and reports its anytime estimate after each arm pull. Note that anytime PSI asks a strictly harder question than the fixed budget setting: an anytime algorithm can perform under the fixed budget setting but not vice versa \cite{pmlr-v48-jun16}.

\begin{table}[ht]
  \centering
  \caption{Per-environment number of arms $K$ and objectives $d$, hardness $H(\nu)$, and sampling budget $B$, with
    $B=\operatorname{round}(H(\nu))$.}
  \label{tab:environment-hardness-budgets}
  \begin{tabular}{crrrr}
    \toprule
    Environment & $K$ & $d$ & $H(\nu)$ & Budget $B$ \\
    \midrule
     1 & $60$  & $2$  & $15\,217.76$  & $15\,218$ \\
     2 & $10$  & $2$  & $2\,543.76$   & $2\,544$ \\
     3 & $200$ & $2$  & $685\,035.73$ & $685\,036$ \\
     4 & $50$  & $10$ & $15\,977.30$  & $15\,977$ \\
     5 & $20$  & $2$  & $197\,372.62$ & $197\,373$ \\
     6 & $10$  & $2$  & $51\,966.41$  & $51\,966$ \\
     7 & $22$  & $2$  & $7\,000.00$   & $7\,000$ \\
     8 & $5$   & $2$  & $237\,957.54$ & $237\,958$ \\
    \bottomrule
  \end{tabular}
\end{table}

\Cref{fig:jaccard_all_envs} reports the evolution of the Jaccard metric. TTPFTS consistently outperforms the Uniform Sampling baseline and the EGE-SH algorithm across all environments. Furthermore, TTPFTS's scores are competitive with those of the EGE-SR and APE algorithms across all environments. Notably, TTPFTS achieves parity with state-of-the-art algorithms explicitly optimized for fixed horizons, despite operating under the stricter anytime constraint. In Environment 3, which is characterized by a large number of arms and the largest hardness $H(\nu)$, TTPFTS clearly outperforms both EGE variants, and scores higher than the APE algorithm. In this challenging setting, TTPFTS converges faster and achieves substantially higher Jaccard scores, suggesting that its posterior-driven allocation is particularly effective in difficult large-scale problems where rigid gap-based arm elimination, as performed by EGE, may be premature. Finally, these results validate the algorithm's effectiveness even under weak prior assumptions. To provide a comprehensive assessment beyond the Jaccard metric, we include Bernoulli and misclassification metrics in \Cref{sec:supp:additional_PSI_metrics}. These metrics show similar good results for TTPFTS and offer complementary insights into the probability of \textit{exact} PSI and the evolution of classification errors, respectively. The implementation of all experiments is available on GitHub\footnote{\url{https://github.com/LennertSaerens/TTPFTS}}.

\section{Uncertainty Quantification}
\label{sec:uq}
In the previous section, we used our knowledge of the ground truth Pareto set to gain insight into the quality of the PSI estimates of MOMAB PSI algorithms. However, in real-world decision-making settings, the decision maker has no access to the true Pareto set, hence necessitating the PSI process. In this section, we define a measure of TTPFTS's uncertainty about its current estimate $\ttpEmpFront{}$. The proposed uncertainty measure leverages TTPFTS's Bayesian nature by using the information captured in the posterior distributions of the top-two Pareto fronts. We quantify uncertainty $\UQ$ as the average Bhattacharyya overlap \cite{Bhattacharyya1943} between the posterior densities of the arms in the current top-two empirical Pareto fronts, i.e.\ based on posterior means \cite{VanMolle2021}. Intuitively, high overlap indicates that the algorithm is struggling to distinguish optimal arms from suboptimal challengers, i.e.\ high uncertainty, while limited overlap signals a clear separation, i.e.\ high confidence. This concept is visualized in \Cref{fig:UQ}.

\begin{figure}[ht]
    \centering
    \includegraphics[width=\linewidth]{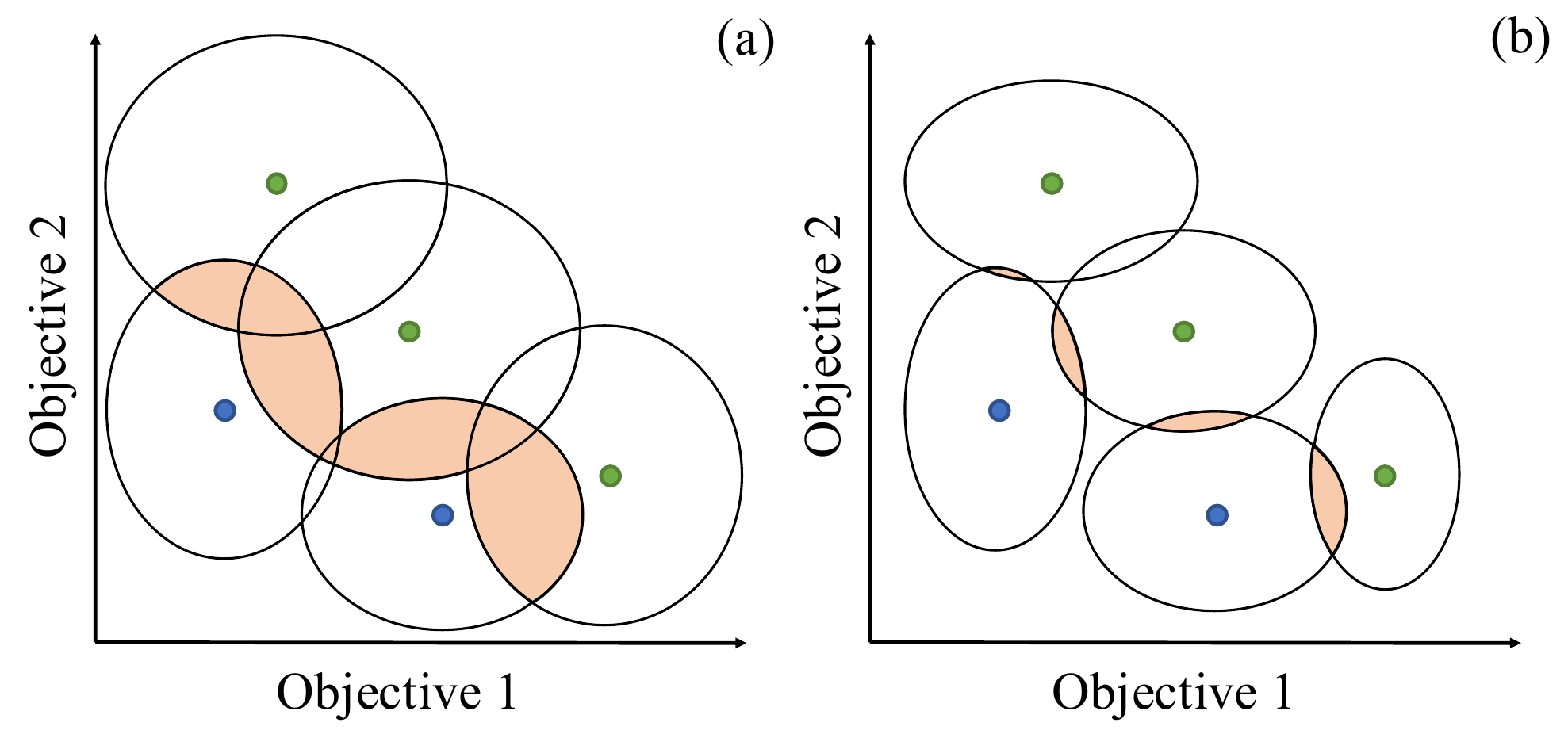}
    \caption{Visualization of proposed uncertainty quantification in the top-two Pareto fronts of a bi-objective maximization setting. The points indicate the mean of the arms' posterior distributions, while the ellipses around the means represent the posterior densities. As TTPFTS samples from the top-two fronts, the posteriors concentrate around the true means, evolving from a situation with high uncertainty, panel (a), to a situation with lower uncertainty, panel (b), indicated by the smaller shaded overlap in the probability density functions.}
    \label{fig:UQ}
\end{figure}

The Bhattacharyya coefficient $\beta$ lies in the interval $\left[0,1\right]$, where 1 indicates perfect overlap between distributions. \Cref{eq:UQ} shows the computation of our proposed uncertainty measure, detailed further in \Cref{sec:supp_uq_calc}.
\begin{equation}
    \UQ\left(\ttpEmpFront{1},\ttpEmpFront{2}\right) = 
 \frac{1}{|\ttpEmpFront{1}| |\ttpEmpFront{2}|} \sum_{a_i \in \ttpEmpFront{1}} \sum_{a_j \in \ttpEmpFront{2}} \Bhat(a_i, a_j)
 \label{eq:UQ}
\end{equation}
To derive a normalized, environment-independent quantification, we compute the coefficient for every pair of arms between the top-two fronts, summing these values and dividing by the total number of pairwise comparisons. As visualized in \Cref{fig:UQ}, this summation may involve double counting in dense regions of the objective space. Consequently, the resulting measure acts as a conservative estimate of total uncertainty. An evaluation justifying our choice of the Bhattacharyya coefficient over alternative information-theoretic metrics, such as KL-divergence and posterior entropy, is detailed in \Cref{sec:uq_comparison}.

\begin{figure}[ht]
    \centering
    \includegraphics[width=1\linewidth]{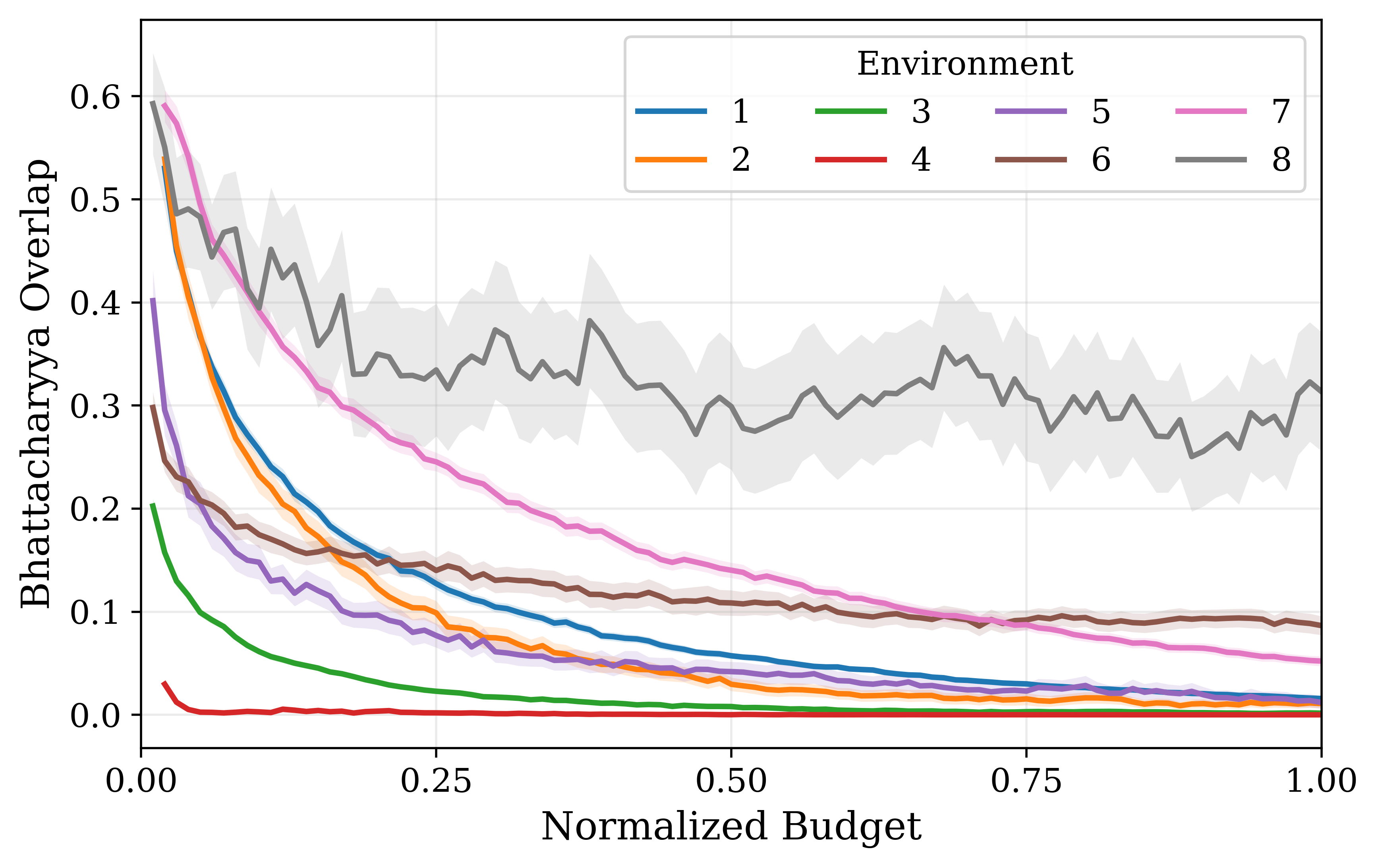}
    \caption{Evolution of the uncertainty of TTPFTS over the normalized budget in each of the eight synthetic benchmarking environments. Results are averaged over 100 experimental trajectories, with shaded areas indicating the 95\% confidence interval around the mean uncertainty.}
    \label{fig:UQ_evolution}
\end{figure}

We first analyze the temporal evolution of this metric across the benchmark environments from \Cref{sec:emp_eval}. As shown in \Cref{fig:UQ_evolution}, the uncertainty decreases as the learner gathers evidence. A decaying trend is evident in most environments, with Environment 8, characterized by a unique optimal arm and geometrically vanishing sub-optimality gaps, exhibiting a slower decrease in uncertainty. Notably, these trends align with the performance curves in \Cref{fig:jaccard_all_envs}: environments that are harder to learn exhibit a slower Jaccard improvement and display a corresponding slower reduction in uncertainty.

\begin{figure}[ht]
    \centering
    \includegraphics[width=1\linewidth]{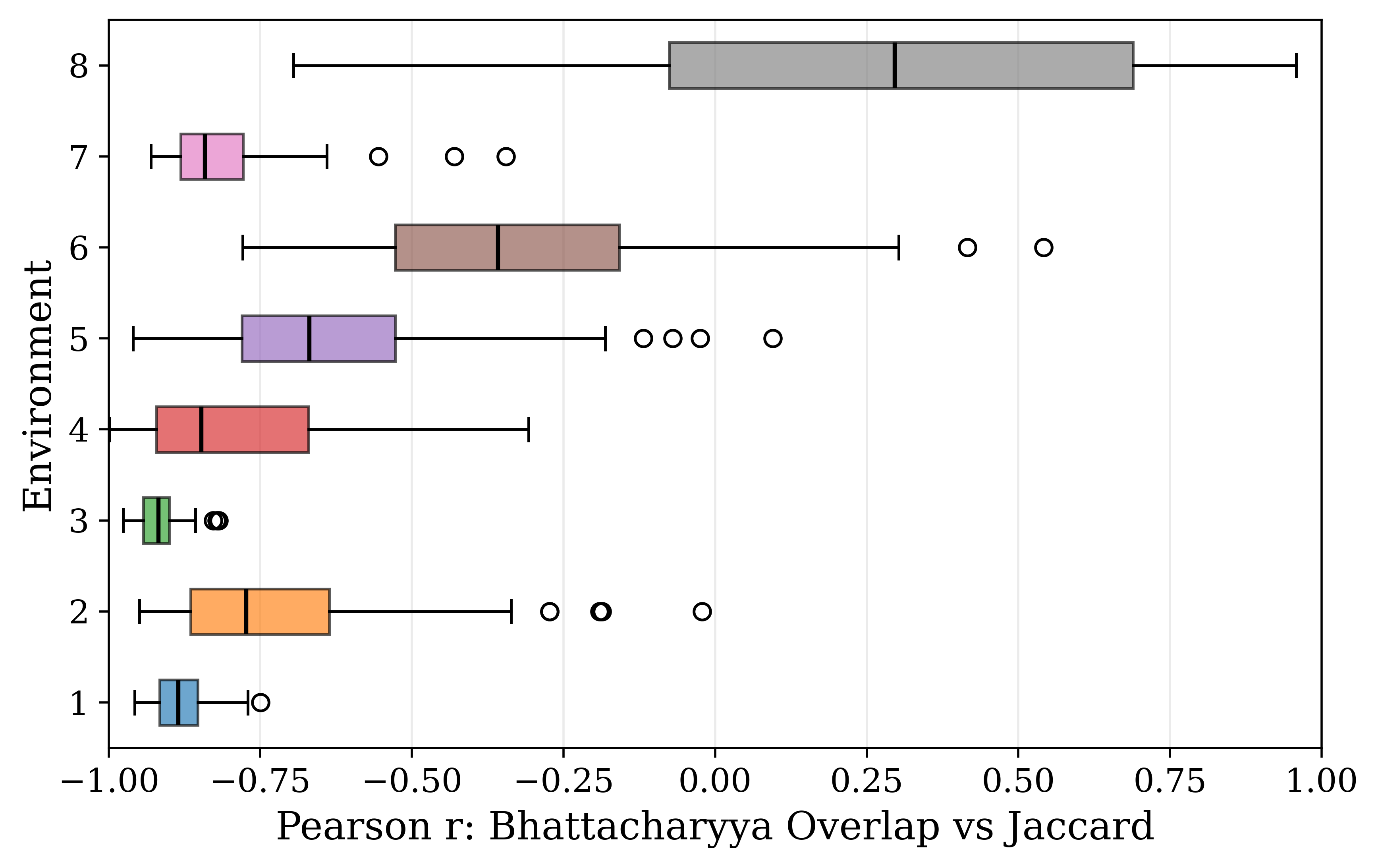}
    \caption{Distribution of Pearson correlation coefficients between the Jaccard metric and the proposed uncertainty quantification over 100 trajectories consisting of $H(\nu)$ arm pulls. A negative correlation indicates that the uncertainty metric is a reliable real-time proxy for identification performance.}
    \label{fig:UQ_correlation}
\end{figure}

To formally validate this metric as a proxy for PSI quality, we examine the correlation between the proposed uncertainty measure and the ground-truth Jaccard metric. For each environment, \Cref{fig:UQ_correlation} presents boxplots of the Pearson correlation coefficients calculated over 100 independent trajectories, each consisting of $H(\nu)$ arm pulls. The results reveal a strong negative correlation for the majority of environments (1--5 and 7). This confirms that a decrease in our uncertainty metric reliably maps to an increase in the accuracy of the Pareto set estimate. Although less pronounced, the correlation remains negative in Environment 6, the all-Pareto instance. At the budget and reward noise level considered in \Cref{sec:emp_eval}, Environment 8 shows a variable relationship. Its unique optimal arm is separated from suboptimal arms by geometrically vanishing gaps: a near-optimal arm can remain on the empirical first front while having little posterior overlap with more distant second-front challengers, yielding low $\UQ$ despite an incorrect estimate. The singleton true set also makes Jaccard relatively coarse and the per-run correlations sensitive to recovery timing. However, extended-budget and reduced-noise experiments (\Cref{sec:sup:UQ_ENV8}) support interpreting this result as correctly indicating unresolved estimation: the expected relation clearly emerges with more pulls or lower noise. Finally, we note that this uncertainty quantification creates a powerful synergy with the anytime nature of TTPFTS. Unlike fixed-budget algorithms, where the stopping time is rigid and predetermined, TTPFTS empowers the decision-maker to stop the process dynamically. By monitoring the uncertainty, decision-makers can decide when the algorithm has reached a low enough uncertainty, where the top-two fronts have sufficiently separated, and terminate the experiment. This minimizes resource expenditure and maximally informs the decision-maker.

\section{Efficiently Screening Molecular Databases}
\label{sec:mols}
Moving beyond synthetic benchmarking, we apply TTPFTS to the exploration of \textit{ultra-large synthesis-on-demand} molecular libraries~\cite{WaltersMols2024}. Our goal is to identify molecules that optimally trade off multiple desirable properties by efficiently exploring chemical spaces where exhaustive enumeration is infeasible. This experiment validates TTPFTS's applicability to complex, real-world multi-objective decision-making tasks.

Synthesis-on-demand libraries, constructed combinatorially from smaller sets of reagents~\cite{GRYGORENKO2020101681}, have expanded to contain billions of candidate molecules, rendering exhaustive virtual screening computationally infeasible. To expedite hit identification, \cite{WaltersMols2024} introduced a multi-agent bandit-based approach where independent Thompson Sampling agents simultaneously select reagents for each reaction component, achieving a substantial speedup over exhaustive search~\cite{Lyu2019UltraLarge}. However, their work was limited to a single molecular property. We demonstrate that TTPFTS enables the extension of this paradigm to the multi-objective optimization of molecular properties. A complete formalization of the multi-agent screening framework and the multi-objective extension is provided in \Cref{sec:supp:mols}. We adopt the framework of \cite{WaltersMols2024}, utilizing a combinatorial library generated via a three-component reaction. The reagent pools for the three components consist of 376, 500, and 500 unique reagents, respectively, yielding an enumerated library of 94 million compounds \cite{WaltersMols2024}. In the bandit formulation, this corresponds to three independent agents, i.e.\ one per component, where the arms represent the available reagents. The reward signal is derived from the evaluation of the final product resulting from the combination of the bandit-selected reagents. We note that the multi-agent nature of this framework induces non-stationary rewards, discussed further in \Cref{sec:supp:mols}.

To transition to the MOMAB PSI setting, we introduce two key modifications. First, we expand the optimization landscape by introducing a second objective. In addition to maximizing structural similarity to a query molecule, we simultaneously optimize for the \textit{LogP}, a key molecular property in drug discovery, representing a compound's lipophilicity, which influences properties such as membrane permeability \cite{LIPINSKI19973,Waring2010}. This bi-objective setup mirrors early stage screening in drug discovery, such as for Alzheimer's disease \cite{GOTE2023101334}. Secondly, we deploy TTPFTS agents for each of the three reagent-selecting bandits. Notably, we exclude fixed-budget algorithms like EGE from this multi-agent combinatorial setting due to significant structural incompatibilities: EGE variants rely on rigid phases requiring specific budget allocations across active arms before eliminations can occur. In a synthesis-on-demand loop, where three independent bandits must synchronously propose reagents to form a single valid product, coordinating these phased requirements becomes computationally intractable. Specifically, aligning the distinct active subsets from multiple agents to form valid reaction products introduces a secondary combinatorial optimization problem. Due to its anytime nature, TTPFTS avoids this bottleneck by naturally fitting the system's sequential, single-pull paradigm.

\begin{figure}[ht]
    \centering
    \includegraphics[width=1\linewidth]{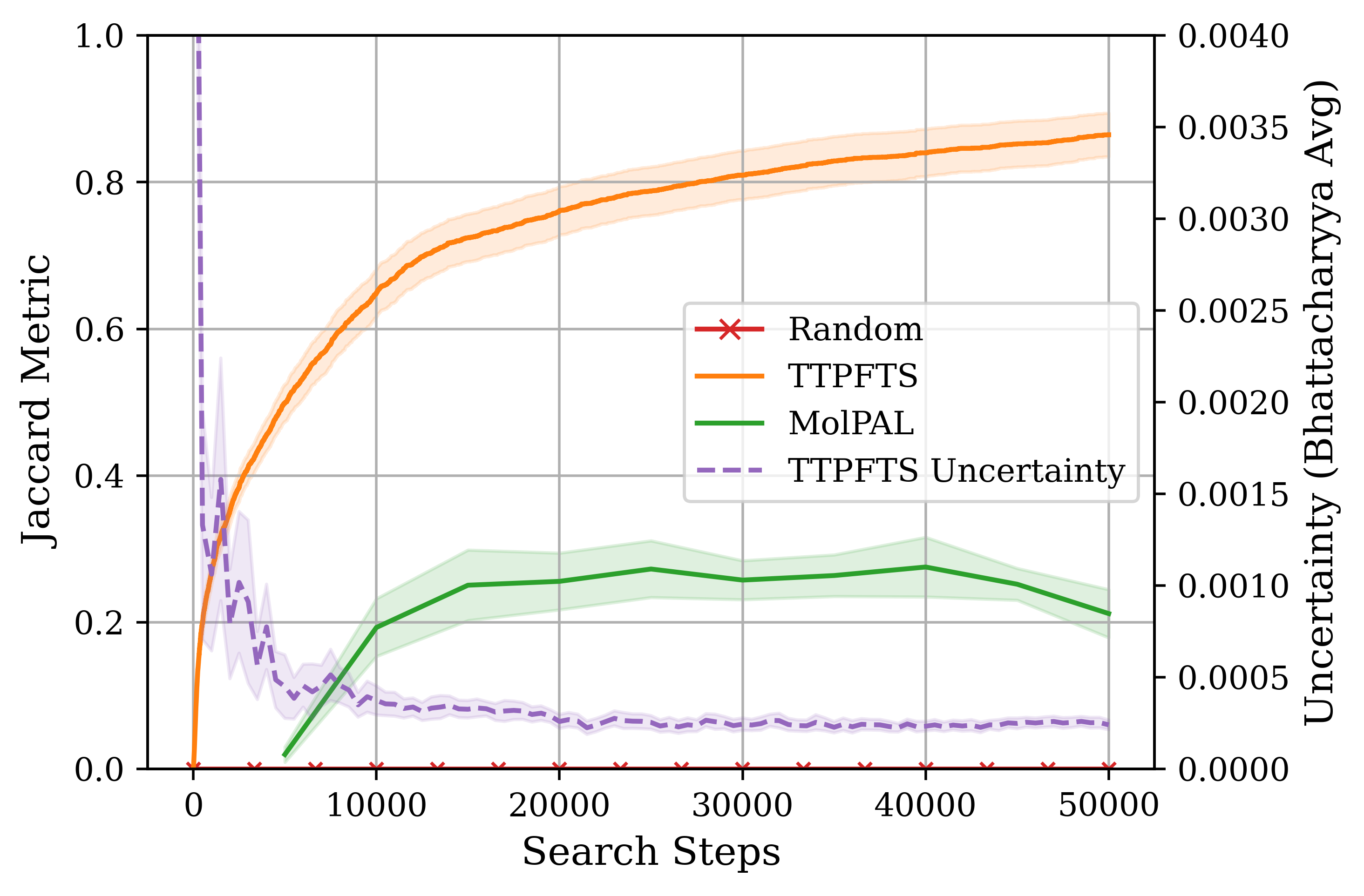}
    \caption{Evolution of the Jaccard metric of the optimal molecules identified by Random Search, MolPAL, and TTPFTS as well as the proposed uncertainty measure for TTPFTS. Results are averaged over 100 experimental trajectories, with the shaded areas indicating the 95\% confidence interval around the mean.}
    \label{fig:mols_evolution}
\end{figure}

To establish a ground truth, we perform an exhaustive search across the full library of 94 million molecules which reveals a true Pareto set consisting of 52 molecules visualized in \Cref{fig:mols_scatter}. Following the protocol in \cite{WaltersMols2024}, we employ a Random Search strategy as a baseline. This strategy establishes the expected performance of the industry-standard exhaustive virtual screening approach. We also include the state-of-the-art MolPAL baseline \cite{FromerMolecules2024} (configuration in \Cref{sec:supp:molpal_config}). All methods are halted after 50,000 steps. 

\begin{figure}[ht]
    \centering
    \includegraphics[width=0.75\linewidth]{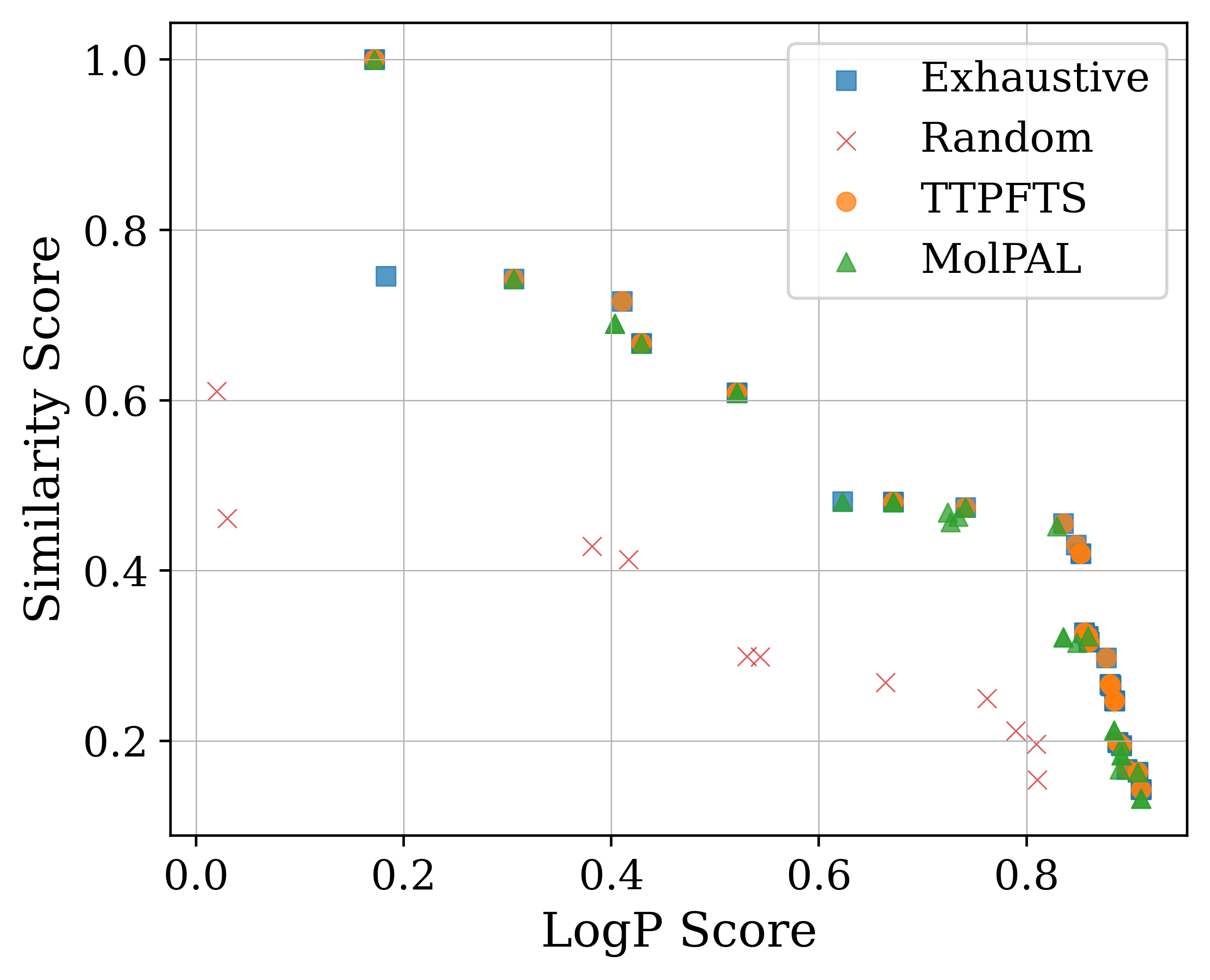}
    \caption{Structural Similarity and LogP scores for the optimal molecules identified in a single run by Random Search, MolPAL, and TTPFTS compared to the ground-truth Pareto-optimal molecules obtained through exhaustive virtual screening of all 94 million candidates.}
    \label{fig:mols_scatter}
\end{figure}

To quantify the average performance over time, \Cref{fig:mols_evolution} reports the Jaccard metric between the estimated and true Pareto sets alongside the evolution of the proposed uncertainty measure. Results are averaged over 100 runs, with shaded regions indicating $95\%$ confidence intervals. The computation of the Jaccard and uncertainty quantification metrics within this multi-agent setting is discussed in \Cref{sec:supp:mols_psi_uq_metrics}. Across all runs, Random Search consistently fails to identify any of the Pareto-optimal molecules. MolPAL reaches a mean Jaccard score of $0.28$ at 40{,}000 oracle calls before declining due to exploitation bias in later iterations. In contrast, TTPFTS reliably obtains scores exceeding $0.8$ within approximately 30{,}000 search steps, confirming the algorithm’s ability to rapidly identify optimal trade-offs. Crucially, the figure reveals a strong inverse relationship between the Jaccard score and the uncertainty measure $\UQ$. As the estimation quality improves, the uncertainty exhibits a sharp, consistent decline. This validates the metric as a reliable proxy for convergence even in this complex chemical space. 

\Cref{fig:mols_scatter} illustrates the outcomes of a representative run. The figure compares the molecules identified by each method to the ground-truth Pareto-optimal set. The results show that, apart from two molecules, the molecules discovered by TTPFTS perfectly coincide with the true Pareto set, despite evaluating only $0.05\%$ of the entire database. In contrast, Random Search fails to recover any of the true Pareto-optimal molecules within the same budget. These findings highlight TTPFTS’s ability to efficiently balance exploration and exploitation in a vast chemical space and to rapidly concentrate sampling on promising regions. Additional runs exhibiting similar behavior are provided in \Cref{sec:supp:add_results_mols}.

These findings highlight a powerful synergy between the anytime nature of TTPFTS and our uncertainty quantification metric. Unlike fixed-budget approaches, this combination provides the drug discovery team with a responsive decision-support system. Drug discovery scientists can monitor the uncertainty in real-time and dynamically determine when the Pareto set is sufficiently resolved to terminate the screening, thereby optimizing the trade-off between computational resources and discovery confidence without requiring prior knowledge of the true Pareto set.

\section{Discussion}
We introduced Top-Two Pareto Front Thompson Sampling, the \textit{first} algorithm for the anytime MOMAB PSI setting.

The anytime nature of Top-Two Pareto Front Thompson Sampling (TTPFTS) provides a more flexible alternative to the fixed-budget and fixed-confidence regimes. Our results demonstrate that this flexibility comes at no cost to performance: TTPFTS matches, and in large-action spaces and complex problem instances outperforms, state-of-the-art fixed-budget and fixed-confidence baselines. We enable the immediate practical use of this anytime property through a novel uncertainty quantification metric. This provides decision-makers with a real-time method for monitoring progress in real-world multi-objective decision making problems.

Beyond synthetic benchmarks, the combined practical power of our contributions was validated through a challenging multi-objective multi-agent molecular discovery setting. TTPFTS efficiently navigated the trade-off between conflicting molecular properties, recovering the true Pareto set while exploring a fraction of the search space, and outperforming both the industry-standard approach and the state-of-the-art MolPAL active learning method. Furthermore, the uncertainty quantification metric accurately tracked this progress, effectively acting as a dynamic stopping criterion for the molecular discovery process. Moreover, we present the first finite-time theoretical analysis of an anytime MOMAB PSI algorithm, proving a general polynomially decreasing upper bound on the error probability, as well as two instance-specific exponentially decaying error probability bounds.

Future work will focus on anytime MOMAB PSI algorithms that admit an exponential upper bound matching the information theoretic lower bound. Furthermore, refining the molecular discovery setting with non-stationary anytime MOMAB PSI algorithms is a promising research direction. Third, we believe that in future work the proposed uncertainty quantification metric could enable the design of adaptive stopping rules for the sampling process. To conclude, TTPFTS delivers immediate practical value and stands as a robust step forward in multi-objective decision-making.

\section*{Acknowledgments}
L.S. is supported by an imec doctoral grant. B.S. acknowledges FWO grant G059423N and OZR-VUB (OZR3863BOF). P.J.K.L. gratefully acknowledges support from the Research council of the Vrije Universiteit Brussel (OZR-VUB via grant number OZR3863BOF), FWO grant G059423N, FWO grant G0A0S26N, and FWO grant S002626N. This research acknowledges funding from the Flemish Government through the AI Research Program. The computational resources and services used in this work were provided by the VSC (Flemish Supercomputer Center), funded by the Research Foundation Flanders (FWO) and the Flemish Government department EWI. This research is partly funded by the $\text{imec}.\text{prospect}$ project IMPACT.

\bibliographystyle{plainnat}
\bibliography{sources}

\appendix
\onecolumn

\section{Theoretical Finite-time Analysis: Model and Notation Contract}
\label{sec:model-contract}

\subsection{Parameters, formal time, and physical budget}
\label{subsec:parameters}

Fix integers $K\geq 1$, $d\geq 1$, a posterior scale $\sigma>0$, and a
Top-Two parameter $\rho\in(0,1)$. Write $[K]=\{1,\ldots,K\}$ and
$[d]=\{1,\ldots,d\}$.

There are \textbf{two time budgets}:
\begin{itemize}
  \item $T\in\N_0$ is the number of formal executions of Algorithm~1,
    indexed by $t=1,\ldots,T$.
  \item To ensure proper initialization of the posterior distributions, before formal round $1$, the algorithm takes \textbf{exactly one}
    initialization pull of every arm. Hence the physical number of reward
    vectors observed after $T$ formal rounds is exactly
    \[
      K+T,
    \]
    not $T$. Every theoretical statement using a sampling budget states which of
    these two conventions it uses.
\end{itemize}

The total-pull count of arm $a$, including its initialization observation,
is
\begin{equation}
  N_a(0)=1,\qquad
  N_a(t)=1+\sum_{s=1}^{t}\mathbf 1\{A_s=a\}\quad(t\geq1).
  \label{eq:pull-count}
\end{equation}
Thus $N_a(t)\geq1$ always, $\sum_a N_a(t)=K+t$, and the stipulated
Gaussian posterior (see \Cref{subsec:pseudo-posterior}) is well-defined at every formal round.

\subsection{Environment and reward stacks}
\label{subsec:environment}

An instance is $\nu=(\nu_a)_{a\in[K]}$, where $\nu_a$ is a probability law
on $\R^d$ with mean
\[
  \mu_a:=\E_{\nu_a}[X]\in\R^d.
\]
The immutable family of true vectors is denoted
\[
  \muvtrue:=(\mu_a)_{a\in[K]}\in(\R^d)^K.
\]

Construct the environment before the algorithm runs by independent reward
stacks
\begin{equation}
  \bigl(\mathbf X_{a,n}\bigr)_{a\in[K],\,n\geq1},
  \qquad \mathbf X_{a,n}\sim\nu_a.
  \label{eq:reward-stacks}
\end{equation}
More precisely, the random \textbf{vectors} in \eqref{eq:reward-stacks}
are mutually independent over distinct pairs $(a,n)$, and, for a fixed
arm, its stack is i.i.d. This is the temporal and across-arm independence
of the ground-truth learning model. Initialization reveals
$\mathbf X_{a,1}$; the $r$-th total pull of arm $a$ reveals
$\mathbf X_{a,r}$. Thus, when round $t$ selects $A_t=a$, it reveals the
fresh vector
\begin{equation}
  \mathbf Y_t:=\mathbf X_{a,N_a(t)}.
  \label{eq:fresh-vector}
\end{equation}
Pre-generating stacks is only a coupling of the sequential experiment; it
does not make future rewards available to the algorithm.

The baseline environmental assumption is \textbf{marginal}
subgaussianity:
\begin{equation}
  \log\E_\nu
  \exp\{\lambda(X_{a,n}^{c}-\mu_a^{c})\}
  \leq \frac{\lambda^2\sigma^2}{2}
  \quad
  (a\in[K],\ n\geq1,\ c\in[d],\ \lambda\in\R).
  \label{eq:marginal-subgaussian}
\end{equation}
No joint law of $(X_{a,n}^{1},\ldots,X_{a,n}^{d})$ is imposed by
\eqref{eq:marginal-subgaussian}. In particular,
\eqref{eq:marginal-subgaussian} permits comonotone, anticomonotone, and
other dependent coordinate couplings within one reward vector.

\subsection{Optional environmental-coordinate assumption}

When explicitly invoked, \textbf{(IndCoord)} means that the $d$
coordinates of each $\mathbf X_{a,n}$ are independent. Together with
\eqref{eq:reward-stacks}, it gives independence across all triples
$(a,n,c)$. This is an extra assumption, not a consequence of marginal
subgaussianity or of the posterior covariance.

The stronger label \textbf{(IndGauss)} means
\begin{equation}
  \mathbf X_{a,n}\sim\mathcal N(\mu_a,\sigma^2 I_d)
  \quad\text{independently over }(a,n).
  \label{eq:indgauss}
\end{equation}
It implies (IndCoord) and \eqref{eq:marginal-subgaussian}.

\subsection{Algorithmic randomness and enlarged filtration}
\label{subsec:filtration}

Let the algorithmic seeds be mutually independent and independent of all
reward stacks:
\begin{itemize}
  \item $\mathbf G_{a,t}=(G_{a,t}^{c})_{c\in[d]}$, with all
    $G_{a,t}^{c}\stackrel{\rm iid}{\sim}\mathcal N(0,1)$ for sampling from the posterior;
  \item $U_t^{\rm br}\stackrel{\rm iid}{\sim}\Unif(0,1)$, for the
    Top-Two branch; and
  \item $U_t^{\rm sel}\stackrel{\rm iid}{\sim}\Unif(0,1)$, for uniform
    selection from a finite front.
\end{itemize}

Set $\Xi_t=(\{\mathbf G_{a,t}\}_{a\in[K]},U_t^{\rm br},U_t^{\rm sel})$.
The post-initialization history and the intra-round sigma fields are
\begin{equation}
  \begin{aligned}
    \mathcal F_0
      &:=\sigma(\mathbf X_{a,1}:a\in[K]),\\
    \mathcal G_t
      &:=\mathcal F_{t-1}\vee\sigma(\Xi_t),\\
    \mathcal F_t
      &:=\mathcal G_t\vee\sigma(A_t,\mathbf Y_t)\qquad(t\geq1).
  \end{aligned}
  \label{eq:filtration}
\end{equation}
Equivalently, $\mathcal F_t$ contains every reward observed through
formal round $t$ and every algorithmic seed used through that
round. In particular, $\mathcal F_{t-1}$ includes all \emph{past}
posterior, branch, and uniform-selection seeds. This is the required
enlarged filtration.

The posterior sample and $A_t$ are $\mathcal G_t$-measurable, while
\begin{equation}
  \mathcal L_\nu(\mathbf Y_t\mid\mathcal G_t)=\nu_{A_t}.
  \label{eq:fresh-law}
\end{equation}
Equation~\eqref{eq:fresh-law} is the precise adaptive-sampling statement:
conditional on the current randomized choice, the newly exposed stack
entry has the chosen arm's law and is fresh. The action is not generally
$\mathcal F_{t-1}$-measurable because the algorithm is randomized; its
conditional selection kernel is with respect to $\mathcal F_{t-1}$.

It is legitimate in a proof to fix a full reward-stack realization in a
measurable stack-good event and then expose the independent, fresh
algorithmic seeds. It is not legitimate to treat a future-stack event as
an ordinary $\mathcal F_{t-1}$-conditioning event.

\subsection{Empirical vectors and the Gaussian posterior}
\label{subsec:pseudo-posterior}

For every arm and formal time,
\begin{equation}
  \widehat\mu_a(t):=\frac1{N_a(t)}
        \sum_{n=1}^{N_a(t)}\mathbf X_{a,n},
  \qquad
  \muvhat{t}:=
        (\widehat\mu_a(t))_{a\in[K]}.
  \label{eq:empirical-mean}
\end{equation}
The empirical family is $\mathcal F_t$-measurable.

At the beginning of formal round $t$, Algorithm~1 uses the stipulated
product Gaussian posterior
\begin{equation}
  \Pi_t
  :=\bigotimes_{a=1}^K
    \mathcal N\!\left(
       \widehat\mu_a(t-1),\frac{\sigma^2}{N_a(t-1)}I_d
    \right),
  \label{eq:pseudo-posterior}
\end{equation}
sampled concretely as
\begin{equation}
  \widetilde\mu_a^{(t)}
  :=\widehat\mu_a(t-1)
    +\frac{\sigma}{\sqrt{N_a(t-1)}}\,\mathbf G_{a,t}.
  \label{eq:posterior-draw}
\end{equation}
Thus, conditional on $\mathcal F_{t-1}$, posterior draws are independent
across arms and across coordinates. This product structure is part of the
algorithmic inference/sampling model even if rewards have dependent
coordinates. Under only \eqref{eq:marginal-subgaussian},
\eqref{eq:pseudo-posterior} is not asserted to be the environment's
literal Bayesian conditional distribution; it is the prescribed Gaussian posterior used for the theoretical analysis.

The following three vector families have different types and roles:

\begin{center}
\begin{tabular}{@{}p{0.20\textwidth}p{0.17\textwidth}p{0.53\textwidth}@{}}
\toprule
Family & Notation & Meaning and measurability \\
\midrule
True means & $\muvtrue$ &
  Fixed instance parameter; neither an empirical statistic nor a sampled
  vector. \\
Empirical means & $\muvhat{t}$ &
  Reward-stack statistic, $\mathcal F_t$-measurable. \\
Posterior sample & $\muvtil{t}$ &
  Fresh random family from \eqref{eq:pseudo-posterior},
  $\mathcal G_t$-measurable; used only to choose $A_t$. \\
\bottomrule
\end{tabular}
\end{center}

No order relation in one row implies an order relation in either other
row. In particular, posterior-sample domination cannot be substituted for
empirical domination, and true Pareto domination cannot be substituted
for empirical domination.

\subsection{Pareto order, fronts, and the exact selection kernel}
\label{subsec:pareto-fronts}

For $x,y\in\R^d$, write
\begin{equation}
  x\prec y
  \quad\Longleftrightarrow\quad
  \bigl(x^c\le y^c\ \forall c\in[d]\bigr)
  \ \text{and}\
  \bigl(x^c<y^c\ \text{for some }c\in[d]\bigr).
  \label{eq:pareto-order}
\end{equation}
For a nonempty index set $I\subseteq[K]$ and a vector family
$v=(v_a)_{a\in[K]}$, define
\begin{equation}
  \Pareto(I;v):=
  \{i\in I:\nexists j\in I\text{ such that }v_i\prec v_j\}.
  \label{eq:pareto-set}
\end{equation}

At round $t$, form fronts from the \textbf{posterior sample}:
\begin{equation}
  \ttpPosteriorFront{1}{t}
  :=\Pareto([K];\muvtil{t}),
  \qquad
  I_2^{(t)}:=[K]\setminus\ttpPosteriorFront{1}{t},
  \label{eq:front1}
\end{equation}
and, when $I_2^{(t)}\ne\varnothing$,
\begin{equation}
  \ttpPosteriorFront{2}{t}
  :=\Pareto(I_2^{(t)};\muvtil{t}).
  \label{eq:front2}
\end{equation}
Let $B_t=\mathbf1\{U_t^{\rm br}\le\rho\}$. With a fixed increasing
enumeration of every finite set, use $U_t^{\rm sel}$ to select uniformly:
\begin{equation}
  A_t\sim
  \begin{cases}
  \Unif(\ttpPosteriorFront{1}{t}),
    &B_t=1,\\
  \Unif(\ttpPosteriorFront{2}{t}),
    &B_t=0,\ I_2^{(t)}\ne\varnothing,\\
  \Unif(\ttpPosteriorFront{1}{t}),
    &B_t=0,\ I_2^{(t)}=\varnothing .
  \end{cases}
  \label{eq:selection-kernel}
\end{equation}
The third line is the empty second front fallback. Membership in a
chosen front is an opportunity to be selected, not selection itself.

The true fronts are defined from $\muvtrue$:
\begin{equation}
  \mathcal S^\star
  :=\Pareto([K];\muvtrue),
  \qquad
  \mathcal S_2^\star
  :=\Pareto([K]\setminus\mathcal S^\star;\muvtrue)
  \quad\text{if }[K]\setminus\mathcal S^\star\ne\varnothing,
  \label{eq:true-fronts}
\end{equation}
and $\mathcal S_2^\star:=\varnothing$ otherwise.

The recommendation is deliberately a different object:
\begin{equation}
  \widehat{\mathcal S}^{\star}_t
  :=\Pareto([K];\muvhat{t}).
  \label{eq:recommendation}
\end{equation}
It is the empirical-mean Pareto set, equivalently, the Pareto set of
posterior means). Define
\begin{equation}
  \Err_T:=
  \{\widehat{\mathcal S}^{\star}_T\ne\mathcal S^\star\},
  \qquad
  e_T(\nu):=\Prob_\nu(\Err_T).
  \label{eq:supp:error-event}
\end{equation}

\subsection{Ground-truth gaps}
\label{subsec:gaps}

For $i,j\in[K]$, set
\begin{equation}
  \mathrm m(i,j):=\min_{c\in[d]}(\mu_j^c-\mu_i^c),
  \qquad
  \mathrm M(i,j):=\max_{c\in[d]}(\mu_i^c-\mu_j^c),
  \qquad (x)_+:=\max\{x,0\}.
  \label{eq:coordinate-gaps}
\end{equation}
For a suboptimal arm $i\notin\mathcal S^\star$, define
\begin{equation}
  \Delta_i:=\max_{j\in\mathcal S^\star}\mathrm m(i,j).
  \label{eq:suboptimal-gap}
\end{equation}
For $i\in\mathcal S^\star$, define
\begin{equation}
  \begin{aligned}
    \delta_i^+
    &:=\min_{j\in\mathcal S^\star\setminus\{i\}}
          \bigl(\mathrm M(i,j)\wedge\mathrm M(j,i)\bigr),\\
    \delta_i^-
    &:=\min_{j\in[K]\setminus\mathcal S^\star}
          \bigl((\mathrm M(j,i))_++\Delta_j\bigr),\\
    \Delta_i&:=\delta_i^+\wedge\delta_i^-,
  \end{aligned}
  \label{eq:optimal-gap}
\end{equation}
using $\min\varnothing=+\infty$. The separated-instance convention is
\begin{equation}
  \Delta_a>0\quad\text{for every }a\in[K].
  \label{eq:separated-instance}
\end{equation}
Based on these gaps we define the hardness of the problem instance:
\begin{equation}
    \label{eq:hardness}
    H(\nu):=\sum_{a=1}^{K}\frac{1}{\Delta_a^2}.
\end{equation}

\subsection{Exact boundary cases}
\label{subsec:boundary-cases}
\textbf{Case $K=1$}

For every $d$, all posterior and empirical vector families have one
member:
\[
  \ttpPosteriorFront{1}{t}=\{1\},\quad I_2^{(t)}=\varnothing,\quad
  A_t=1,\quad N_1(T)=T+1.
\]
Also $\mathcal S^\star=\widehat{\mathcal S}^{\star}_T=\{1\}$, so
\begin{equation}
  e_T(\nu)=0\quad\text{exactly for all }T\ge0.
  \label{eq:k1-boundary}
\end{equation}
Here $\mathcal S_2^\star=\varnothing$ and \eqref{eq:optimal-gap} gives
$\Delta_1=+\infty$.

\textbf{Case $d=1$}

The Pareto order is the usual strict scalar order. If $K\ge2$, every arm
is Pareto-optimal exactly when all true scalar means are equal. In that
all-Pareto situation the optimal-arm gap in \eqref{eq:optimal-gap} is
zero, so the positive-gap all-Pareto subclass is empty:
\begin{equation}
  d=1,\ K\ge2,\ \mathcal S^\star=[K]
  \quad\Longrightarrow\quad \min_a\Delta_a=0.
  \label{eq:d1-boundary}
\end{equation}
This does not make general separated one-dimensional instances empty,
nor does it remove the exact $K=1$ case.

\textbf{Case $K=2$}

For every posterior realization, each arm lies in one of the two
selectable fronts. If the samples are incomparable, then
$\ttpPosteriorFront{1}{t}=\{1,2\}$ and the residual is empty. If, for example,
$\widetilde\mu_1^{(t)}\prec\widetilde\mu_2^{(t)}$, then
$\ttpPosteriorFront{1}{t}=\{2\}$ and $\ttpPosteriorFront{2}{t}=\{1\}$; the reverse
case is symmetric. Consequently, conditionally on the posterior
realization, each arm's probability of selection over the branch and
uniform seeds is at least
\begin{equation}
  q:=\min\{\rho,1-\rho\}.
  \label{eq:k2-coverage}
\end{equation}
In the incomparable/fallback case the probability is $1/2$, which is at
least $q$.

\section{Theoretical Finite-time Analysis: Finite-Time Exponential Bounds}
\label{sec:exponential-bounds}

Throughout, $\log$ is the natural logarithm, $T$ is the number of
\textbf{formal rounds}, and the number of physically observed reward
vectors after those rounds is $K+T$; see \eqref{eq:pull-count}.

\subsection{Scope and contract references}
\label{subsec:exp-scope}

Both theorems of this section use exactly the following assumptions.
\begin{itemize}
  \item Reward vectors are independent across distinct stack locations
    $(a,n)$, and the scalar marginal condition
    \eqref{eq:marginal-subgaussian} holds. For a fixed arm-coordinate
    pair this supplies an i.i.d.\ scalar $\sigma$-subgaussian stack.
  \item Algorithmic seeds are independent of the stacks and fresh across
    rounds, as in \Cref{subsec:filtration}. Posterior draws are the
    stipulated Gaussian draws \eqref{eq:pseudo-posterior}--%
    \eqref{eq:posterior-draw}, and front selection, including the
    empty-residual fallback, is exactly
    \eqref{eq:front1}--\eqref{eq:selection-kernel}.
  \item We assume separated problem instances
    \eqref{eq:separated-instance}, with gaps defined by
    \eqref{eq:coordinate-gaps}--\eqref{eq:optimal-gap}. Recommendations
    and error events are precisely \eqref{eq:recommendation}--%
    \eqref{eq:supp:error-event}.
\end{itemize}

In particular, neither theorem assumes (IndCoord) or (IndGauss). We first record two probability lemmas and one deterministic front
lemma.

\subsection{A stack-uniform scalar concentration lemma}
\label{subsec:scalar-concentration}

For a fixed arm $a$, coordinate $c$, and stack length $n\geq1$, write
\[
  \overline X_{a,n}^{\,c}
  :=\frac1n\sum_{r=1}^n X_{a,r}^{c}.
\]

\begin{lemma}[One-stack, anytime-after-$n_0$ bound]
\label{lem:scalar-concentration}
Under assumption \eqref{eq:marginal-subgaussian}, for every
integer $n_0\geq1$ and every $\varepsilon>0$,
\begin{equation}
  \Prob_\nu\!\left(
    \exists n\geq n_0:
    \left|\overline X_{a,n}^{\,c}-\mu_a^c\right|>\varepsilon
  \right)
  \leq
  2\exp\!\left(-\frac{n_0\varepsilon^2}{2\sigma^2}\right).
  \label{eq:scalar-tail}
\end{equation}
\end{lemma}

\begin{proof}
Let
\[
  S_n:=\sum_{r=1}^n(X_{a,r}^c-\mu_a^c).
\]
Notice that
\[
\overline X_{a,n}^{\,c}-\mu_a^c=\frac1n\sum_{r=1}^n \left( X_{a,r}^{c}-\mu_a^c \right)=\frac{S_n}{n}.
\]
Therefore \[\left|\overline X_{a,n}^{\,c}-\mu_a^c\right|>\varepsilon\] is exactly the same as \[\left|S_n\right|>n\varepsilon\]
For every $\lambda\in\R$, marginal subgaussianity and independence along
this scalar stack make
\[
  M_n(\lambda)
  :=\exp\!\left\{\lambda S_n-\frac{\lambda^2\sigma^2 n}{2}\right\}
\]
a nonnegative supermartingale with $M_0(\lambda)=1$. 

\begin{itemize}
    \item Let $Z_n=X_{a,n}^{\,c}-\mu_a^c$. Then $S_n=S_{n-1}+Z_n$. Hence \[M_n(\lambda)
  =\exp\!\left\{\lambda S_{n-1}+\lambda Z_n-\frac{\lambda^2\sigma^2 n}{2}\right\}=M_{n-1}(\lambda)\exp\!\left\{\lambda Z_n-\frac{\lambda^2\sigma^2}{2}\right\}.\]
  \item Now take the conditional expectation \[\E\left[M_n(\lambda)|\mathcal F_{n-1}\right]=M_{n-1}(\lambda)\E\left[\exp\!\left\{\lambda Z_n-\frac{\lambda^2\sigma^2}{2}\right\}\right].\]
  \item By subgaussianity $\E\left[e^{\lambda Z_n}\right] \le e^{\lambda \sigma^2/2}$. Therefore \[\E\left[\exp\!\left\{\lambda Z_n-\frac{\lambda^2\sigma^2}{2}\right\}\right] \le 1. \quad \text{So} \quad \E\left[M_n(\lambda)|\mathcal F_{n-1}\right]\le M_{n-1}(\lambda).\] That is exactly the definition of a nonnegative supermartingale.
\end{itemize}

The intuition for $M_n(\lambda)$ is that $e^{\lambda S_n}$ becomes very large when $S_n$ is large. The subtraction term is a normalization term. Subgaussianity tells us exactly how much growth we should expect, and we subtract that amount.

Suppose the empirical mean is too large: $\overline X_{a,n}^{\,c}-\mu_a^c>\varepsilon$. Equivalently $S_n>n\varepsilon$. Take $\lambda=\varepsilon/\sigma^2$. Then
\[
  M_n(\lambda)=\exp\!\left(\frac{\varepsilon}{\sigma^2}\lambda-\frac{n\varepsilon^2}{2\sigma^2}\right)>
  \exp\!\left(\frac{n\varepsilon^2}{2\sigma^2}\right)
  \geq
  \exp\!\left(\frac{n_0\varepsilon^2}{2\sigma^2}\right).
\]
For a nonnegative supermartingale with $M_0(\lambda)=1$, Ville's inequality gives \[\Prob\left(\sup_nM_n\ge c\right)\le\frac{1}{c}\] so we do not need to calculate this probability for every $n$. Set $c=\exp\!\left(\frac{n_0\varepsilon^2}{2\sigma^2}\right)$. Then \[\Prob\left(\sup_nM_n\ge \exp\!\left(\frac{n_0\varepsilon^2}{2\sigma^2}\right)\right)\le\exp\!\left(\frac{-n_0\varepsilon^2}{2\sigma^2}\right).\] Hence \[\Prob\!\left(\exists n\geq n_0:\overline X_{a,n}^{\,c}-\mu_a^c>\varepsilon\right)\leq \exp\!\left(-\frac{n_0\varepsilon^2}{2\sigma^2}\right).\]

This gives the upper-tail bound in \eqref{eq:scalar-tail} without its factor $2$. Applying the same argument to $-S_n$, which is permitted because \eqref{eq:marginal-subgaussian} holds for every real $\lambda$, gives the lower-tail bound. Boole's inequality proves \eqref{eq:scalar-tail}.
\end{proof}

\Cref{lem:scalar-concentration} is uniform over the \textbf{stack index}
$n$. It is not a union over formal rounds, so adaptive arm choices do not
add a factor $T$.

\subsection{An adapted Bernoulli lower-tail lemma}
\label{subsec:bernoulli-tail}

\Cref{lem:bernoulli-tail} says: if a binary event has conditional probability at least \(p\) on every round, even after seeing the entire past, then the total number of successes is unlikely to be much smaller than \(pT\). The key point is that the successes do not need to be independent. Each round may depend on earlier outcomes. What matters is that, whatever history has occurred, the next success still has probability at least \(p\). This is useful later for counting adaptive arm selections. An arm’s selection on one round can affect what happens on later rounds, so treating the selections as independent Bernoulli trials would be unjustified. This Lemma provides the needed lower-tail bound directly for adapted trials.

\begin{lemma}[Conditional-mgf lower tail]
\label{lem:bernoulli-tail}
Let $(\mathcal A_t)_{t=0}^T$ be a filtration, let $Y_t\in\{0,1\}$ be
$\mathcal A_t$-measurable, and suppose that for a constant
$p\in(0,1]$,
\begin{equation}
  \Prob(Y_t=1\mid\mathcal A_{t-1})\geq p
  \quad\text{a.s.\ for every }t\in[T].
  \label{eq:bernoulli-floor}
\end{equation}
No independence between the $Y_t$'s is assumed. Then
\begin{equation}
  \Prob\!\left(\sum_{t=1}^T Y_t\leq\frac{pT}{2}\right)
  \leq
  \exp\!\left(-\frac{1-\log 2}{2}\,pT\right)
  \leq
  \exp\!\left(-\frac{pT}{8}\right).
  \label{eq:bernoulli-tail}
\end{equation}
\end{lemma}

\begin{proof}
For $\lambda>0$, the Bernoulli property and \eqref{eq:bernoulli-floor}
give
\begin{equation}
  \begin{aligned}
    \E\!\left[e^{-\lambda Y_t}\mid\mathcal A_{t-1}\right]
    &=1-(1-e^{-\lambda})
          \Prob(Y_t=1\mid\mathcal A_{t-1})\\
    &\leq 1-p(1-e^{-\lambda})
    \leq \exp\{-p(1-e^{-\lambda})\}.
  \end{aligned}
  \label{eq:bernoulli-mgf-step}
\end{equation}
Iterating conditional expectations yields
\begin{equation}
  \E\left[\exp\!\left(-\lambda\sum_{t=1}^T Y_t\right)\right]
  \leq\exp\{-pT(1-e^{-\lambda})\}.
  \label{eq:bernoulli-mgf}
\end{equation}
Markov's inequality on the event in \eqref{eq:bernoulli-tail}, followed
by $\lambda=\log 2$, gives
\[
  \begin{aligned}
    \Prob\!\left(\sum_{t=1}^T Y_t\leq\frac{pT}{2}\right)
    &\leq
    \exp\!\left\{\frac{\lambda pT}{2}
          -pT(1-e^{-\lambda})\right\}\\
    &=\exp\!\left(-\frac{1-\log2}{2}\,pT\right).
  \end{aligned}
\]
Finally $\log2<3/4$. Thus $(1-\log2)/2>1/8$, proving the second inequality in
\eqref{eq:bernoulli-tail}.
\end{proof}

This is a conditional-mgf argument for an adapted process, not a claim
that $\sum_tY_t$ has a binomial law.

\subsection{All-Pareto geometry and front eligibility}
\label{subsec:pareto-geometry}

\Cref{lem:protecting-coordinates} says that for every ordered pair of distinct arms \(a,b\), there is at least one objective where arm \(a\)’s true mean exceeds arm \(b\)’s by at least \(\Delta\). The word \textit{ordered} matters: the coordinate that shows \(a\) ahead of \(b\) can differ from the coordinate that shows \(b\) ahead of \(a\). Since all arms are Pareto-optimal, no arm can be strictly dominated by another. This gives a useful certificate for the later recommendation argument. If empirical means are within \(\Delta/4\) of the true means in every coordinate, the advantage in each protecting coordinate remains strictly positive. So neither arm appears to dominate the other, and the empirical Pareto set keeps both.

\begin{lemma}[Ordered protecting coordinates]
\label{lem:protecting-coordinates}
Suppose $K\geq2$, $\mathcal S^\star=[K]$, and
\begin{equation}
  \Delta:=\min_{a\in[K]}\Delta_a>0.
  \label{eq:all-pareto-gap}
\end{equation}
Then for every ordered pair $a\ne b$, there is a coordinate
$c(a,b)\in[d]$ such that
\begin{equation}
  \mu_a^{c(a,b)}-\mu_b^{c(a,b)}\geq\Delta.
  \label{eq:protecting-coordinate}
\end{equation}
\end{lemma}

\begin{proof}
When all arms are Pareto-optimal, the suboptimal set in
\eqref{eq:optimal-gap} is empty, so $\delta_a^-=\infty$ and
\[
  \Delta_a
  =\min_{j\ne a}\bigl(\mathrm M(a,j)\wedge\mathrm M(j,a)\bigr).
\]
A minimum over all competitors is no larger than any one of its terms. Thus, for $b\ne a$, this implies $\mathrm M(a,b)\geq\Delta_a\geq\Delta$, with the second inequality given by \eqref{eq:all-pareto-gap}. By the definition of $\mathrm M$ in \eqref{eq:coordinate-gaps}, its maximum is attained by a coordinate satisfying \eqref{eq:protecting-coordinate}.
\end{proof}

\Cref{lem:poset-eligibility} is a fact about finite partially ordered sets. It says that if \(H\) is an antichain (no two distinct elements of \(H\) are comparable) and \(L=V\setminus H\) is nonempty, then \(L\) contains an element in either the first Pareto front or the second Pareto front. Here, \(\Pareto(W)\) means the maximal elements of \(W\): elements for which there is no strictly larger element in \(W\).
The proof uses a two-case strategy:
\begin{enumerate}
    \item If the first front already contains an element of \(L\), we are done.
    \item Otherwise, the entire first front lies in \(H\). The antichain condition then forces \(H\) to be exactly the first front. Removing it leaves \(L\), and the nonempty finite set \(L\) must have a maximal element, which is on the second front of \(V\).
\end{enumerate}
In the exponential-bounds proof, the selectable arms come from the first or second Pareto front of a posterior sample. \Cref{lem:poset-eligibility} helps ensure that when some arms are “low” (in \(L\)), at least one low arm is on a selectable front. The lemma guarantees eligibility; the later argument separately accounts for the probability of actually selecting that arm.

\begin{lemma}[Finite-poset low-arm eligibility]
\label{lem:poset-eligibility}
Let $V$ be a nonempty finite set with a strict partial order $\prec$,
let $H\subseteq V$ be an antichain, and set $L:=V\setminus H$. If
$L\ne\varnothing$, then
\begin{equation}
  L\cap\bigl(\Pareto(V)
        \cup\Pareto(V\setminus\Pareto(V))
        \bigr)\ne\varnothing,
  \label{eq:poset-eligibility}
\end{equation}
where the second Pareto set is taken only when its argument is
nonempty.
\end{lemma}

\begin{proof}
Let $F_1:=\Pareto(V)$. If $L\cap F_1\ne\varnothing$,
\eqref{eq:poset-eligibility} is immediate. Assume instead that the first front contains no element of \(L\) such that every element of \(F_1\) must be in \(H\): \[L\cap F_1=\varnothing, \quad \text{so} \quad F_1\subseteq H.\] We claim that $H\subseteq F_1$. If not, choose $h\in H\setminus F_1$. Because the order is finite, repeatedly moving from a nonmaximal element
to a strictly larger one terminates at some $m\in F_1$ with $h\prec m$.
The antichain property forces $m\notin H$, hence $m\in L$, contradicting
$F_1\subseteq H$. Therefore $F_1=H$. Its residual is exactly the
nonempty set $L$, whose finite strict partial order has a maximal
element. That element belongs to $\Pareto(L)$, proving
\eqref{eq:poset-eligibility}.
\end{proof}

\subsection{All-Pareto exponential theorem}
\label{subsec:all-pareto-theorem}

\subsubsection{Fixed-stack disintegration}
Think of the experiment as being driven by two independent sources of randomness:
\begin{enumerate}
    \item Reward stacks: for each arm, a pre-existing sequence of possible rewards.
    \item Algorithmic seeds: the fresh randomness used to draw posterior samples and choose arms.
\end{enumerate}
The proof first freezes the entire collection of reward stacks at some realization \(x\). Once those rewards are fixed, the only randomness left is the algorithm’s seeds. The proof can then ask: for this particular reward environment, how likely are the algorithm’s random choices to lead to an error?

The main challenge is adaptivity. An arm’s pull count depends on earlier rewards and algorithmic choices, so the final empirical mean uses a random number of rewards from that arm’s stack. It is not generally valid to treat that random count as fixed, or to claim that the observed rewards are an independent sample of a fixed size.

There is another tempting but problematic move: condition on the event that every reward-stack average is well behaved, then argue round by round as if the usual reward law still applied. That event constrains rewards at future, not-yet-observed stack positions. It is therefore not an ordinary history event; conditioning on it inside the original time filtration can change the conditional distribution of future rewards. The fixed-stack approach avoids that problem:
\begin{itemize}
    \item It handles the reward values through a stack-only good event and a concentration bound that holds simultaneously for every stack length \(n\ge n_0\).
    \item With the stacks fixed, it analyzes the algorithm’s adaptive behavior using only the fresh algorithmic randomness.
    \item It combines the two bounds by integrating over the stacks.
\end{itemize}
This section provides a clean bridge between the reward concentration argument and the adaptive selection argument, without making an unsupported conditional-independence claim.

Let $\mathfrak X=(X_{a,n})_{a,n}$ be the complete reward-stack family from
\eqref{eq:reward-stacks}, and let $\mathfrak Z=(\Xi_t)_{t\geq1}$ be the
complete algorithm-seed family from
\Cref{subsec:filtration}. Their joint law is the product law
\begin{equation}
  \Prob_\nu=\Prob_{\mathfrak X}\otimes\Prob_{\mathfrak Z}.
  \label{eq:product-law}
\end{equation}
For a fixed stack realization $x$, initialization observations and every
reward revealed after an action are fixed. The trajectory through round
$t$ is then a measurable function of $x$ and $\Xi_1,\ldots,\Xi_t$.
Define the seed-history sigma field
\begin{equation}
  \mathcal A_t^x:=\sigma(\Xi_1,\ldots,\Xi_t).
  \label{eq:seed-history}
\end{equation}
In particular, before round $t$, the counts, empirical centers, and any
sets defined from them are $\mathcal A_{t-1}^x$-measurable. The current
$\Xi_t$ remains independent of $\mathcal A_{t-1}^x$.

If $E$ is an event and $E_x$ is its seed section
at stack $x$, Fubini's theorem gives, for every stack-only measurable
set $S$,
\begin{equation}
  \Prob_\nu(E\cap\{\mathfrak X\in S\})
  =\int_{\{x\in S\}}\Prob_{\mathfrak Z}(E_x)\,
        \Prob_{\mathfrak X}(dx).
  \label{eq:fubini-disintegration}
\end{equation}
Here \(E_x\) means the event \(E\) evaluated with the reward stacks fixed to \(x\). In words: the probability of the event together with a stack condition is the average, over stacks satisfying that condition, of the event’s seed-only probability.

We will prove a seed-probability bound for each good $x$, and only then
integrate it using \eqref{eq:fubini-disintegration}. This is not
conditioning on a future-stack event inside the original filtration
$(\mathcal F_t)$.

\begin{theorem}[All-Pareto finite-time bound]
\label{thm:supp:all-pareto}
Assume $K\geq2$, $d\geq1$, marginal subgaussianity
\eqref{eq:marginal-subgaussian}, and $\mathcal S^\star=[K]$ with the
positive gap \eqref{eq:all-pareto-gap}. Let
\begin{equation}
  q:=\min\{\rho,1-\rho\},\qquad
  \alpha:=\frac{q}{16K^2},\qquad
  \gamma:=\frac{q}{2K},
  \label{eq:sup:ap-constants}
\end{equation}
and, for an integer formal horizon $T$, set
\begin{equation}
  n_0(T):=\lceil\alpha T\rceil,
  \qquad
  m_\star:=
  \left\lceil
    \frac{128\sigma^2}{\Delta^2}\log(4Kd)
  \right\rceil,
  \qquad
  T_0:=
  \left\lceil\frac{m_\star}{\alpha}\right\rceil
  =\left\lceil\frac{16K^2m_\star}{q}\right\rceil.
  \label{eq:sup:ap-horizon}
\end{equation}
Then, for every integer $T\geq T_0$,
\begin{equation}
  e_T(\nu)
  \leq
  2Kd\exp\!\left(-\frac{\alpha T\Delta^2}{32\sigma^2}\right)
  +\exp\!\left(-\frac{\gamma T}{8}\right).
  \label{eq:sup:ap-bound}
\end{equation}
Equivalently, with
\begin{equation}
  C_{\rm AP}=2Kd+1,\qquad
  c_{\rm AP}
  =\min\!\left\{
     \frac{q\Delta^2}{512K^2\sigma^2},\,
     \frac{q}{16K}
   \right\}>0,
  \label{eq:ap-simplified-constants}
\end{equation}
the right-hand side of \eqref{eq:sup:ap-bound} is at most
$C_{\rm AP}e^{-c_{\rm AP}T}$.
\end{theorem}

Because $K\geq2$, $d\geq1$, and $\sigma,\Delta>0$, one has
$\log(4Kd)>0$, $m_\star\geq1$, and all constants in
\eqref{eq:sup:ap-constants}--\eqref{eq:sup:ap-horizon} are finite and strictly
positive. In particular, \eqref{eq:sup:ap-horizon} is an explicit
large-horizon condition rather than an unstated asymptotic qualifier. It
also implies
\begin{equation}
  T_0\geq\left\lceil\frac1\alpha\right\rceil
  =\left\lceil\frac{8K}{\gamma}\right\rceil,
  \label{eq:ap-horizon-lower}
\end{equation}
but its additional $m_\star$ factor is essential for the posterior-box
probability in \eqref{eq:posterior-box-prob}.

\begin{proof}
Fix $T\geq T_0$ and abbreviate $n_0=n_0(T)$. Since
$\alpha T\geq m_\star$, we have
\begin{equation}
  n_0\geq m_\star.
  \label{eq:n0-geq-mstar}
\end{equation}

\textbf{1. Stack event.} Define the stack-only good event
\begin{equation}
  \mathsf S_T:=
  \bigcap_{a=1}^K\bigcap_{c=1}^d
  \left\{
    \left|\overline X_{a,n}^{\,c}-\mu_a^c\right|
    \leq\frac{\Delta}{4}
    \quad\text{for every }n\geq n_0
  \right\}.
  \label{eq:stack-event}
\end{equation}
\Cref{lem:scalar-concentration} with $\varepsilon=\Delta/4$, followed by
Boole's inequality over the $Kd$ arm-coordinate pairs, gives
\begin{equation}
  \begin{aligned}
    \Prob_\nu(\mathsf S_T^c)
    &\leq
    2Kd\exp\!\left(-\frac{n_0\Delta^2}{32\sigma^2}\right)\\
    &\leq
    2Kd\exp\!\left(-\frac{\alpha T\Delta^2}{32\sigma^2}\right).
  \end{aligned}
  \label{eq:stack-event-bound}
\end{equation}
This term accounts for the chance that the stacks are bad. Even on a good stack, an arm could still end with fewer than \(n_0(T)\) pulls, so its empirical mean might not be covered by the good event. That remaining possibility is what the second error term handles.

\textbf{2. Posterior box on a fixed good stack.} Fix a realization
$x\in\mathsf S_T$, and work only under the independent seed law in
\eqref{eq:product-law}. At the beginning of round $t$, define
\begin{equation}
  \begin{aligned}
    H_{t-1}&:=\{a:N_a(t-1)\geq n_0\},\\
    L_{t-1}&:=[K]\setminus H_{t-1},\\
    D_{t-1}&:=\sum_{a=1}^K(n_0-N_a(t-1))_+ .
  \end{aligned}
  \label{eq:high-low-deficit}
\end{equation}
Where deficit $D_{t-1}$ measures how many additional pulls are needed, in total, to bring all low arms up to the threshold. These are $\mathcal A_{t-1}^x$-measurable. Define the current posterior
box event, which constrains only high arms,
\begin{equation}
  \mathsf{Box}_t:=
  \bigcap_{h\in H_{t-1}}
  \left\{
   \left\|\widetilde\mu_h^{(t)}
         -\widehat\mu_h(t-1)\right\|_\infty
   \leq\frac{\Delta}{8}
  \right\}.
  \label{eq:box-event}
\end{equation}
For a high arm $h$ and coordinate $c$, \eqref{eq:posterior-draw} and the
Gaussian tail bound $\Prob(|G|>z)\leq2e^{-z^2/2}$ imply
\begin{equation}
  \begin{aligned}
    &\Prob_{\mathfrak Z}\!\left(
     \left|\widetilde\mu_h^{(t),c}
           -\widehat\mu_h^c(t-1)\right|>\frac{\Delta}{8}
      \,\middle|\,\mathcal A_{t-1}^x\right)\\
    &\hspace{38mm}\leq
    2\exp\!\left(-\frac{N_h(t-1)\Delta^2}{128\sigma^2}\right)
    \leq
    2\exp\!\left(-\frac{n_0\Delta^2}{128\sigma^2}\right).
  \end{aligned}
  \label{eq:box-coordinate-tail}
\end{equation}
There are at most $Kd$ such coordinates. By a Gaussian union bound and
\eqref{eq:n0-geq-mstar},
\begin{equation}
  \begin{aligned}
    \Prob_{\mathfrak Z}(\mathsf{Box}_t^c\mid\mathcal A_{t-1}^x)
    &\leq
    2Kd\exp\!\left(-\frac{n_0\Delta^2}{128\sigma^2}\right)\\
    &\leq
    2Kd\,e^{-\log(4Kd)}
    =\frac12.
  \end{aligned}
  \label{eq:box-union-bound}
\end{equation}
Thus
\begin{equation}
  \Prob_{\mathfrak Z}(\mathsf{Box}_t\mid\mathcal A_{t-1}^x)\geq\frac12.
  \label{eq:posterior-box-prob}
\end{equation}
If $H_{t-1}$ is empty, \eqref{eq:posterior-box-prob} also holds because
then the intersection defining \(\mathsf{Box}_t\) is empty of constraints, so the event holds automatically.

On $\mathsf{Box}_t$, high arms are a sampled antichain. Indeed, take
distinct $h,g\in H_{t-1}$ and use the protecting coordinate $c(h,g)$
from \Cref{lem:protecting-coordinates}. The fixed stack belongs to
\eqref{eq:stack-event}, while \eqref{eq:box-event} controls both
posterior perturbations, so
\begin{equation}
  \begin{aligned}
    \widetilde\mu_h^{(t),c(h,g)}
    -\widetilde\mu_g^{(t),c(h,g)}
    &\geq
    \Delta-\frac{\Delta}{4}-\frac{\Delta}{4}
           -\frac{\Delta}{8}-\frac{\Delta}{8}=\frac{\Delta}{4}>0.
  \end{aligned}
  \label{eq:box-antichain}
\end{equation}
Applying the same calculation to the reverse ordered pair $(g,h)$ gives a
coordinate where the reverse difference is at least $\Delta/4>0$. Hence
neither sampled vector strictly Pareto-dominates the other. 

\textbf{3. A conditional low-arm selection probability.} Suppose
$D_{t-1}>0$, equivalently $L_{t-1}\ne\varnothing$. On $\mathsf{Box}_t$,
\Cref{lem:poset-eligibility} applied to the posterior sample says that
at least one low arm lies in
$\ttpPosteriorFront{1}{t}\cup\ttpPosteriorFront{2}{t}$.

For every posterior realization satisfying that conclusion, the branch
and uniform-selection seeds select a low arm with probability at least
$q/K$:
\begin{itemize}
  \item if a low arm is in $\ttpPosteriorFront{1}{t}$, use the first branch,
    whose probability is $\rho$, and its uniform factor
    $1/|\ttpPosteriorFront{1}{t}|\geq1/K$;
  \item if a low arm is in $\ttpPosteriorFront{2}{t}$, use the second branch,
    whose probability is $1-\rho$, and
    $1/|\ttpPosteriorFront{2}{t}|\geq1/K$.
\end{itemize}
If the residual is empty, every arm is in the first front and the first
bullet still applies. Let
\begin{equation}
  Z_t:=\mathbf 1\{D_{t-1}>0,\ A_t\in L_{t-1}\}.
  \label{eq:low-arm-indicator}
\end{equation}
So \(Z_t=1\) exactly when the algorithm selects a low arm while there is still a deficit to work on. The current Gaussian, branch, and uniform seeds are independent of
$\mathcal A_{t-1}^x$. Averaging the preceding $q/K$ bound over the
box event and then using \eqref{eq:posterior-box-prob} yields the conditional bound
\begin{equation}
  \Prob_{\mathfrak Z}(Z_t=1\mid\mathcal A_{t-1}^x)
  \geq
  \frac{q}{K}\cdot\frac12\,
  \mathbf1\{D_{t-1}>0\}
  =\gamma\,\mathbf1\{D_{t-1}>0\}.
  \label{eq:low-arm-floor}
\end{equation}

\textbf{4. Deficit tail without a binomial-independence assertion.} A
selected low arm has count strictly below $n_0$, so one pull reduces its
deficit by exactly one. Consequently,
\begin{equation}
  D_t=D_{t-1}-Z_t
  \quad\text{for all }t\in[T].
  \label{eq:deficit-recursion}
\end{equation}
For analysis only, enlarge the seed space by independent
$V_t\sim\Unif(0,1)$, independent of all original variables, and set
\begin{equation}
  Y_t:=
  \begin{cases}
  Z_t, & D_{t-1}>0,\\
  \mathbf1\{V_t\leq\gamma\}, & D_{t-1}=0.
  \end{cases}
  \label{eq:auxiliary-process}
\end{equation}
With the augmented seed filtration, \eqref{eq:low-arm-floor} says that
the first line has conditional success probability at least $\gamma$,
and the second has conditional success probability exactly $\gamma$.
Hence \Cref{lem:bernoulli-tail} applies to $(Y_t)$ with $p=\gamma$. This
auxiliary extension changes no probability of an event in the original
experiment.

The initial-deficit arithmetic is strict and is the point at which the
ceiling must not be weakened incorrectly:
\begin{equation}
  D_0
  =K\bigl(\lceil\alpha T\rceil-1\bigr)
  <K\alpha T
  =\frac{qT}{16K}
  =\frac{\gamma T}{8}
  <\frac{\gamma T}{2}.
  \label{eq:initial-deficit}
\end{equation}
The strict first inequality follows from $\lceil x\rceil-1<x$ for every
$x>0$. 

If $D_T>0$, then $D_{t-1}>0$ for every $t\leq T$, by
\eqref{eq:deficit-recursion}. Therefore $Y_t=Z_t$ throughout, and
integrality together with \eqref{eq:deficit-recursion} gives
\begin{equation}
  \sum_{t=1}^T Y_t
  =D_0-D_T
  <D_0
  <\frac{\gamma T}{2}.
  \label{eq:deficit-sum}
\end{equation}
\Cref{lem:bernoulli-tail} now proves, for every fixed $x\in\mathsf S_T$,
\begin{equation}
  \Prob_{\mathfrak Z}(D_T>0)\leq e^{-\gamma T/8}.
  \label{eq:deficit-tail}
\end{equation}

\textbf{5. Recommendation and integration over stacks.} On
$\{D_T=0\}$, every final total-pull count is at least $n_0$. For the
fixed good stack,
\begin{equation}
  \left\|\widehat\mu_a(T)-\mu_a\right\|_\infty\leq\frac{\Delta}{4}
  \quad(a\in[K]).
  \label{eq:final-empirical-error}
\end{equation}
For every ordered pair $a\ne b$, use the protecting coordinate from
\Cref{lem:protecting-coordinates} once more. Equation
\eqref{eq:final-empirical-error} gives
\begin{equation}
  \widehat\mu_a^{\,c(a,b)}(T)
  -\widehat\mu_b^{\,c(a,b)}(T)
  \geq \Delta-\frac{\Delta}{4}-\frac{\Delta}{4}
  =\frac{\Delta}{2}>0.
  \label{eq:final-empirical-gap}
\end{equation}
The reverse ordered pair supplies a coordinate with the reverse strict
inequality. Thus the empirical vectors are pairwise incomparable, so the
bare empirical recommendation in \eqref{eq:recommendation} is
\begin{equation}
  \widehat{\mathcal S}^{\star}_T=[K]=\mathcal S^\star.
  \label{eq:ap-recommendation-correct}
\end{equation}
This is recommendation correctness, not a conclusion about posterior
fronts.

It follows from \eqref{eq:deficit-tail}--%
\eqref{eq:ap-recommendation-correct} that
$\Prob_{\mathfrak Z}(\mathsf{Err}_{T,x})\leq e^{-\gamma T/8}$ for every
good stack $x$. Applying the fixed-stack integration
\eqref{eq:fubini-disintegration} and then
\eqref{eq:stack-event-bound} gives
\[
  \begin{aligned}
    e_T(\nu)
    &\leq \Prob_\nu(\mathsf S_T^c)
       +\int_{\{x\in\mathsf S_T\}}
         \Prob_{\mathfrak Z}(\mathsf{Err}_{T,x})\,
         \Prob_{\mathfrak X}(dx)\\
    &\leq
    2Kd\exp\!\left(-\frac{\alpha T\Delta^2}{32\sigma^2}\right)
    +e^{-\gamma T/8},
  \end{aligned}
\]
which is \eqref{eq:sup:ap-bound}. Equation
\eqref{eq:ap-simplified-constants} follows by substituting
\eqref{eq:sup:ap-constants} and taking the smaller exponent.
\end{proof}

\subsubsection*{Formal-round and physical-pull form}

\Cref{thm:supp:all-pareto} is stated in formal rounds. If $P=K+T$ is instead
the physical number of observed reward vectors, then for every integer
$P\geq K+T_0$,
\begin{equation}
  e^{\rm phys}_{P}(\nu):=e_{P-K}(\nu)
  \leq
  2Kd\exp\!\left(-\frac{\alpha(P-K)\Delta^2}{32\sigma^2}\right)
  +\exp\!\left(-\frac{\gamma(P-K)}8\right).
  \label{eq:ap-physical-form}
\end{equation}

\subsection{Exact two-arm finite-time exponential theorem}
\label{subsec:two-arm-theorem}

The two-arm argument does not need a posterior box or a stack
disintegration. Its selection floor is unconditional and exact at the
level of the stated lower bound.

\begin{theorem}[Separated $K=2$ bound]
\label{thm:sup:two-arm}
Assume $K=2$, $d\geq1$, marginal subgaussianity
\eqref{eq:marginal-subgaussian}, and
\begin{equation}
  \Delta:=\min\{\Delta_1,\Delta_2\}>0.
  \label{eq:sup:two-arm-gap}
\end{equation}
Let $q=\min\{\rho,1-\rho\}$. For every formal horizon $T\in\N_0$,
\begin{equation}
  e_T(\nu)
  \leq
  4d\exp\!\left(-\frac{qT\Delta^2}{64\sigma^2}\right)
  +2\exp\!\left(-\frac{qT}{8}\right).
  \label{eq:sup:two-arm-bound}
\end{equation}
\end{theorem}

\begin{proof}
For a realized posterior pair at a round, exactly one of the following
cases holds.
\begin{enumerate}
  \item If $\widetilde\mu_1^{(t)}\prec\widetilde\mu_2^{(t)}$, then
    $\ttpPosteriorFront{1}{t}=\{2\}$ and
    $\ttpPosteriorFront{2}{t}=\{1\}$. The
    conditional selection probabilities of arms $1,2$, over branch and
    uniform seeds, are $1-\rho,\rho$, respectively.
  \item The reverse strict domination is symmetric.
  \item If neither vector strictly dominates the other, then
    $\ttpPosteriorFront{1}{t}=\{1,2\}$ and the residual is empty. The fallback
    in \eqref{eq:selection-kernel} selects each arm with probability
    $1/2$, regardless of the branch.
\end{enumerate}
The third case includes exact posterior ties; strict Pareto domination
does not regard equal vectors as comparable. Since $q\leq1/2$, every arm
has conditional selection probability at least $q$ in all three cases.
Integrating over the posterior draw therefore gives, in the enlarged
filtration \eqref{eq:filtration},
\begin{equation}
  \Prob_\nu(A_t=a\mid\mathcal F_{t-1})\geq q
  \quad(a\in\{1,2\},\ t\geq1).
  \label{eq:two-arm-floor}
\end{equation}

For $T\geq1$, let
\begin{equation}
  C_a(T):=\sum_{t=1}^T\mathbf1\{A_t=a\}=N_a(T)-1.
  \label{eq:arm-count}
\end{equation}
Apply \Cref{lem:bernoulli-tail} to $Y_t=\mathbf1\{A_t=a\}$, using
\eqref{eq:two-arm-floor}, to obtain
\begin{equation}
  \Prob_\nu\!\left(C_a(T)<\frac{qT}{2}\right)
  \leq e^{-qT/8}.
  \label{eq:arm-count-tail}
\end{equation}
This is again an adapted conditional-mgf conclusion, not an independent
binomial calculation.

Put $n_2:=\lceil qT/2\rceil$, and define the two-arm stack event
\begin{equation}
  \mathsf S_T^{(2)}:=
  \bigcap_{a=1}^2\bigcap_{c=1}^d
  \left\{
   \left|\overline X_{a,n}^{\,c}-\mu_a^c\right|
   \leq\frac{\Delta}{4}
   \quad\text{for every }n\geq n_2
  \right\}.
  \label{eq:two-arm-stack-event}
\end{equation}
\Cref{lem:scalar-concentration} and Boole's inequality yield
\begin{equation}
  \begin{aligned}
    \Prob_\nu((\mathsf S_T^{(2)})^c)
    &\leq4d\exp\!\left(-\frac{n_2\Delta^2}{32\sigma^2}\right)\\
    &\leq4d\exp\!\left(-\frac{qT\Delta^2}{64\sigma^2}\right).
  \end{aligned}
  \label{eq:two-arm-stack-bound}
\end{equation}

For completeness, the required recommendation certificate is proved
directly from the contract gaps.
\begin{itemize}
  \item If both arms are truly Pareto-optimal,
    \eqref{eq:optimal-gap} reduces to
    \[
      \Delta_a=\mathrm M(a,b)\wedge\mathrm M(b,a)
      \quad(\{a,b\}=\{1,2\}).
    \]
    Hence both ordered pairs have a coordinate gap at least $\Delta$.
    Empirical errors at most $\Delta/4$ leave a strict gap at least
    $\Delta/2>0$ in each direction, so the two empirical vectors remain
    incomparable.
  \item If $i$ is the unique suboptimal arm and $j$ the unique Pareto
    arm, then \eqref{eq:suboptimal-gap} says $\Delta_i=\mathrm m(i,j)
    \geq\Delta>0$. Thus every coordinate satisfies
    $\mu_j^c-\mu_i^c\geq\Delta$. Empirical errors at most $\Delta/4$
    give
    \[
      \widehat\mu_j^c(T)-\widehat\mu_i^c(T)
      \geq\frac{\Delta}{2}>0
      \quad\text{for every }c,
    \]
    so $i$ is empirically dominated by $j$.
\end{itemize}
These are all possible true Pareto sets for two arms. Therefore, on
$\mathsf S_T^{(2)}$ and the two count events $C_a(T)\geq qT/2$,
integrality gives
\begin{equation}
  C_a(T)\geq\left\lceil\frac{qT}{2}\right\rceil=n_2,
  \qquad
  N_a(T)=1+C_a(T)\geq n_2+1>n_2.
  \label{eq:two-arm-count-floor}
\end{equation}
Thus the stack event applies to both final empirical means and the bare
empirical recommendation \eqref{eq:recommendation} is correct. A union
bound with \eqref{eq:arm-count-tail} and
\eqref{eq:two-arm-stack-bound} proves \eqref{eq:sup:two-arm-bound} for
$T\geq1$.

For $T=0$, $e_0(\nu)\leq1$, whereas the right-hand side of
\eqref{eq:sup:two-arm-bound} is $4d+2\geq1$; hence
\eqref{eq:sup:two-arm-bound} also holds at that endpoint.
\end{proof}

The total-pull version has $P=2+T$: for every integer $P\geq2$,
\begin{equation}
  e^{\rm phys}_{P}(\nu):=e_{P-2}(\nu)
  \leq
  4d\exp\!\left(-\frac{q(P-2)\Delta^2}{64\sigma^2}\right)
  +2\exp\!\left(-\frac{q(P-2)}8\right).
  \label{eq:two-arm-physical-form}
\end{equation}

\section{General Polynomial Bound}
\label{sec:universal-polynomial}

Throughout this section, $T$ denotes formal rounds; the corresponding
physical number of observed reward vectors is $K+T$, as fixed in
\eqref{eq:pull-count}.

\begin{definition}[Mixed-front separated instance]
\label{def:sup:poly-mixed-instance}
An instance is \emph{mixed-front separated} when $K\geq3$, $d\geq1$,
$\sigma>0$, $\rho\in(0,1)$, the model and algorithm obey
\Cref{sec:model-contract}, and
\begin{equation}
  \varnothing\ne \ttpTrueFront\subsetneq[K],
  \qquad
  \Delta_a>0\quad(a\in[K]).
  \label{eq:sup:poly-mixed-scope}
\end{equation}
\end{definition}

The first condition in \eqref{eq:sup:poly-mixed-scope} makes the suboptimal set
nonempty; the nonemptiness of $\ttpTrueFront$ itself follows from finiteness
and \eqref{eq:true-fronts}.  The next lemma records the resulting finiteness
needed by the direct theorem.

\begin{lemma}[Finite gaps in the mixed-front scope]
\label{lem:poly-finite-gaps}
Under \Cref{def:sup:poly-mixed-instance}, $0<\Delta_a<\infty$ for every
$a\in[K]$.  Consequently,
\begin{equation}
  \Delta_0:=\min_{a\in[K]}\Delta_a\in(0,\infty).
  \label{eq:poly-gap-min}
\end{equation}
\end{lemma}

\begin{proof}
Positivity is part of \eqref{eq:sup:poly-mixed-scope}.  If
$a\notin\ttpTrueFront$, then \eqref{eq:suboptimal-gap} is a maximum over the
nonempty finite set $\ttpTrueFront$ of finite values, so $\Delta_a<\infty$.
If $i\in\ttpTrueFront$, the nonempty set
$[K]\setminus\ttpTrueFront$ indexes the minimum defining $\delta_i^-$ in
\eqref{eq:optimal-gap}; every one of its terms is finite by the preceding
argument.  Thus $\Delta_i\leq\delta_i^-<\infty$.  Taking a finite minimum
proves \eqref{eq:poly-gap-min}.
\end{proof}

\subsection{Notation, constants, and explicit horizon}
\label{subsec:poly-notation}

Let $\Phi$, $\overline\Phi$, and $\varphi$ denote respectively the standard
normal distribution function, survival function, and density:
\[
  \Phi(x)=\Prob(Z\leq x),\qquad
  \overline\Phi(x)=1-\Phi(x),\qquad
  \varphi(x)=\frac{e^{-x^2/2}}{\sqrt{2\pi}},
  \qquad Z\sim\mathcal N(0,1).
\]
We use $(x)_+=\max\{x,0\}$ and $x\wedge y=\min\{x,y\}$. Fix deterministic tie-breaking: $\iota(N)$ is the smallest index attaining
$\min_iN_i$; $\iota_t:=\iota(N(t-1))$.  For ordered $a\ne j$, let $c(a,j)$
be the smallest coordinate attaining
$\min_{c\in[d]}(\mu_j^c-\mu_a^c)$.  These choices depend only on their
arguments and are respectively $\ttpHistory{t-1}$-measurable and fixed by
the instance.

For a mixed-front separated instance, put
\begin{equation}
  \begin{gathered}
    \Gamma:=\max_{a,j\in[K]}(\ttpGapMin(a,j))_+,\qquad
    g:=1+\frac{\Gamma}{\Delta_0},\qquad
    \beta:=\frac{1}{8(K-1)g^2},\\
    c_{\rm sel}:=\frac{\rho}{K}(3\sqrt\pi)^{-(K-1)},\qquad
    \theta:=\frac{8\sigma^2}{\Delta_0^2},\qquad
    \ell:=\log(4Kd),\\
    D_\star:=4Kd\left(\frac{9\theta\beta}{8}\right)^2,\qquad
    \ell_\star:=(\log D_\star)_+,\qquad
    R_{\rm hit}:=\frac{9K\theta\beta}{4c_{\rm sel}}.
  \end{gathered}
  \label{eq:poly-constants}
\end{equation}
These constants are finite, $\beta\in(0,1/16]$, and
\[
  \frac{32}{\beta}\geq512,
  \qquad
  \frac{8(K-1)}3\geq\frac{16}{3}>1.
\]
Define $\mathcal L(x):=8\log(4\max\{e,x\})$ for $x>0$ and
\begin{equation}
  \begin{aligned}
  \lambda_0:=&\max\biggl\{
    \frac{2}{\theta\beta},\
    \frac{8\ell}{\beta},\
    \frac{16\ell_\star}{\beta},\
    \frac{64}{\beta}\log\frac{64}{\beta},\
    \frac{16(K-1)}3\log\frac{16(K-1)}3,
    \mathcal L(R_{\rm hit}),\
    \mathcal L\!\left(\frac{16\beta}{c_{\rm sel}}\right),\
    4\log\!\left(\max\left\{e,\frac{16\log K}{c_{\rm sel}}\right\}\right)
  \biggr\},\\
  T_0^{\rm poly}:=&\left\lceil e^{\lambda_0}\right\rceil .
  \end{aligned}
  \label{eq:poly-horizon}
\end{equation}
For the proof, fix $T\geq T_0^{\rm poly}$ and abbreviate
\[
  \lambda:=\log T,\qquad
  \delta_T:=T^{-\beta},\qquad
  \eta_T:=T^{-\beta}.
\]
The distinct names $\delta_T$ and $\eta_T$ distinguish the stack-tube and
final-concentration accounting.

\subsection{Analytic lemmas}
\label{subsec:poly-analytic}

\begin{lemma}[Inverse-log threshold]
\label{lem:poly-inverse-log}
If $A\geq1$ and $u\geq2A\log(2A)$, then $\log u\leq u/A$.
\end{lemma}

\begin{proof}
Let $\psi(u)=u/A-\log u$.  It is nondecreasing for $u\geq A$.  At
$u_0=2A\log(2A)>A$,
\[
  \psi(u_0)=\log(2A)-\log\log(2A)\geq0,
\]
because $y\geq\log y$ for $y>0$.  Hence $\psi(u)\geq0$ for all
$u\geq u_0$.
\end{proof}

\begin{lemma}[Exponential dominance]
\label{lem:poly-exponential-dominance}
If $x>0$ and $u\geq\mathcal L(x)$, then
\begin{equation}
  e^{u/4}\geq xu.
  \label{eq:poly-exponential-dominance}
\end{equation}
Consequently, for $y,\gamma>0$,
\begin{equation}
  \lambda\geq\frac{1}{4\gamma}
  \mathcal L\!\left(\frac{y}{4\gamma}\right)
  \quad\Longrightarrow\quad
  e^{\gamma\lambda}\geq y\lambda.
  \label{eq:poly-exponential-rescaling}
\end{equation}
\end{lemma}

\begin{proof}
Set $R=\max\{e,x\}$ and $u_0=8\log(4R)>4$.  The function
$e^{u/4}/u$ increases for $u>4$, whence
\[
  \frac{e^{u/4}}u\geq
  \frac{(4R)^2}{8\log(4R)}
  =\frac{2R^2}{\log(4R)}\geq R\geq x.
\]
The penultimate inequality follows because
$2R-\log(4R)>0$ for $R\geq e$.  This proves
\eqref{eq:poly-exponential-dominance}.  Apply it with
$u=4\gamma\lambda$ and $x=y/(4\gamma)$ to obtain
\eqref{eq:poly-exponential-rescaling}.
\end{proof}

\begin{lemma}[Absorption of a stretched exponential]
\label{lem:poly-absorption}
Let $C\geq1$, $c,b,\gamma>0$, $T>1$, and $\lambda=\log T$.  If
\[
  \lambda\geq
  \max\left\{
  \frac{\mathcal L(b/(2\gamma c))}{4\gamma},\
  \frac1\gamma\log\left(\max\left\{e,\frac{2\log C}{c}\right\}\right)
  \right\},
\]
then $C\exp\{-cT^\gamma\}\leq T^{-b}$.
\end{lemma}

\begin{proof}
By \eqref{eq:poly-exponential-rescaling}, the first condition gives
$c e^{\gamma\lambda}\geq2b\lambda$, and the second gives
$c e^{\gamma\lambda}\geq2\log C$.  Averaging yields
$c e^{\gamma\lambda}\geq b\lambda+\log C$, which is equivalent to the
claim.
\end{proof}

\begin{lemma}[Mills lower bound]
\label{lem:poly-mills}
For $x>0$,
\[
  \overline\Phi(x)\geq\frac{x}{1+x^2}\varphi(x);
  \qquad
  \overline\Phi(x)\geq\frac{e^{-x^2/2}}{2\sqrt{2\pi}\,x}
  \quad(x\geq1).
\]
\end{lemma}

\begin{proof}
Let $u(x)=x/(1+x^2)$ and
$\psi(x)=\overline\Phi(x)-u(x)\varphi(x)$.  Direct differentiation gives
\[
  \psi'(x)=-\frac{2\varphi(x)}{(1+x^2)^2}<0,
  \qquad
  \lim_{x\to\infty}\psi(x)=0.
\]
Thus $\psi(x)\geq0$.  When $x\geq1$,
$x/(1+x^2)\geq1/(2x)$, proving the second inequality.
\end{proof}

\begin{lemma}[Tube-radius monotonicity]
\label{lem:poly-tube-monotonicity}
For $C\geq2$, the sequence
$f_n=\log(Cn(n+1))/n$ is strictly decreasing for $n\geq1$.
\end{lemma}

\begin{proof}
Multiplication by $n(n+1)$ shows that $f_n-f_{n+1}$ has the sign of
\[
  \log\!\left(\frac{C n(n+1)}{(1+2/n)^n}\right).
\]
For $n=1$ its argument is $2C/3>1$; for $n\geq2$ it is larger than
$12/e^2>1$.  Thus every difference is positive.
\end{proof}

\subsection{Stack tube and closed anytime control}
\label{subsec:poly-stack-tube}

\begin{lemma}[Stack identity]
\label{lem:poly-stack-identity}
For every $a\in[K]$, $c\in[d]$, and $t\geq0$, with
$\ttpStackMean{a}{n}{c}:=n^{-1}\sum_{r=1}^nX_{a,r}^c$,
\begin{equation}
  \widehat\mu_a^c(t)=\ttpStackMean{a}{N_a(t)}{c}.
  \label{eq:poly-stack-identity}
\end{equation}
\end{lemma}

\begin{proof}
The reward stack construction \eqref{eq:reward-stacks} and count definition
\eqref{eq:pull-count} say that the observations of arm $a$ after $t$ rounds
are its first $N_a(t)$ stack entries.  Substitution in
\eqref{eq:empirical-mean} gives \eqref{eq:poly-stack-identity}.
\end{proof}

\begin{lemma}[Closed anytime tail]
\label{lem:poly-closed-anytime}
For every arm $a$, coordinate $c$, integer $n_0\geq1$, and
$\varepsilon>0$,
\begin{equation}
  \Prob_\nu\!\left(
  \exists n\geq n_0:
  \left|\ttpStackMean{a}{n}{c}-\mu_a^c\right|\geq\varepsilon\right)
  \leq2\exp\!\left(-\frac{n_0\varepsilon^2}{2\sigma^2}\right).
  \label{eq:poly-closed-anytime}
\end{equation}
\end{lemma}

\begin{proof}
\Cref{lem:scalar-concentration} gives the same bound with strict $>$ in the
event.  For $0<\xi<\varepsilon$, the event on the left of
\eqref{eq:poly-closed-anytime} is contained in the strict event at
$\varepsilon-\xi$.  Let $\xi\downarrow0$ and use continuity of the bound.
\end{proof}

\begin{lemma}[Stack tube]
\label{lem:poly-stack-tube}
For deterministic $n\geq1$, let
\[
  C_T:=2KdT^\beta,\qquad
  b_{T,n}:=\sigma\sqrt{\frac{2}{n}\log(C_Tn(n+1))},
  \qquad
  \mathsf{Stk}_T:=
  \bigcap_{\substack{a\in[K],\,c\in[d]\\n\geq1}}
  \left\{\left|\ttpStackMean{a}{n}{c}-\mu_a^c\right|\leq b_{T,n}\right\}.
\]
Then
\begin{equation}
  \Prob_\nu(\mathsf{Stk}_T^c)\leq\delta_T.
  \label{eq:poly-stack-tube-probability}
\end{equation}
Moreover $b_{T,m}\leq b_{T,n}$ for $m\geq n$, and
\begin{equation}
  x_{T,n}:=\frac{(\Gamma+2b_{T,n})\sqrt n}{\sigma}
  =\frac{\Gamma\sqrt n}{\sigma}
   +2\sqrt{2\log(C_Tn(n+1))}
  \label{eq:poly-standardized-tube}
\end{equation}
is nondecreasing in $n$.
\end{lemma}

\begin{proof}
Apply \Cref{lem:poly-closed-anytime} at the single index $n$ with
$\varepsilon=b_{T,n}$ to obtain
\[
  \Prob_\nu\!\left(
  \left|\ttpStackMean{a}{n}{c}-\mu_a^c\right|>b_{T,n}\right)
  \leq\frac{\delta_T}{Kd\,n(n+1)}.
\]
Boole's inequality over $a,c,n$, together with
$\sum_{n\geq1}[n(n+1)]^{-1}=1$, proves
\eqref{eq:poly-stack-tube-probability}.  The monotonicity of $b_{T,n}$ is
\Cref{lem:poly-tube-monotonicity}; the two summands in
\eqref{eq:poly-standardized-tube} are nondecreasing.
\end{proof}

\subsection{Predictable hull and posterior selection}
\label{subsec:poly-selection-floor}

Fix $t\in[T]$, put $a:=\iota_t$, and set
\[
  n:=N_a(t-1)=\min_{i\in[K]}N_i(t-1),
  \qquad
  n_j:=N_j(t-1)\quad(j\ne a).
\]
Thus $n_j\geq n\geq1$, and these values are
$\ttpHistory{t-1}$-measurable.

\begin{lemma}[Predictable hull and witness deficit]
\label{lem:poly-predictable-hull}
Define
\begin{equation}
  \mathsf P_t:=
  \bigcap_{\substack{i\in[K]\\c\in[d]}}
  \left\{
  \left|\widehat\mu_i^c(t-1)-\mu_i^c\right|
  \leq b_{T,N_i(t-1)}\right\}.
  \label{eq:poly-predictable-hull}
\end{equation}
Then $\mathsf P_t\in\ttpHistory{t-1}$ and
$\mathsf{Stk}_T\subseteq\mathsf P_t$ pointwise.  For
$j\ne a$, set $c_j=c(a,j)$ and
\[
  h_j:=\widehat\mu_j^{c_j}(t-1)-\widehat\mu_a^{c_j}(t-1).
\]
Pointwise on $\mathsf P_t$,
\begin{equation}
  h_j\leq\Gamma+2b_{T,n}\qquad(j\ne a).
  \label{eq:poly-witness-bound}
\end{equation}
\end{lemma}

\begin{proof}
Measurability in \eqref{eq:poly-predictable-hull} follows from the
$\ttpHistory{t-1}$-measurability of counts and empirical means.  On
$\mathsf{Stk}_T$, evaluate its deterministic-index inequalities at
$N_i(t-1)$ and use \Cref{lem:poly-stack-identity}; this proves the pointwise
inclusion without conditioning on a future stack event.

By the definition of $c_j$,
$\mu_j^{c_j}-\mu_a^{c_j}=\ttpGapMin(a,j)$.  On $\mathsf P_t$,
\[
  h_j\leq\ttpGapMin(a,j)+b_{T,n_j}+b_{T,n}
  \leq\Gamma+2b_{T,n},
\]
where the last inequality uses $n_j\geq n$ and
\Cref{lem:poly-stack-tube}.  This proves
\eqref{eq:poly-witness-bound}.
\end{proof}

Conditionally on $\ttpHistory{t-1}$, posterior coordinates are independent
Gaussian variables.  Thus
\[
  \widetilde\mu_i^{(t),c}
  =\widehat\mu_i^c(t-1)
   +\frac{\sigma}{\sqrt{N_i(t-1)}}G_{i,t}^c.
\]
For $j\ne a$, write
$E_j=\{\widetilde\mu_a^{(t),c_j}>
\widetilde\mu_j^{(t),c_j}\}$.

\begin{lemma}[Per-competitor posterior floor]
\label{lem:poly-posterior-floor}
Pointwise on $\mathsf P_t$,
\begin{equation}
  \Prob_\nu(E_j\mid\ttpHistory{t-1})
  =\overline\Phi\!\left(\frac{h_j}{s_j}\right)
  \geq\overline\Phi(x_{T,n}),
  \qquad
  s_j:=\sigma\sqrt{\frac1n+\frac1{n_j}}.
  \label{eq:poly-posterior-floor}
\end{equation}
\end{lemma}

\begin{proof}
Conditional on $\ttpHistory{t-1}$, the difference in $E_j$ is
$\mathcal N(-h_j,s_j^2)$, which gives the equality.  The sign-safe
standardization is
\[
  \frac{h_j}{s_j}\leq\frac{(h_j)_+\sqrt n}{\sigma}.
\]
Indeed it is immediate if $h_j\leq0$; if $h_j>0$, use
$s_j\geq\sigma/\sqrt n$.  Equation \eqref{eq:poly-witness-bound} then gives
$(h_j)_+\leq\Gamma+2b_{T,n}$, so the right side is at most
$x_{T,n}$.  Since $\overline\Phi$ is nonincreasing, this proves
\eqref{eq:poly-posterior-floor}.
\end{proof}

\begin{lemma}[One-dimensional association]
\label{lem:poly-association}
If $W$ is real-valued and $f_1,\ldots,f_k:\R\to[0,1]$ are nondecreasing
Borel functions, then
\[
  \E\left[\prod_{r=1}^kf_r(W)\right]
  \geq\prod_{r=1}^k\E[f_r(W)].
\]
\end{lemma}

\begin{proof}
For $k=2$, take an independent copy $W'$ of $W$.  The two monotonicity
differences have the same sign, so
\[
  0\leq\frac12\E[
  (f_1(W)-f_1(W'))(f_2(W)-f_2(W'))]
  =\E[f_1(W)f_2(W)]-\E[f_1(W)]\E[f_2(W)].
\]
For general $k$, apply the two-function result to
$\prod_{r<k}f_r$ and $f_k$, then induct.
\end{proof}

\begin{lemma}[Block factorization and first-front membership]
\label{lem:poly-first-front}
Pointwise on $\mathsf P_t$,
\begin{equation}
  \Prob_\nu\!\left(\bigcap_{j\ne a}E_j\ \middle|\ \ttpHistory{t-1}\right)
  \geq\prod_{j\ne a}\Prob_\nu(E_j\mid\ttpHistory{t-1})
  \geq\overline\Phi(x_{T,n})^{K-1},
  \label{eq:poly-block-factorization}
\end{equation}
and, pointwise,
\begin{equation}
  \bigcap_{j\ne a}E_j\subseteq\{a\in\ttpPosteriorFront{1}{t}\}.
  \label{eq:poly-first-front-membership}
\end{equation}
\end{lemma}

\begin{proof}
Partition the competitors according to their fixed witness coordinate:
$J_c=\{j\ne a:c(a,j)=c\}$ for $c\in[d]$.  Conditional on
$W_c:=\widetilde\mu_a^{(t),c}=w$, the events
$\{W_c>\widetilde\mu_j^{(t),c}\}$ for $j\in J_c$ factorize because their
competitor seeds are distinct and independent.  Their conditional
probabilities are the nondecreasing functions
\[
  F_j(w)=\Phi\!\left(
  \frac{(w-\widehat\mu_j^c(t-1))\sqrt{n_j}}{\sigma}\right).
\]
\Cref{lem:poly-association} therefore shows that the intersection
probability for each coordinate block is at least the product of its
marginal probabilities.  Distinct coordinate blocks use disjoint posterior
seeds, hence are independent.  Multiplying blocks and applying
\eqref{eq:poly-posterior-floor} proves
\eqref{eq:poly-block-factorization}.

On $\bigcap_{j\ne a}E_j$, arm $a$ is strictly larger than each competitor in
at least its witness coordinate.  No competitor can therefore strictly
Pareto-dominate $a$, and \eqref{eq:front1} proves
\eqref{eq:poly-first-front-membership}.  This argument includes repeated
witness coordinates: their dependence is handled only through the scalar
association inequality.
\end{proof}

\begin{lemma}[Least-count selection floor]
\label{lem:poly-selection-floor}
Set
\begin{equation}
  q_{T,n}:=\frac{\rho}{K}\overline\Phi(x_{T,n})^{K-1}.
  \label{eq:poly-selection-rate}
\end{equation}
Then $n\mapsto q_{T,n}$ is nonincreasing and, pointwise on $\mathsf P_t$,
\begin{equation}
  \Prob_\nu(A_t=\iota_t\mid\ttpHistory{t-1})\geq q_{T,n}.
  \label{eq:poly-selection-floor}
\end{equation}
\end{lemma}

\begin{proof}
\Cref{lem:poly-stack-tube} makes $x_{T,n}$ nondecreasing, so
$q_{T,n}$ is nonincreasing.  Since $\rho\in(0,1)$ and
$\overline\Phi(x)>0$ for finite $x$, $q_{T,n}\in(0,1)$.  On the event in
\eqref{eq:poly-first-front-membership}, select the first Top-Two branch and
then the least-count arm from its first posterior front.  The branch has
probability $\rho$ and uniform selection has probability at least $1/K$.
Independence of those seeds from posterior draws, followed by
\eqref{eq:poly-block-factorization}, yields
\eqref{eq:poly-selection-floor}.  The second front and fallback are not used.
\end{proof}

\subsection{Count target, potential, and adapted count bound}
\label{subsec:poly-count-growth}

\begin{lemma}[Count target properties]
\label{lem:poly-count-target}
Define
\[
  n_T:=\left\lceil\theta(\ell+\beta\lambda)\right\rceil,
  \qquad n_T^-:=n_T-1,
  \qquad
  Q_T:=q_{T,n_T^-}.
\]
In particular, $Q_T\in(0,1)$.
For $T\geq T_0^{\rm poly}$,
\begin{equation}
  n_T\geq2,\qquad n_T^-\geq1,\qquad
  n_T^-<\theta(\ell+\beta\lambda)
  \leq\tfrac98\theta\beta\lambda,\qquad
  \frac{n_T}{\theta}\geq\ell+\beta\lambda.
  \label{eq:poly-count-properties}
\end{equation}
Moreover, on
$\mathsf P_t\cap\{\min_iN_i(t-1)\leq n_T^-\}$,
\begin{equation}
  \Prob_\nu\!\left(
  A_t\in\arg\min_iN_i(t-1)\mid\ttpHistory{t-1}\right)\geq Q_T.
  \label{eq:poly-minimum-selection-floor}
\end{equation}
\end{lemma}

\begin{proof}
The first two horizon entries in \eqref{eq:poly-horizon} yield
$\theta\beta\lambda\geq2$ and $\ell\leq\beta\lambda/8$.
The ceiling inequalities give \eqref{eq:poly-count-properties}.  If the
least count is at most $n_T^-$, then monotonicity in
\Cref{lem:poly-selection-floor} implies that the right side of
\eqref{eq:poly-selection-floor} is at least $Q_T$, proving
\eqref{eq:poly-minimum-selection-floor}.
\end{proof}

\begin{lemma}[Least-count potential]
\label{lem:poly-least-count}
For $N\in\N^K$, define
\[
  \Psi(N):=K\min_iN_i+
  \#\{i\in[K]:N_i>\min_jN_j\}.
\]
Let
$H_T:=\sum_{t=1}^T\ttpInd\{A_t\in\arg\min_iN_i(t-1)\}$.  Then
\begin{equation}
  \Psi(N(T))\geq K+H_T,
  \qquad
  \{\min_iN_i(T)<n_T\}\subseteq\{H_T<Kn_T^-\}.
  \label{eq:poly-potential-count}
\end{equation}
\end{lemma}

\begin{proof}
Initially $\Psi(N(0))=K$, and always
$K\min_iN_i\leq\Psi(N)\leq K\min_iN_i+(K-1)$.  Pulling a non-minimizer does
not change $\Psi$.  Pulling one of at least two minimizers increases its
second term by one.  Pulling the unique minimizer raises the minimum by at
least one, while the preceding second term equals $K-1$, and again increases
$\Psi$ by at least one.  Summing gives the first assertion.  If the final
minimum is at most $n_T^-$, the upper bound on $\Psi$ gives
$K+H_T\leq Kn_T^-+(K-1)$, which proves the strict second assertion.
\end{proof}

\begin{lemma}[Augmented indicator]
\label{lem:poly-augmented-count}
Extend the probability space by iid $V_t\sim\Unif(0,1)$, independent of the
experiment, and write
$\mathcal A_t:=\ttpHistory{t}\vee\sigma(V_1,\ldots,V_t)$.  Define
\begin{equation}
  Y_t:=
  \begin{cases}
  \ttpInd\{A_t\in\arg\min_iN_i(t-1)\},
  &\text{on }\mathsf P_t\cap\{\min_iN_i(t-1)\leq n_T^-\},\\
  \ttpInd\{V_t\leq Q_T\},&\text{otherwise}.
  \end{cases}
  \label{eq:poly-augmented-indicator}
\end{equation}
Then $Y_t$ is $\mathcal A_t$-measurable and
\begin{equation}
  \Prob(Y_t=1\mid\mathcal A_{t-1})\geq Q_T
  \qquad(t\in[T]).
  \label{eq:poly-augmented-floor}
\end{equation}
\end{lemma}

\begin{proof}
The selector in the first line of \eqref{eq:poly-augmented-indicator} is
$\ttpHistory{t-1}$-measurable.  There,
\eqref{eq:poly-minimum-selection-floor} applies, and conditioning additionally
on past auxiliary uniforms changes nothing.  Off that selector,
$V_t$ is fresh and has success probability exactly $Q_T$.  These two
cases prove both measurability and \eqref{eq:poly-augmented-floor}.
\end{proof}

\begin{lemma}[Adapted simultaneous count bound]
\label{lem:poly-count-bound}
If
\begin{equation}
  Kn_T^-<\frac12TQ_T,
  \label{eq:poly-count-hypothesis}
\end{equation}
then
\begin{equation}
  \Prob_\nu\!\left(\min_iN_i(T)<n_T\right)
  \leq\delta_T+\exp\!\left(-\frac{TQ_T}{8}\right).
  \label{eq:poly-count-bound}
\end{equation}
\end{lemma}

\begin{proof}
On $\mathsf{Stk}_T\cap\{\min_iN_i(T)<n_T\}$, the predictable hull holds at
every round by \Cref{lem:poly-predictable-hull}; the nondecreasing minimum
count remains at most $n_T^-$.  Hence the first line of
\eqref{eq:poly-augmented-indicator} applies at every round, so
$\sum_tY_t=H_T$.  By \eqref{eq:poly-potential-count} and
\eqref{eq:poly-count-hypothesis},
$\sum_tY_t<TQ_T/2$.  Apply the adapted lower-tail bound
\Cref{lem:bernoulli-tail} to \Cref{lem:poly-augmented-count}, then add
\eqref{eq:poly-stack-tube-probability}.  This proves
\eqref{eq:poly-count-bound} without a binomial-independence assertion.
\end{proof}

\begin{theorem}[Mixed-front finite polynomial bound]
\label{thm:sup:poly-mixed-finite}
Let $\nu$ be a mixed-front separated instance.  Define
$\beta(\nu)$, $c_{\rm sel}(\nu)$, and $T_0^{\rm poly}(\nu)$ below in
\eqref{eq:poly-constants}--\eqref{eq:poly-horizon}.  Then every integer
$T\geq T_0^{\rm poly}(\nu)$ satisfies
\begin{equation}
  \ttpErrProb{T}{\nu}
  \leq
  \left(K+\tfrac12\right)T^{-\beta}
  +K\exp\!\left\{-\frac{c_{\rm sel}}8T^{1/4}\right\}.
  \label{eq:sup:poly-mixed-bound}
\end{equation}
\end{theorem}

\begin{corollary}[Pure polynomial form]
\label{cor:poly-pure}
Under the hypotheses of \Cref{thm:sup:poly-mixed-finite}, every integer
$T\geq T_0^{\rm poly}(\nu)$ satisfies
\begin{equation}
  \ttpErrProb{T}{\nu}
  \leq\min\left\{1,\left(K+\tfrac32\right)T^{-\beta(\nu)}\right\}.
  \label{eq:poly-pure-bound}
\end{equation}
\end{corollary}

\begin{corollary}[Physical-pull form]
\label{cor:poly-physical}
Under the hypotheses of \Cref{thm:sup:poly-mixed-finite}, let $P=K+T$ be the
number of physical pulls.  For every integer $P\geq K+T_0^{\rm poly}(\nu)$,
\begin{equation}
  e^{\rm phys}_P(\nu):=\ttpErrProb{P-K}{\nu}
  \leq
  \min\left\{1,\left(K+\tfrac32\right)(P-K)^{-\beta(\nu)}\right\}.
  \label{eq:poly-physical-bound}
\end{equation}
\end{corollary}

\begin{theorem}[Combined positive-gap package]
\label{thm:poly-combined}
Let an instance obey \Cref{sec:model-contract}, with $K\geq1$, $d\geq1$,
$\sigma>0$, $\rho\in(0,1)$, and \eqref{eq:separated-instance}.  There are
explicit $\beta(\nu)>0$, $C(\nu)\geq1$, and a finite integer
$T_\sharp(\nu)$ such that
\begin{equation}
  \ttpErrProb{T}{\nu}
  \leq\min\{1,C(\nu)T^{-\beta(\nu)}\}
  \qquad(T\geq T_\sharp(\nu)).
  \label{eq:poly-combined-bound}
\end{equation}
The constants are instance-wise; in particular, this theorem does not assert
a common exponent over all positive-gap instances.
\end{theorem}

\begin{table}[htbp]
  \centering
  \caption{Proof roadmap for \Cref{thm:sup:poly-mixed-finite}.}
  \label{tab:poly-roadmap}
  \begin{tabular}{@{}p{0.27\textwidth}p{0.63\textwidth}@{}}
    \toprule
    Layer & Result \\
    \midrule
    Stack control & \Cref{lem:poly-stack-identity,lem:poly-closed-anytime,lem:poly-stack-tube}
      make random empirical indices safe. \\
    Posterior selection & \Cref{lem:poly-predictable-hull,lem:poly-association,lem:poly-first-front,lem:poly-selection-floor}
      give a predictable least-count selection floor. \\
    Count growth & \Cref{lem:poly-least-count,lem:poly-count-bound,lem:poly-calibration}
      bring every arm to the count target. \\
    Recommendation & \Cref{lem:poly-recommendation} and
      \Cref{lem:poly-final-concentration} certify the empirical front. \\
    \bottomrule
  \end{tabular}
\end{table}

\subsection{Exponent calibration}
\label{subsec:poly-calibration}

\begin{lemma}[Calibration]
\label{lem:poly-calibration}
For every $T\geq T_0^{\rm poly}$,
\begin{align}
  \log\!\left(C_Tn_T^-(n_T^-+1)\right)
    &\leq\tfrac98\beta\lambda, \label{eq:poly-calibration-log}\\
  1<x_{T,n_T^-}
    &\leq X:=3g\sqrt{\beta\lambda},
  &\frac{X^2}{2}&=\frac{9\lambda}{16(K-1)}, \label{eq:poly-calibration-x}\\
  Q_T
    &\geq c_{\rm sel}\lambda^{-(K-1)/2}e^{-9\lambda/16}
     \geq c_{\rm sel}T^{-3/4}, \label{eq:poly-calibration-tail}\\
  TQ_T
    &\geq c_{\rm sel}T^{1/4}>2Kn_T^- . \label{eq:poly-calibration-hit}
\end{align}
In particular, \eqref{eq:poly-count-hypothesis} holds.
\end{lemma}

\begin{proof}
By \eqref{eq:poly-count-properties}, $n_T^-\geq1$ and
$n_T^-<\frac98\theta\beta\lambda$.  Therefore
\[
  2Kd\,n_T^-(n_T^-+1)
  <4Kd\left(\frac98\theta\beta\right)^2\lambda^2
  =D_\star\lambda^2.
\]
Since $C_T=2KdT^\beta$, this gives
\[
  \log(C_Tn_T^-(n_T^-+1))
  \leq\beta\lambda+\ell_\star+2\log\lambda.
\]
The third horizon entry in \eqref{eq:poly-horizon} bounds
$\ell_\star$ by $\beta\lambda/16$.  Apply
\Cref{lem:poly-inverse-log} with $A=32/\beta$ and use the fourth entry of
\eqref{eq:poly-horizon}; it gives $2\log\lambda\leq\beta\lambda/16$.
This proves \eqref{eq:poly-calibration-log}.

Use \eqref{eq:poly-standardized-tube},
\eqref{eq:poly-count-properties}, and
$\sqrt\theta=2\sqrt2\,\sigma/\Delta_0$ to get
\[
  \frac{\Gamma\sqrt{n_T^-}}{\sigma}
  <3\frac{\Gamma}{\Delta_0}\sqrt{\beta\lambda},
  \qquad
  2\sqrt{2\log(C_Tn_T^-(n_T^-+1))}
  \leq3\sqrt{\beta\lambda}.
\]
Their sum is at most $X$.  It is also greater than one because
$C_Tn_T^-(n_T^-+1)>12$.  Finally,
$g^2\beta=1/(8(K-1))$ proves the equality in
\eqref{eq:poly-calibration-x}.

By \Cref{lem:poly-mills} and \eqref{eq:poly-calibration-x},
\[
  2\sqrt{2\pi}\,X
  =6\sqrt{2\pi}\,g\sqrt{\beta\lambda}
  =3\sqrt{\frac{\pi\lambda}{K-1}}
  \leq3\sqrt{\pi\lambda},
\]
where $g\sqrt\beta=1/\sqrt{8(K-1)}$.  Hence
\[
  \overline\Phi(x_{T,n_T^-})
  \geq\overline\Phi(X)
  \geq\frac{e^{-9\lambda/(16(K-1))}}{3\sqrt{\pi\lambda}}.
\]
Raise this to the $(K-1)$st power and use \eqref{eq:poly-selection-rate}:
\[
  Q_T\geq c_{\rm sel}\lambda^{-(K-1)/2}e^{-9\lambda/16}.
\]
The fifth horizon entry and \Cref{lem:poly-inverse-log}, now with
$A=8(K-1)/3$, imply
$\lambda^{-(K-1)/2}\geq e^{-3\lambda/16}$.  This proves
\eqref{eq:poly-calibration-tail}.  The sixth horizon entry and
\eqref{eq:poly-exponential-dominance} give
\[
  TQ_T\geq c_{\rm sel}e^{\lambda/4}
  \geq c_{\rm sel}R_{\rm hit}\lambda
  =\frac{9K\theta\beta\lambda}{4}
  >2Kn_T^-,
\]
where the last inequality is \eqref{eq:poly-count-properties}.  This is
\eqref{eq:poly-calibration-hit}.
\end{proof}

\begin{remark}[Available exponent slack]
\label{rem:poly-slack}
The calibration uses exactly
\begin{equation}
  \underbrace{\frac9{16}}_{\text{Gaussian tail}}
  +\underbrace{\frac3{16}}_{\text{Mills prefactor}}
  =\frac34<1.
  \label{eq:poly-slack-exponent}
\end{equation}
The remaining $1/4$ power pays for the logarithmic count target.  This is
only arithmetic slack in this proof, not a sharpness or critical-exponent
claim.
\end{remark}

\subsection{Deterministic recommendation certificate}
\label{subsec:poly-recommendation}

\begin{lemma}[Per-arm recommendation certificate]
\label{lem:poly-recommendation}
Let $K\geq2$, suppose \eqref{eq:separated-instance}, and put
$\varepsilon_a:=\ttpNormInf{\widehat\mu_a(T)-\mu_a}$.  If
\[
  \varepsilon_a<\frac{\Delta_a}{2}\qquad(a\in[K]),
\]
then
\begin{equation}
  \ttpEmpFront{T}=\ttpTrueFront.
  \label{eq:poly-recommendation-certificate}
\end{equation}
\end{lemma}

\begin{proof}
For distinct arms $u,v$ and every coordinate $c$, the hypothesis gives
\begin{equation}
  \widehat\mu_u^c(T)-\widehat\mu_v^c(T)
  >(\mu_u^c-\mu_v^c)-\frac{\Delta_u+\Delta_v}{2}.
  \label{eq:poly-empirical-difference}
\end{equation}
Let $a\notin\ttpTrueFront$ and choose
$j^\star\in\ttpTrueFront$ attaining
$\ttpGapMin(a,j^\star)=\Delta_a$.  Then
$\mu_{j^\star}^c-\mu_a^c\geq\Delta_a$ for every $c$.  Hence
$\ttpGapMax(a,j^\star)<0$ and, by \eqref{eq:optimal-gap},
$\Delta_{j^\star}\leq\Delta_a$.  Substitution in
\eqref{eq:poly-empirical-difference} shows that
$\widehat\mu_a(T)\prec\widehat\mu_{j^\star}(T)$, excluding $a$ from the
empirical front.

Next let distinct $i,j\in\ttpTrueFront$.  Both $\Delta_i$ and $\Delta_j$ are
at most $\ttpGapMax(i,j)$ by \eqref{eq:optimal-gap}.  A coordinate attaining
this maximum and \eqref{eq:poly-empirical-difference} therefore show
$\widehat\mu_i(T)\not\prec\widehat\mu_j(T)$.

Finally, take $i\in\ttpTrueFront$ and $j\notin\ttpTrueFront$.  Choose
$k\in\ttpTrueFront$ with $\ttpGapMin(j,k)=\Delta_j$.  If $k=i$, then
$\ttpGapMax(i,j)\geq\Delta_j$ and
$\Delta_i\leq\Delta_j$ by the $\delta_i^-$ term in
\eqref{eq:optimal-gap}.  If $k\ne i$, then
$\Delta_i\leq\ttpGapMax(i,k)$ by $\delta_i^+$ and, at a maximizing
coordinate,
\[
  \mu_i^c-\mu_j^c
  =(\mu_i^c-\mu_k^c)+(\mu_k^c-\mu_j^c)
  \geq\Delta_i+\Delta_j.
\]
In either case \eqref{eq:poly-empirical-difference} supplies a coordinate
at which $i$ exceeds $j$, so $j$ cannot dominate $i$.  Thus no empirical
vector dominates an arm in $\ttpTrueFront$, while every arm outside it is
empirically dominated.  This proves \eqref{eq:poly-recommendation-certificate}.
\end{proof}

\subsection{End-to-end assembly}
\label{subsec:poly-assembly}

\begin{lemma}[Final concentration at a random count]
\label{lem:poly-final-concentration}
For every $a\in[K]$,
\begin{equation}
  \Prob_\nu\!\left(\ttpNormInf{\widehat\mu_a(T)-\mu_a}
     \geq\frac{\Delta_a}{2}\right)
  \leq
  \Prob_\nu(N_a(T)<n_T)+\frac{\eta_T}{2K}.
  \label{eq:poly-final-concentration}
\end{equation}
\end{lemma}

\begin{proof}
On $\{N_a(T)\geq n_T\}$, \Cref{lem:poly-stack-identity} identifies the
empirical vector with a deterministic-index stack mean.  Hence the
large-error event is contained in the union, over $c\in[d]$, of
\[
  \left\{\exists n\geq n_T:
  \left|\ttpStackMean{a}{n}{c}-\mu_a^c\right|
  \geq\frac{\Delta_a}{2}\right\}.
\]
\Cref{lem:poly-closed-anytime} and Boole's inequality bound this by
\[
  2d\exp\!\left(-\frac{n_T\Delta_a^2}{8\sigma^2}\right)
  \leq2d e^{-n_T/\theta}
  \leq\frac{T^{-\beta}}{2K},
\]
where the last inequality is \eqref{eq:poly-count-properties}.  Combining
the count split with $\eta_T=T^{-\beta}$ proves
\eqref{eq:poly-final-concentration}.
\end{proof}

\begin{proof}[Proof of \Cref{thm:sup:poly-mixed-finite}]
By the contrapositive of \Cref{lem:poly-recommendation},
\[
  \ttpErrEvent{T}\subseteq
  \bigcup_{a=1}^K
  \left\{\ttpNormInf{\widehat\mu_a(T)-\mu_a}\geq\frac{\Delta_a}{2}\right\}.
\]
Use \Cref{lem:poly-final-concentration} and then
$\{N_a(T)<n_T\}\subseteq\{\min_iN_i(T)<n_T\}$ to obtain
\begin{equation}
  \ttpErrProb{T}{\nu}
  \leq K\Prob_\nu(\min_iN_i(T)<n_T)+\frac{\eta_T}{2}.
  \label{eq:poly-assembly-split}
\end{equation}
\Cref{lem:poly-calibration} verifies \eqref{eq:poly-count-hypothesis}, so
\Cref{lem:poly-count-bound} and
\eqref{eq:poly-calibration-hit} imply
\[
  \Prob_\nu(\min_iN_i(T)<n_T)
  \leq T^{-\beta}
    +\exp\!\left(-\frac{c_{\rm sel}}8T^{1/4}\right).
\]
Substitution in \eqref{eq:poly-assembly-split} proves
\eqref{eq:sup:poly-mixed-bound}.
\end{proof}

\begin{proof}[Proof of \Cref{cor:poly-pure}]
Apply \Cref{lem:poly-absorption} with
$C=K$, $c=c_{\rm sel}/8$, $\gamma=1/4$, and $b=\beta$.  Its two conditions
are exactly the seventh and eighth entries of \eqref{eq:poly-horizon};
therefore
\begin{equation}
  K\exp\!\left(-\frac{c_{\rm sel}}8T^{1/4}\right)\leq T^{-\beta}.
  \label{eq:poly-remainder-absorption}
\end{equation}
Combine \eqref{eq:poly-remainder-absorption} with
\eqref{eq:sup:poly-mixed-bound}, then cap the probability by one.
\end{proof}

\begin{proof}[Proof of \Cref{cor:poly-physical}]
The contract identity $P=K+T$ in \eqref{eq:pull-count} has inverse
$T=P-K$.  Rewriting \eqref{eq:poly-pure-bound} proves
\eqref{eq:poly-physical-bound}.
\end{proof}

\section{Empirical Evaluation}
\label{sec:supp:empirical_eval}
We empirically evaluate the performance of the proposed Top-Two Pareto Front Thompson Sampling (TTPFTS) algorithm on the Pareto Set Identification (PSI) task. Our evaluation is conducted using the eight synthetic Multi-objective Multi-armed Bandit (MOMAB) PSI benchmarking environments introduced in~\cite{Kone2023BanditPS}. These environments are designed to capture a wide range of structural properties and difficulty levels commonly encountered in MOMAB problems. For each environment, we use Gaussian reward distributions with diagonal covariance matrix $\sigma^2 I_d$ where $I_d$ is the identity matrix of order $d$ and variance $\sigma^2=0.25$.

In this section, we provide detailed specifications for each of the eight synthetic benchmarking environments. \textbf{Note:} To ensure completeness and reproducibility, the definitions and parameters listed below are adopted directly from the original specifications provided by~\cite{Kone2023BanditPS}.

\begin{figure}[htbp]
    \centering
    \begin{subfigure}{0.49\linewidth}
        \centering
        \includegraphics[width=\linewidth]{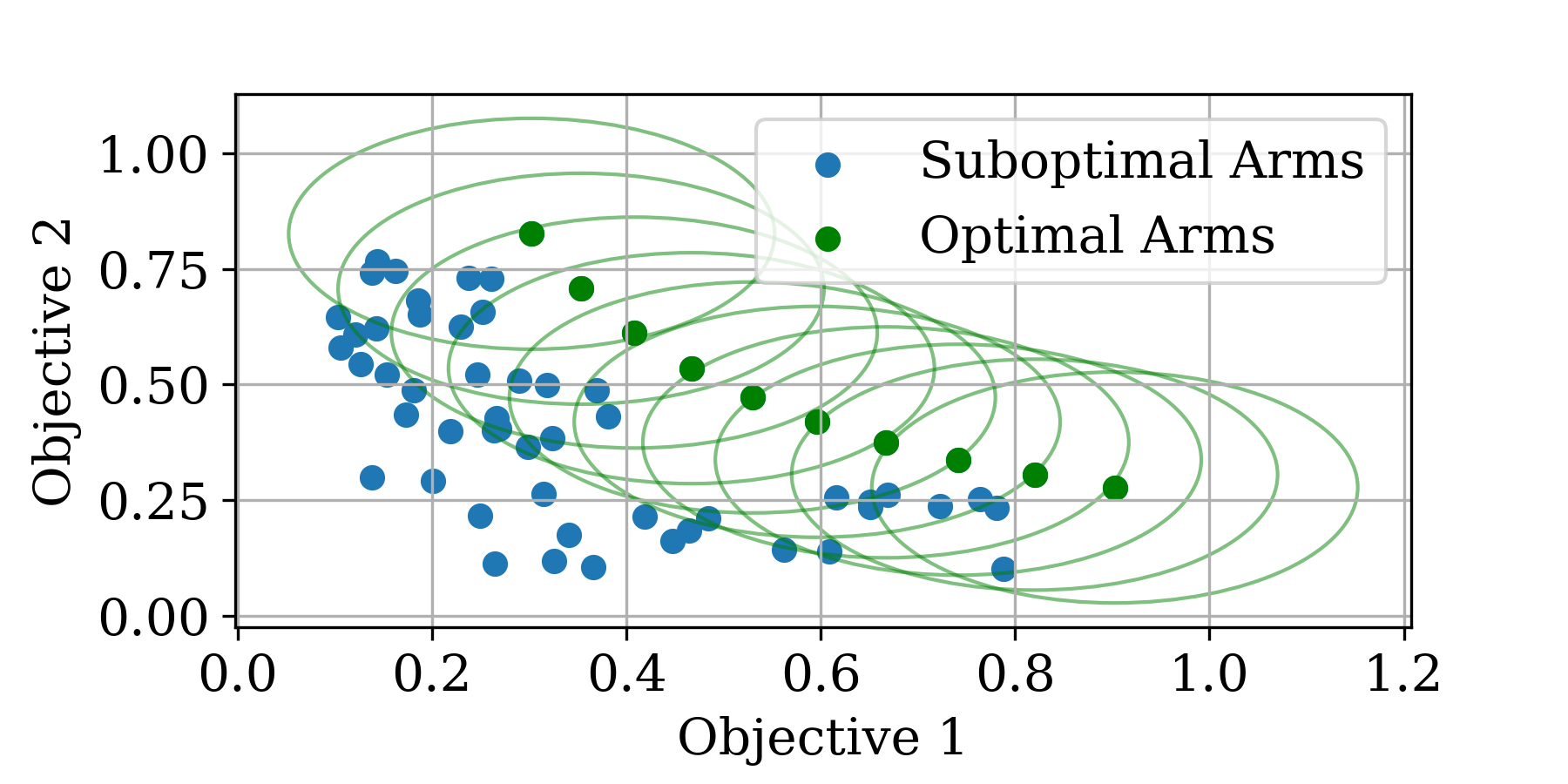}
        \caption{Environment 1}
    \end{subfigure}\hfill
    \begin{subfigure}{0.49\linewidth}
        \centering
        \includegraphics[width=\linewidth]{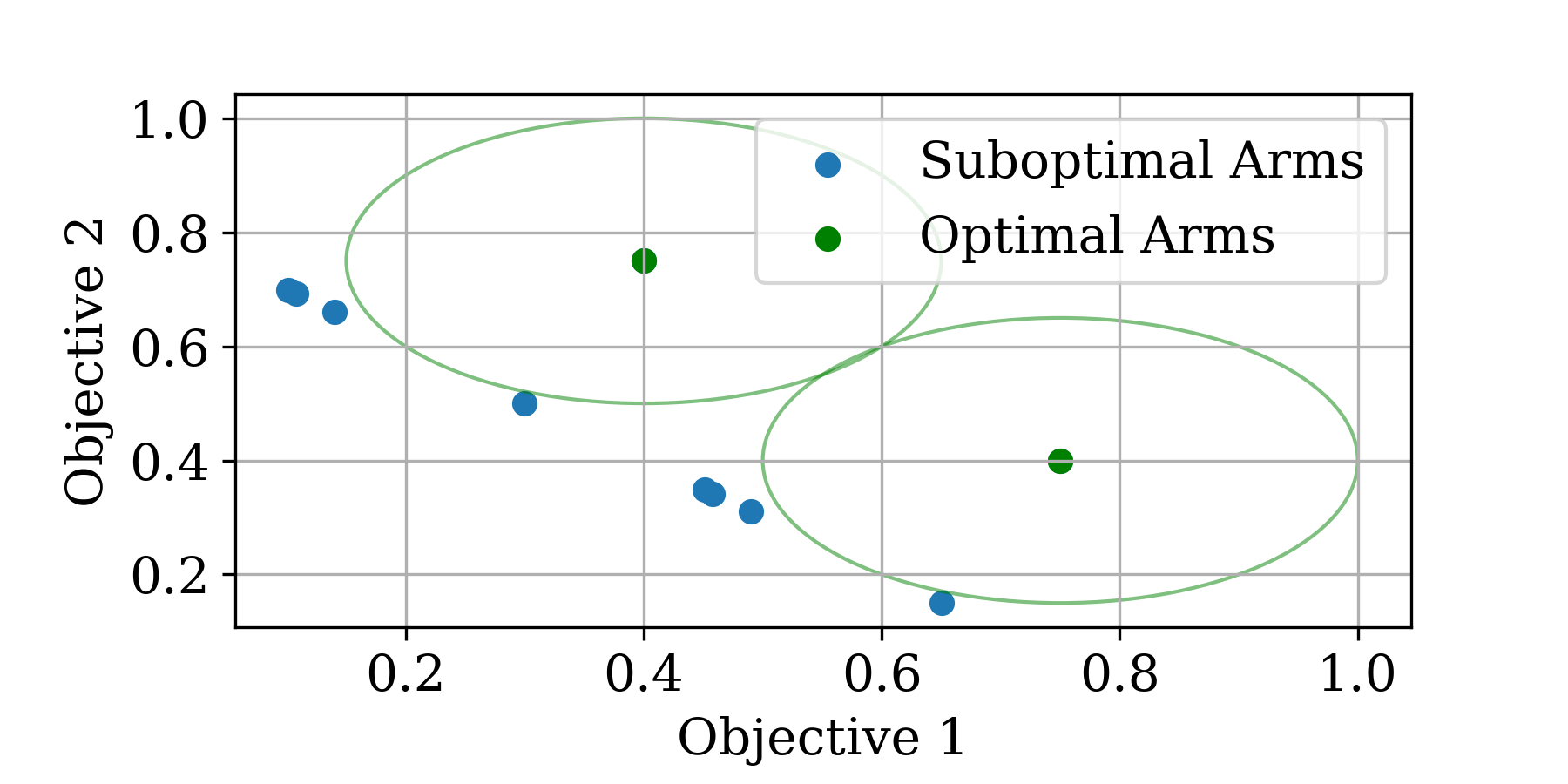}
        \caption{Environment 2}
    \end{subfigure}

    \begin{subfigure}{0.49\linewidth}
        \centering
        \includegraphics[width=\linewidth]{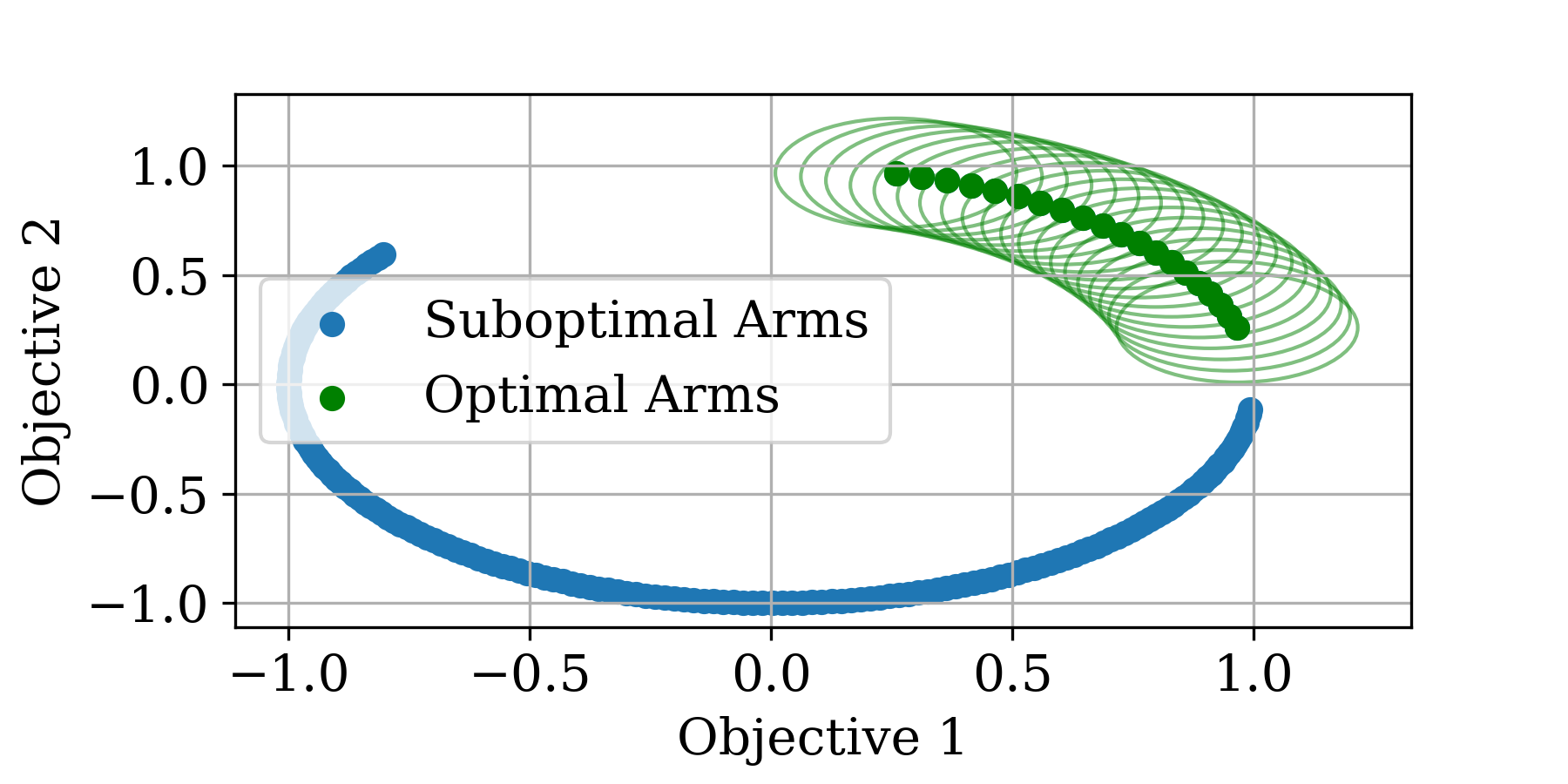}
        \caption{Environment 3}
    \end{subfigure}\hfill
    \begin{subfigure}{0.49\linewidth}
        \centering
        \includegraphics[width=\linewidth]{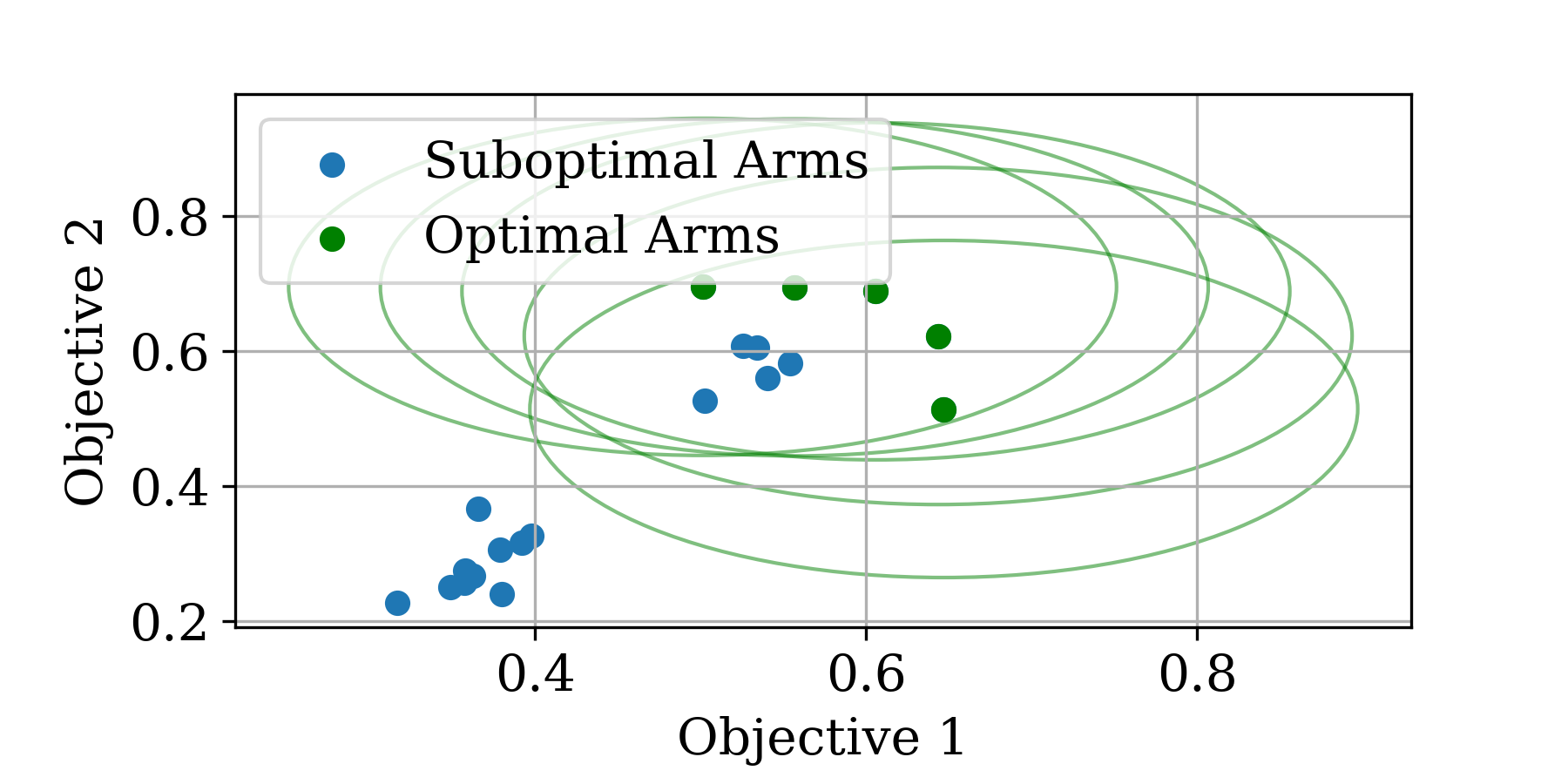}
        \caption{Environment 5}
    \end{subfigure}

    \begin{subfigure}{0.49\linewidth}
        \centering
        \includegraphics[width=\linewidth]{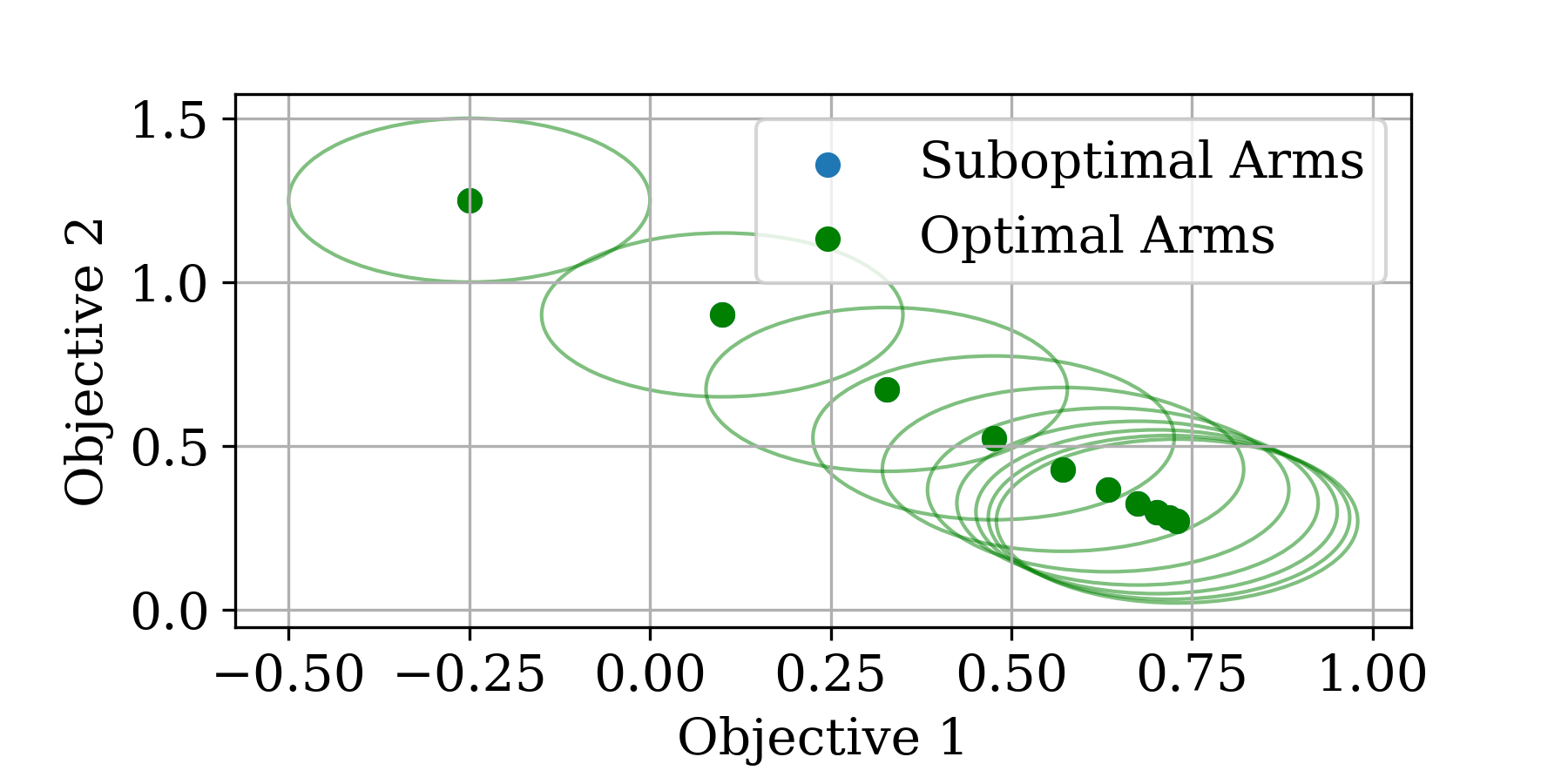}
        \caption{Environment 6}
    \end{subfigure}\hfill
    \begin{subfigure}{0.49\linewidth}
        \centering
        \includegraphics[width=\linewidth]{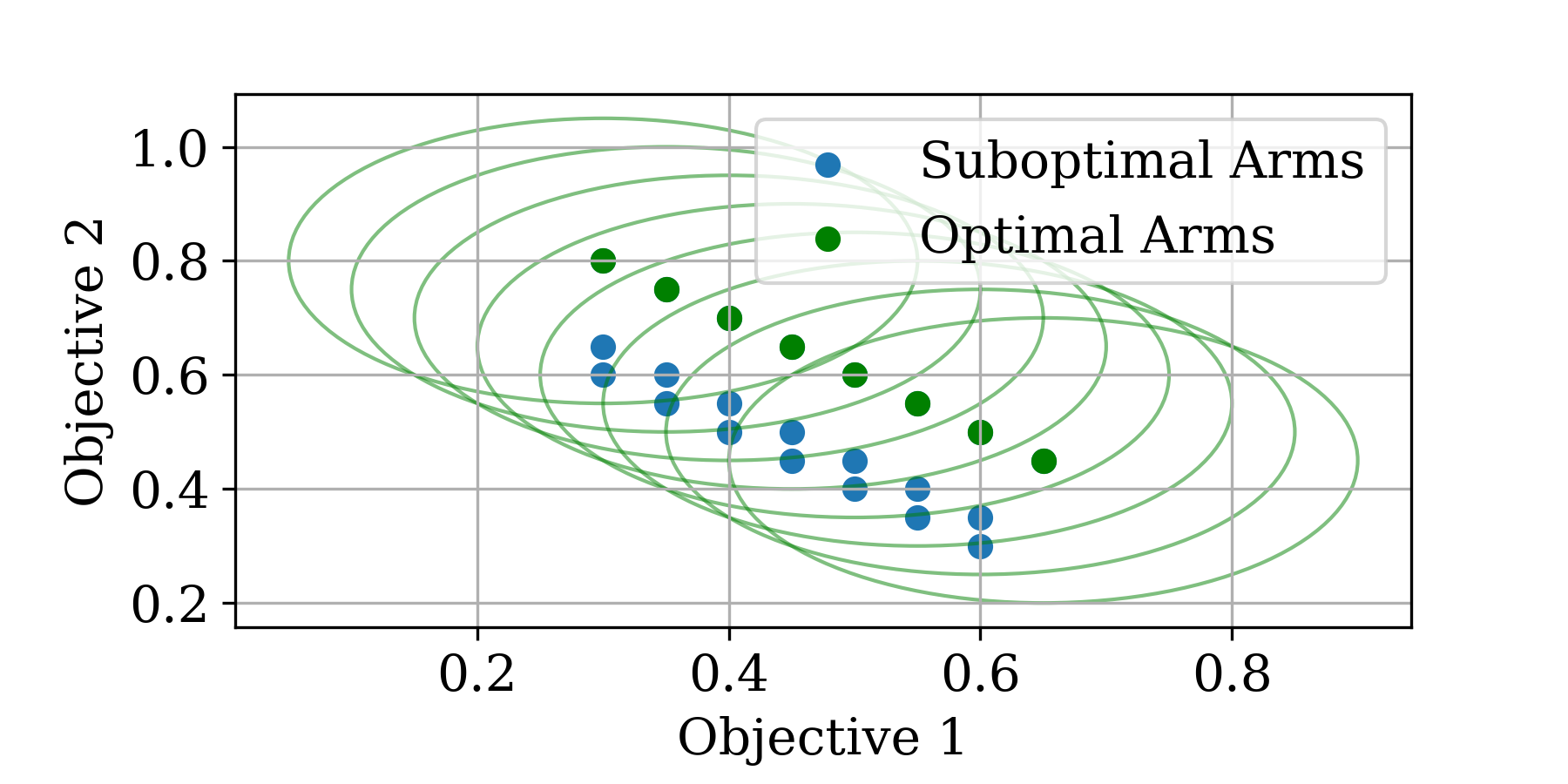}
        \caption{Environment 7}
    \end{subfigure}

    \begin{subfigure}{0.49\linewidth}
        \centering
        \includegraphics[width=\linewidth]{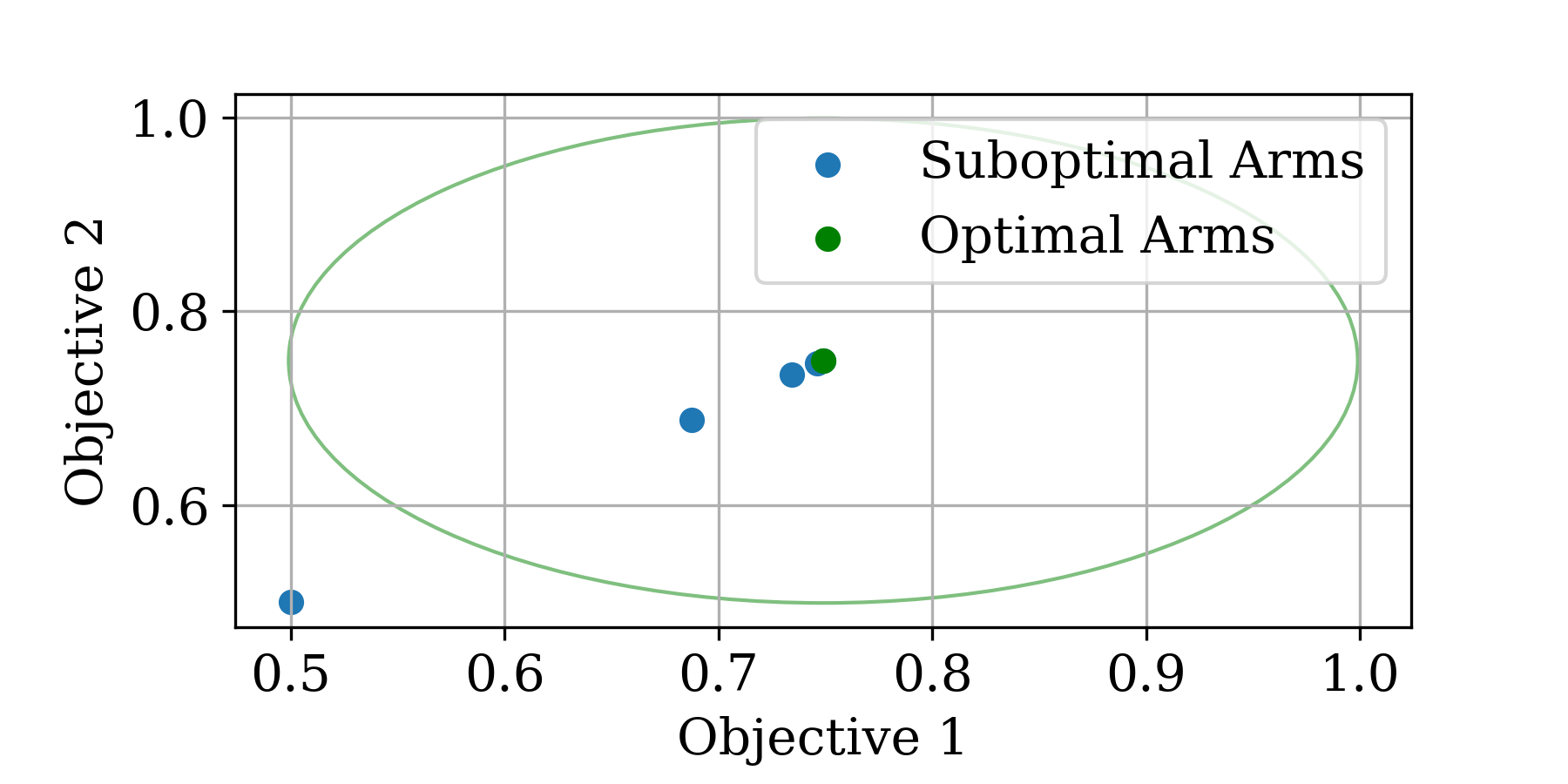}
        \caption{Environment 8}
    \end{subfigure}

    \caption{Visualizations of the bi-objective synthetic maximization benchmarking environments proposed in $\left[ \text{Kone \textit{et al}., 2023}\right]$. We use Gaussian reward distributions with diagonal covariance matrix $\sigma^2 I_d$ with $\sigma^2=0.25$. For better readability, we visualize $\sigma^2$ as an ellipse around the mean only for the optimal arms. Environment 4 cannot be visualized in the same way due to its high-dimensional nature ($D=10$).}
    \label{fig:EgeExpEnvs}
\end{figure}

\subsection{Detailed Synthetic Benchmarking Environment Specifications}
\label{sec:supp:env_specifications}

\begin{table}[ht]
  \centering
  \caption{Per-environment number of arms $K$ and objectives $d$, hardness $H(\nu)$, and sampling budget $B$, with
    $B=\operatorname{round}(H(\nu))$.}
  \label{tab:sup:environment-hardness-budgets}
  \begin{tabular}{crrrr}
    \toprule
    Environment & $K$ & $d$ & $H(\nu)$ & Budget $B$ \\
    \midrule
     1 & $60$  & $2$  & $15\,217.761905$  & $15\,218$ \\
     2 & $10$  & $2$  & $2\,543.757464$   & $2\,544$ \\
     3 & $200$ & $2$  & $685\,035.731162$ & $685\,036$ \\
     4 & $50$  & $10$ & $15\,977.302551$  & $15\,977$ \\
     5 & $20$  & $2$  & $197\,372.624118$ & $197\,373$ \\
     6 & $10$  & $2$  & $51\,966.414556$  & $51\,966$ \\
     7 & $22$  & $2$  & $7\,000.000000$   & $7\,000$ \\
     8 & $5$   & $2$  & $237\,957.543887$ & $237\,958$ \\
    \bottomrule
  \end{tabular}
\end{table}

\paragraph{Environment 1: Arms on a convex Pareto set.}
This environment features $K=60$ arms and $D=2$ objectives. The first 20 parameters $x_1, \dots, x_{20}$ are chosen to be equally spaced in the interval $[0.55, 0.95]$. For arms $i=1, \dots, 10$, the means are defined as $\boldsymbol{\mu}_i := (x_i^2, 1/(4x_i^2))^\intercal$. The remaining arms $\boldsymbol{\mu}_{11}, \dots, \boldsymbol{\mu}_{60}$ are chosen from the set $\{(x, y) \in [0.1, 0.8]^2 : xy \leq 1/5\}$.

\paragraph{Environment 2: Group of sub-optimal arms with unique optimal dominators.}
Here, $K=10$ and $D=2$, with a Pareto set size of $|\ttpTrueFront|=2$. The structure dictates that for each sub-optimal arm $i$, there is a unique optimal arm $j$ such that $\boldsymbol{\mu}_i \prec \boldsymbol{\mu}_j$. The optimal arms are set as $\boldsymbol{\mu}_1 := (0.4, 0.75)^\intercal$ and $\boldsymbol{\mu}_2 := (0.75, 0.4)^\intercal$. For $i=1, \dots, 4$, the sub-optimal arms are defined as $\boldsymbol{\mu}_{2i+1} := (0.45 + 0.2^i, 0.35 - 0.2^i)^\intercal$ and $\boldsymbol{\mu}_{2i+2} := (0.10 + 0.20^i, 0.70 - 0.20^i)^\intercal$.

\paragraph{Environment 3: Many arms on the unit circle.}
This setting includes a large number of arms, $K=200$, with $D=2$. The instance is generated as an isotropic multivariate normal. Angles $\beta_1, \dots, \beta_{20}$ are evenly spaced in $[\pi/12, \pi/2 - \pi/12]$, while $\beta_{21}, \dots, \beta_{200}$ are evenly spaced in $[\pi/2 + \pi/6, 2\pi - \pi/6]$. The mean vectors are defined as $\boldsymbol{\mu}_i := (\cos(\beta_i), \sin(\beta_i))^\intercal$.

\paragraph{Environment 4: High dimension ($D=10$).}
This high-dimensional environment sets $K=50$ and $D=10$. The mean vectors are generated uniformly: $(\boldsymbol{\mu}_1, \dots, \boldsymbol{\mu}_{30}) \sim \mathcal{U}([0.2, 0.45]^{10})^{\otimes 30}$ and $(\boldsymbol{\mu}_{31}, \dots, \boldsymbol{\mu}_{50}) \sim \mathcal{U}([0.55, 0.75]^{10})^{\otimes 20}$.

\paragraph{Environment 5: 2 clusters of arms.}
With $K=20$ arms, this environment samples means from two distinct clusters: $(\boldsymbol{\mu}_1, \dots, \boldsymbol{\mu}_{10}) \sim \mathcal{U}([0.2, 0.4]^2)^{\otimes 10}$ and $(\boldsymbol{\mu}_{11}, \dots, \boldsymbol{\mu}_{20}) \sim \mathcal{U}([0.5, 0.7]^2)^{\otimes 10}$.

\paragraph{Environment 6: All arms are optimal.}
In this setting, $K=10$ and $D=2$, and every arm belongs to the Pareto set. The means are defined such that for any arm $i$, $\mu_i^1 = 0.75 - 0.65^i$ and $\mu_i^2 = 0.25 + 0.65^i$.

\paragraph{Environment 7: All arms have the same sub-optimality gap.}
This environment uses $K=22$ and $D=2$, choosing $\boldsymbol{\mu}$ such that $\Delta_1 = \dots = \Delta_{22}$. The means are constructed as follows: for $i=1, \dots, 8$, $\boldsymbol{\mu}_i := (0.3 + c_i, 0.8 - c_i)^\intercal$ where $c_i = (i-1) * c$; for $i=9, \dots, 15$, $\boldsymbol{\mu}_i := (0.25 + c_{i-8}, 0.7 - c_{i-8})^\intercal$; and for $i=16, \dots, 22$, $\boldsymbol{\mu}_i := \boldsymbol{\mu}_{i-7} - (0, -0.05)^\intercal$.

\paragraph{Environment 8: Geometric progression with a single optimal arm.}
This instance features $K=5$ and $D=2$, where for any $i \in \{1, \dots, 5\}$, the components are identical: $\mu_i^1 = \mu_i^2 = 0.75 - 0.25^i$. This results in a unique optimal arm.\\

Figure~\ref{fig:EgeExpEnvs} provides a visual representation of the reward landscapes for the seven bi-objective maximization environments, excluding the high-dimensional Environment 4. To facilitate reproducibility and encourage further research, we have made the full implementations of these benchmarking environments publicly available. The source code can be accessed in our GitHub repository\footnote{\url{https://github.com/LennertSaerens/TTPFTS}}.

\subsection{TTPFTS Hyperparameter Selection Robustness and Ablation}
\label{sec:supp:rhobustness}

We assess TTPFTS's sensitivity to its hyperparameter $\rho$ by evaluating 11 values across the eight synthetic environments. The experimental setup matches that in Section 5 of the main text: rewards are Gaussian with covariance $\sigma^2 I_d$, where $\sigma^2=0.25$, and the prior is uniform~\cite{Honda2013OptimalityOT}. We use a Gaussian likelihood with unknown mean and variance, which yields a $t$-distributed posterior. As in the main experiments, the sampling budget for each environment is its hardness $H(\nu)$.

\begin{figure}[ht]
    \centering
    \includegraphics[width=1\linewidth]{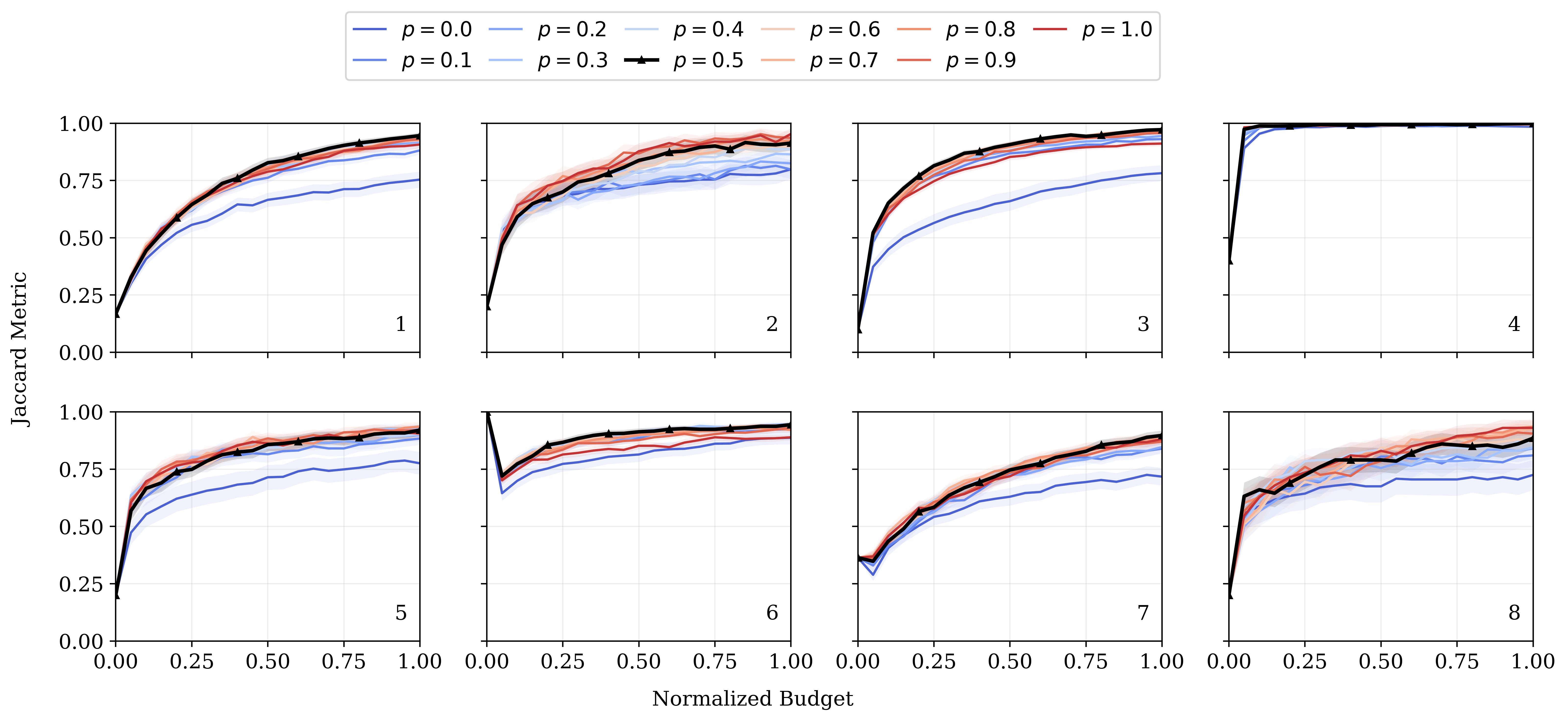}
    \caption{Mean Jaccard performance of TTPFTS across 11 values of $\rho$ in the eight synthetic benchmarking environments, shown over the normalized budget. Rewards are Gaussian with covariance $\sigma^2 I_d$, where $\sigma^2=0.25$. Results are averaged over 100 experimental trajectories; shaded areas show 95\% confidence intervals around the mean.}
    \label{fig:ttpfts_rho_sweep}
\end{figure}

TTPFTS performs well across a broad range of $\rho$ values, particularly between 0.1 and 0.9. The endpoints also provide ablations of the sampling rule. At $\rho=0$, the algorithm selects arms from the second posterior front $\ttpPosteriorFront{2}{t}$, falling back to the first front if the second is empty. At $\rho=1$, it selects arms only from the first posterior front $\ttpPosteriorFront{1}{t}$.

Across most environments, setting $\rho=0$ reduces performance relative to other values. Sampling only from the second front gives less attention to arms on the first front, which are currently most likely to be Pareto optimal, and can therefore slow the refinement of the estimated Pareto set. Environment 6 is an exception: all arms in this environment are Pareto optimal, so the second front is often empty and the fallback selects from the first front, resulting in good performance.

The effect of setting $\rho=1$ is less pronounced. Overall, the sweep suggests that TTPFTS is not highly sensitive to $\rho$ across the tested range, while allocating samples to both fronts can be beneficial.

\subsection{Additional PSI Performance Metrics and Results}
\label{sec:supp:additional_PSI_metrics}
\subsubsection{Bernoulli Metric}
An alternative way to evaluate the quality of a Pareto set estimate $\ttpEmpFront{}$ with respect to the actual set of Pareto optimal arms $\ttpTrueFront$ is to consider the estimation as a \textit{Bernoulli} trial. The trial is considered a success if the recommendation $\ttpEmpFront{}$ is exactly equal to the ground-truth set of Pareto optimal arms $\ttpTrueFront$. Otherwise, it is considered a failure. A score of one is associated with a success, and a score of zero is associated with a failure. These scores for successes and failures can be averaged over multiple runs of the algorithm to obtain the averaged empirical success rate of the algorithm. This yields a proportion that gives a good indication of the algorithm's ability to give perfect recommendations.

\begin{equation}
\label{eq:bern}
\bern(\ttpTrueFront,\ttpEmpFront{}):=\left\{\begin{matrix}
0 \quad \text{if } \ttpEmpFront{} \neq \ttpTrueFront \\
1 \quad \text{if } \ttpEmpFront{} = \ttpTrueFront
\end{matrix}\right.
\end{equation}

The Bernoulli metric $\bern$, presented in \Cref{eq:bern}, is a very strict metric: only perfect Pareto set estimates $\ttpEmpFront{}$ are rewarded. This means that if the estimate contains all but one of the Pareto optimal arms, the Bernoulli metric will be equal to 0. Similarly, if the estimate contains all Pareto optimal arms and some additional suboptimal arm, the Bernoulli metric will also be equal to 0. From an intuitive point of view, Pareto set estimates that are far from perfect receive the same score as estimates that are nearly perfect. Consequently, the Bernoulli metric does not provide additional insight into the acquisition process of the Pareto set by the MOMAB PSI algorithm. However, we argue that the Bernoulli metric is a strong indicator of the algorithm's ability to identify the Pareto optimal arms, that is particularly useful when the goal is to perfectly identify the Pareto set.

\begin{figure}[ht]
    \centering
    \includegraphics[width=1\linewidth]{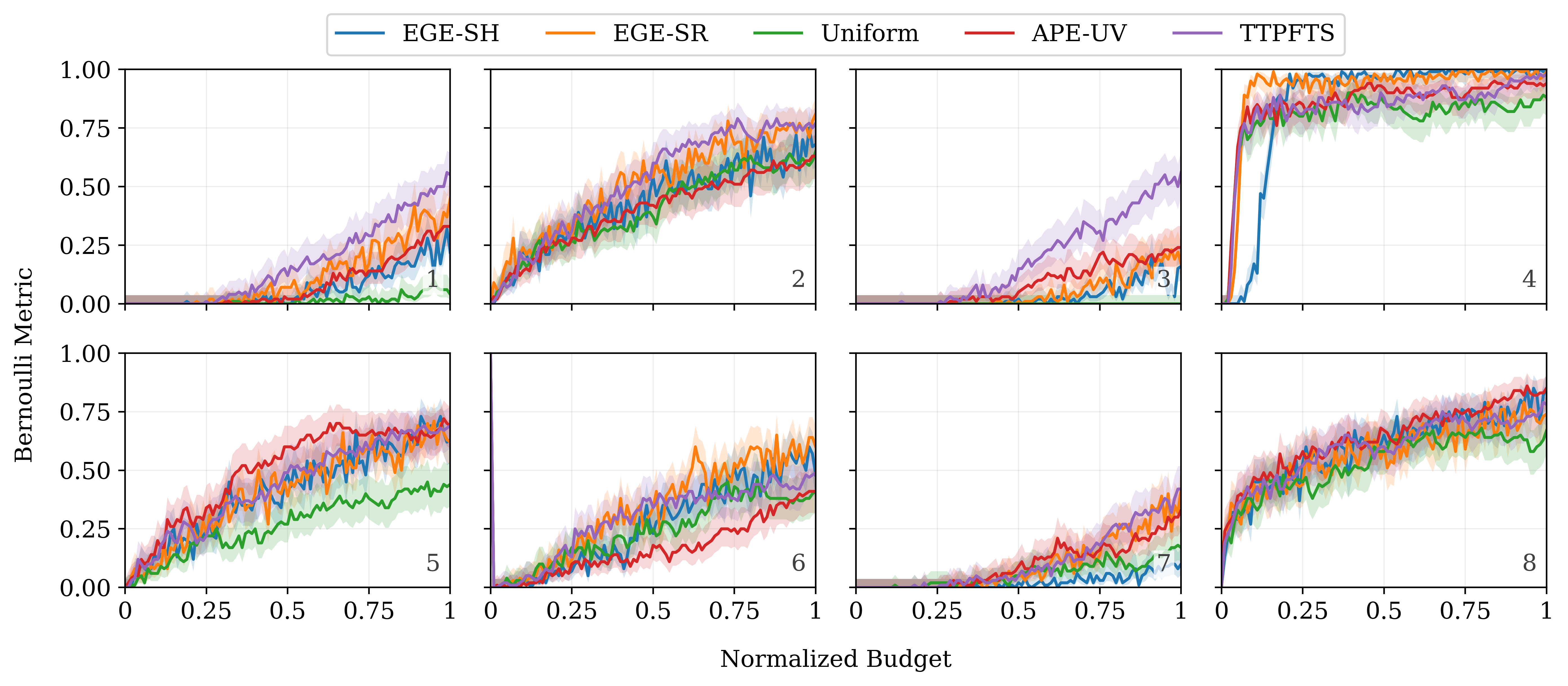}
    \caption{Performance of the MOMAB PSI algorithms with respect to the Bernoulli metric in the eight synthetic benchmarking environments. We use Gaussian reward distributions with diagonal covariance matrix $\sigma^2 I_d$ with $\sigma^2=0.25$. Results are averaged over 100 experimental trajectories, with shaded areas indicating the 95\% confidence interval around the mean performance.}
    \label{fig:bernoulli_all_envs}
\end{figure}

\Cref{fig:bernoulli_all_envs} illustrates the evolution of the Bernoulli metric across the synthetic benchmarking environments. In general, these results corroborate the findings from the Jaccard metric analysis, confirming that TTPFTS achieves results similar to the APE and EGE-SR algorithms, despite the inherent challenge of operating without a fixed budget or knowledge of the reward variance.

A salient observation across almost all environments (notably Environment 1, 2, 5, 6, and 8) is the contrast in stability between the methods. The EGE algorithm variants exhibit significant performance variability, characterized by jagged learning curves. This stems from the fixed-budget nature of EGE: its performance is highly sensitive to the specific budget $T$. For certain budgets, EGE may achieve high quality Pareto set estimates, while for slightly different budgets, performance can drop. In contrast, TTPFTS demonstrates a smoother improvement in success rate as it gathers more evidence. This stability highlights a key advantage of the anytime framework: in real-world scenarios where problem hardness is unknown and selecting an optimal budget is difficult, TTPFTS offers a robust and consistent alternative that is able to uphold both performance and flexibility.

In the high-dimensional Environment 4 ($D=10$), TTPFTS exhibits a more conservative learning profile compared to the EGE algorithm variants regarding the Bernoulli metric. This behavior is likely attributable to the sparsity of the action space in high dimensions, which can make the posterior separation of the top-two fronts less stable. However, this result should be interpreted alongside the Jaccard metric, where TTPFTS remains highly competitive. This divergence suggests that while TTPFTS may occasionally miss a single arm or include a near-optimal contender, resulting in a Bernoulli score of 0, the overall quality of the estimated Pareto set remains high.

\subsubsection{Misidentification Metric}
Another metric to evaluate the quality of Pareto set estimates $\ttpEmpFront{}$ in the context of the MOMAB PSI problem is the misidentification. The use of this metric was proposed in \cite{Kone2023BanditPS}. This metric offers a more granular perspective than the Bernoulli metric and complements the Jaccard metric by quantifying the extent to which individual arms are correctly classified as Pareto (sub)optimal.

\begin{table}[ht]
    \centering
    \caption{Classification of arms based on Pareto set estimate $\ttpEmpFront{}$ and ground-truth Pareto set $\ttpTrueFront$.}
    \begin{tabular}{ll}
        \toprule
        Classification & Condition \\
        \midrule
        True Positive (TP) & $a \in \ttpEmpFront{}$ and $a \in \ttpTrueFront$ \\
        False Positive (FP) & $a \in \ttpEmpFront{}$ and $a \notin \ttpTrueFront$ \\
        True Negative (TN) & $a \notin \ttpEmpFront{}$ and $a \notin \ttpTrueFront$ \\
        False Negative (FN) & $a \notin \ttpEmpFront{}$ and $a \in \ttpTrueFront$ \\
        \bottomrule
    \end{tabular}
    \label{tab:classification_arms}
\end{table}

Given a Pareto set estimate, every arm $a \in [K]$ can be classified into one of four categories, based on the comparison between the algorithm's recommendation $\ttpEmpFront{}$ and the ground-truth Pareto set $\ttpTrueFront$. This is shown in \Cref{tab:classification_arms}. Using these classifications, we define the misidentification rate $M$ as the proportion of arms that are incorrectly classified as Pareto (sub)optimal:
\begin{equation}
    M(\ttpTrueFront,\ttpEmpFront{}) := \frac{\left|\text{FP}\right| + \left|\text{FN}\right|}{K}
\end{equation}
This metric can be averaged over multiple runs of the algorithm to yield a robust estimate of the average misidentification rate. Unlike the Bernoulli metric, which is binary, the misidentification metric provides a finer-grained signal on how well the algorithm is learning to distinguish optimal from suboptimal arms over time.

\begin{figure}[ht]
    \centering
    \includegraphics[width=1\linewidth]{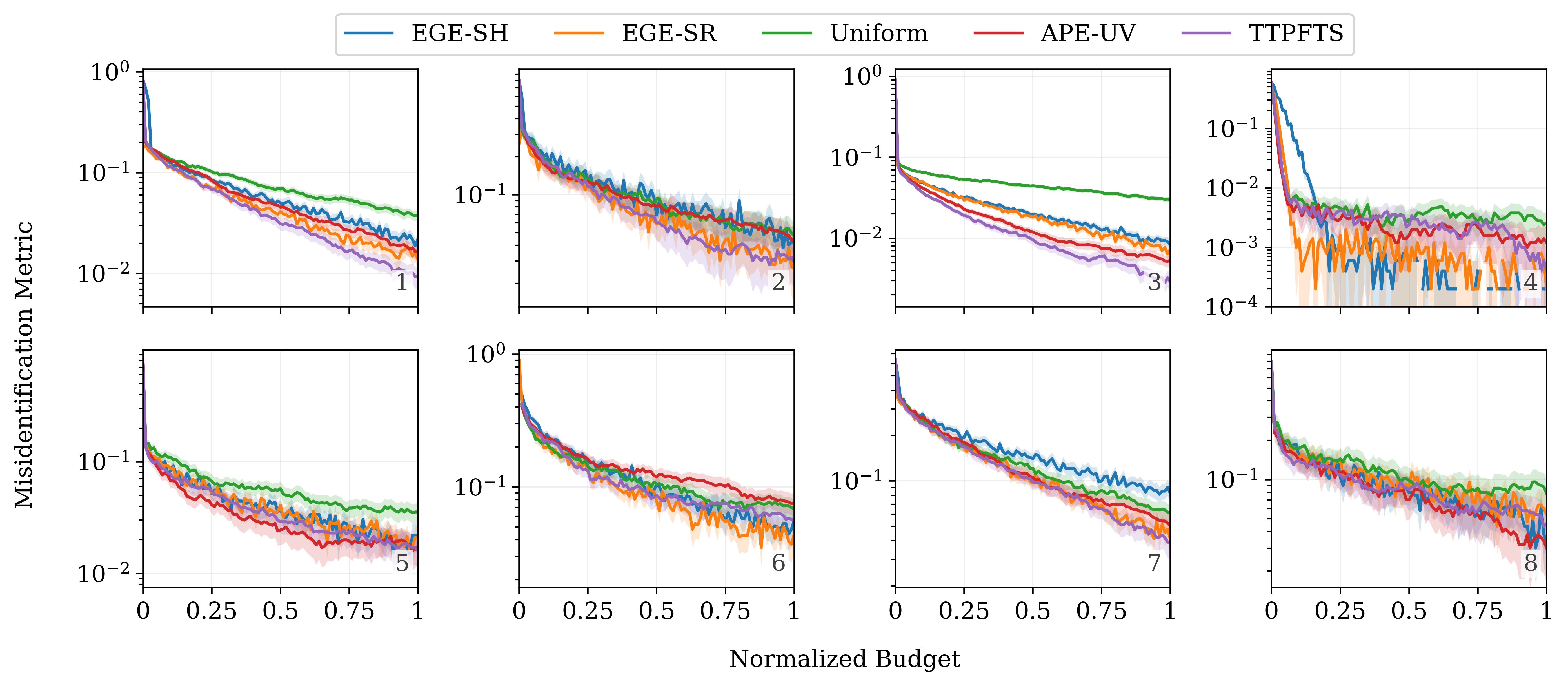}
    \caption{Performance of the MOMAB PSI algorithms with respect to the Misidentification metric in the eight synthetic benchmarking environments. We use Gaussian reward distributions with diagonal covariance matrix $\sigma^2 I_d$ with $\sigma^2=0.25$. Results are averaged over 100 experimental trajectories, with shaded areas indicating the 95\% confidence interval around the mean performance.}
    \label{fig:misclassification_all_envs}
\end{figure}

\Cref{fig:misclassification_all_envs} presents the evolution of the misidentification rate across the eight benchmarking environments. These results corroborate the trends observed in the Jaccard and Bernoulli analyses, confirming that TTPFTS delivers PSI performance comparable to the state-of-the-art EGE and APE algorithms. This parity is particularly significant given that the fixed-budget EGE baselines are optimized to minimize error at specific horizons, whereas TTPFTS operates under the stricter constraint of maintaining a valid anytime estimate.

A key advantage of the anytime framework is evident in the stability of the TTPFTS learning curves. Consistent with the Bernoulli results, the EGE variants display marked variability, particularly in Environments 2, 4, and 6 where the misidentification rate fluctuates significantly depending on the budget. In contrast, TTPFTS exhibits a smoother reduction in misidentification errors as it gathers evidence. This stability ensures reliable performance without the need for a priori knowledge of the optimal budget, a crucial feature for real-world deployment.

In terms of specific environments, Environment 3 ($K=200$) again highlights the efficacy of TTPFTS in large-scale settings. Here, TTPFTS achieves a distinctly lower misidentification rate compared to both EGE variants and APE, reinforcing the conclusion from the Jaccard analysis that its posterior-driven exploration is well-suited for larger action spaces. Conversely, in the high-dimensional setting of Environment 4 ($D=10$), TTPFTS shows a slightly higher misidentification rate than the EGE algorithms, equaling them by the end of sampling. However, when viewed in conjunction with the high Jaccard scores reported in the main text, this suggests that while TTPFTS may misclassify a small number of arms near the decision boundary, the overall quality and utility of its recommended Pareto set remain high.

\section{Molecular Discovery Setting}
\label{sec:supp:mols}
\subsection{Formalization of the Multi-objective Multi-agent Nature}
In the main text, we demonstrate the application of TTPFTS in exploring an ultra-large synthesis-on-demand molecular library. Unlike the synthetic benchmarks, which are evaluated in a standard single-agent framework, this chemical space is generated combinatorially, akin to a Cartesian product: a vast space of product molecules is defined by reacting all valid combinations from smaller, discrete sets of molecular building blocks, called reagents~\cite{GRYGORENKO2020101681}. To clarify the mechanics of this experiment, we formalize this domain as a cooperative multi-agent MOMAB PSI problem.

Let $M$ denote the number of reaction components. We deploy an independent TTPFTS agent for each component $c \in \{1, \dots, M\}$. Each agent $c$ faces a distinct action space $\mathcal{A}_c$, where each arm $a_c \in \mathcal{A}_c$ corresponds to a unique chemical reagent available for that specific component.

At each time step $t$, every agent $c$ simultaneously and independently selects a reagent $a_{c,t} \in \mathcal{A}_c$ according to its own TTPFTS policy. These selections form a complete combinatorial action tuple, $\boldsymbol{a}_t = (a_{1,t}, \dots, a_{M,t}) \in \mathcal{A}_1 \times \dots \times \mathcal{A}_M$, which defines a fully specified product molecule. The properties of this assembled molecule are then evaluated, yielding a $D$-dimensional deterministic reward $\boldsymbol{X}_t=f(\boldsymbol{a}_t) \in \mathbb{R}^D$.

Since in the main text we are considering a three-component reaction, in our experiment $M=3$. Furthermore, in our case, the number of objectives $D=2$, representing Structural Similarity and LogP.

\paragraph{Stochasticity and Non-Stationarity.}
A defining feature of this formulation is the source of the reward stochasticity. From the perspective of a complete combinatorial action tuple, $\boldsymbol{a}_t$, the evaluation function $f(\boldsymbol{a}_t)$ is deterministic. However, from the localized perspective of a single agent $c$, the reward vector observed after pulling arm $a_{c,t}$ is stochastic, as this observed reward $\boldsymbol{X}_t$ depends on the concurrent selections made by the \textit{other} agents, denoted as $\boldsymbol{a}_{-c,t}$.

Consequently, the reward observed by agent $c$ for pulling reagent $a_c$ is an evaluation of $f(a_c, \boldsymbol{a}_{-c,t})$. Because the other agents are learning from online interactions, continuously updating their posterior distributions and shifting their sampling between their respective top-two Pareto fronts, the distribution of $\boldsymbol{a}_{-c,t}$ changes over time. Thus, the expected marginal utility of a reagent $a_c$ is not fixed; the reward distribution perceived by agent $c$ is inherently \textit{non-stationary}.

\subsection{Pareto Front Extraction and Metric Computation.}
\label{sec:supp:mols_psi_uq_metrics}
To evaluate the algorithms' search efficiency over time, we compute the Jaccard similarity between the estimated Pareto front and the true Pareto front at each evaluation step. During the search process, every evaluated molecule and its corresponding multi-objective score are logged, preserving the exact chronological exploration order. In post-processing, we reconstruct the algorithm's estimated Pareto set at any given time step $t$ by applying non-dominated sorting to the collection of molecules evaluated up to that point. As each new evaluation is incorporated, any previously retained molecule that becomes strictly dominated is pruned from the active set. This rolling Pareto front represents the exact non-dominated subset of the agent's search history at step $t$. The learning curve is then generated by calculating the Jaccard metric between this step-wise estimated set and the ground-truth Pareto set.

To compute the uncertainty quantification metric in the combinatorial molecular discovery setting, we extend the Bhattacharyya coefficient to evaluate the joint product space. Because the library consists of 94 million possible products derived from three reaction components, explicitly tracking posteriors for all combinations is computationally prohibitive. Instead, we leverage the independence of the agents: for any given product, its 6-dimensional joint posterior (representing two objectives across three components) is constructed by concatenating the independent 2-dimensional Gaussian mean and variance vectors of its constituent reagents. We construct these 6-dimensional product posteriors and immediately apply non-dominated sorting to retain only those products belonging to the local top-two Pareto fronts, discarding the vast majority of strictly dominated combinations. After aggregating these candidates, we compute the global first and second Pareto fronts for the joint space. The final uncertainty quantification metric is then calculated as the average Bhattacharyya coefficient across all pairs of products between the first and second fronts.

\subsection{MolPAL Baseline Configuration}
\label{sec:supp:molpal_config}
The surrogate model and hyperparameters for the MolPAL baseline were selected through a systematic two-phase tuning process conducted on a reservoir sample of 200,000 molecules drawn from the 94-million quinazoline library. In the first phase, four tree-based ensemble architectures (Random Forest, Extra Trees, Gradient Boosted Regression, and Histogram-based Gradient Boosting) were compared using an offline evaluation protocol that mirrors the active learning training schedule, training each model on incremental subsets of 5,000 to 50,000 molecules and measuring $R^2$, RMSE, and wall-clock time on a held-out set for both objectives. Histogram-based Gradient Boosting achieved the highest predictive accuracy (similarity $R^2 = 0.960$; LogP $R^2 = 0.667$ at 50k training points) while being 5–18× faster to train (35s vs. 200–625s), a critical advantage given that the surrogate is retrained from scratch each iteration and inference over 94M molecules requires ~1.5h per cycle. In the second phase, a focused grid search over 36 Histogram-based Gradient Boosting configurations yielded the final setting (learning rate = 0.1, 500 iterations, unlimited depth, 63 leaf nodes), raising surrogate quality to $R^2 = 0.973$ for similarity and $R^2 = 0.721$ for LogP score.

Molecules were represented as Morgan fingerprints (2048-bit, radius 2) \cite{Rogers2010ECFP}. For acquisition, we employed non-dominated sorting, which ranks candidates front-by-front without scalarization. Because this strategy relies only on mean predictions, it is fully compatible with Histogram-based Gradient Boosting. Neural network surrogates were not pursued, as gradient boosting methods match or exceed Neural Network performance on fingerprint-based inputs while remaining far more computationally efficient \cite{Zhang2019LightGBM}.

\subsection{Additional Results}
\label{sec:supp:add_results_mols}
In the main text, we presented the results of a single representative experimental run to illustrate the efficacy of TTPFTS in exploring the ultra-large synthesis-on-demand molecular library. To demonstrate the robustness and consistency of our approach, and to verify that the reported performance is not the result of a favorable random seed, we provide visualizations of ten additional independent experimental runs in this section.

\Cref{fig:pareto_scatter_runs1_6,fig:pareto_scatter_runs7_10} display the Pareto fronts identified by TTPFTS and Random Search compared to the ground-truth Pareto front, obtained via exhaustive virtual screening, for ten additional independent trajectories. These qualitative results align strongly with the quantitative Jaccard metric results reported in the main text.

\begin{figure}[htbp]
    \centering

    \begin{subfigure}{0.45\linewidth}
        \centering
        \includegraphics[width=\linewidth]{figures/pareto_scatter_run0.png}
        \caption{Run 1}
    \end{subfigure}\hfill
    \begin{subfigure}{0.45\linewidth}
        \centering
        \includegraphics[width=\linewidth]{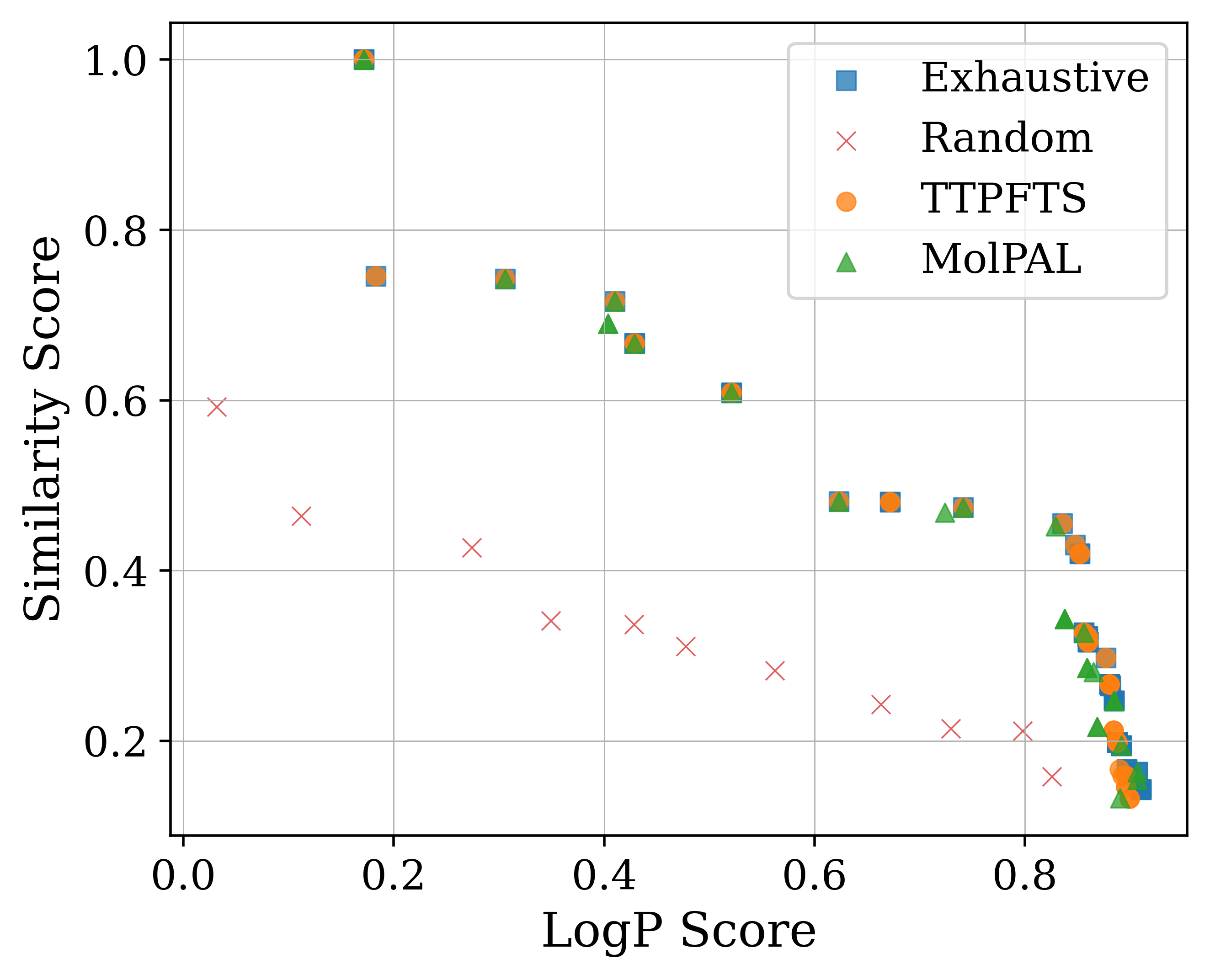}
        \caption{Run 2}
    \end{subfigure}

    \begin{subfigure}{0.45\linewidth}
        \centering
        \includegraphics[width=\linewidth]{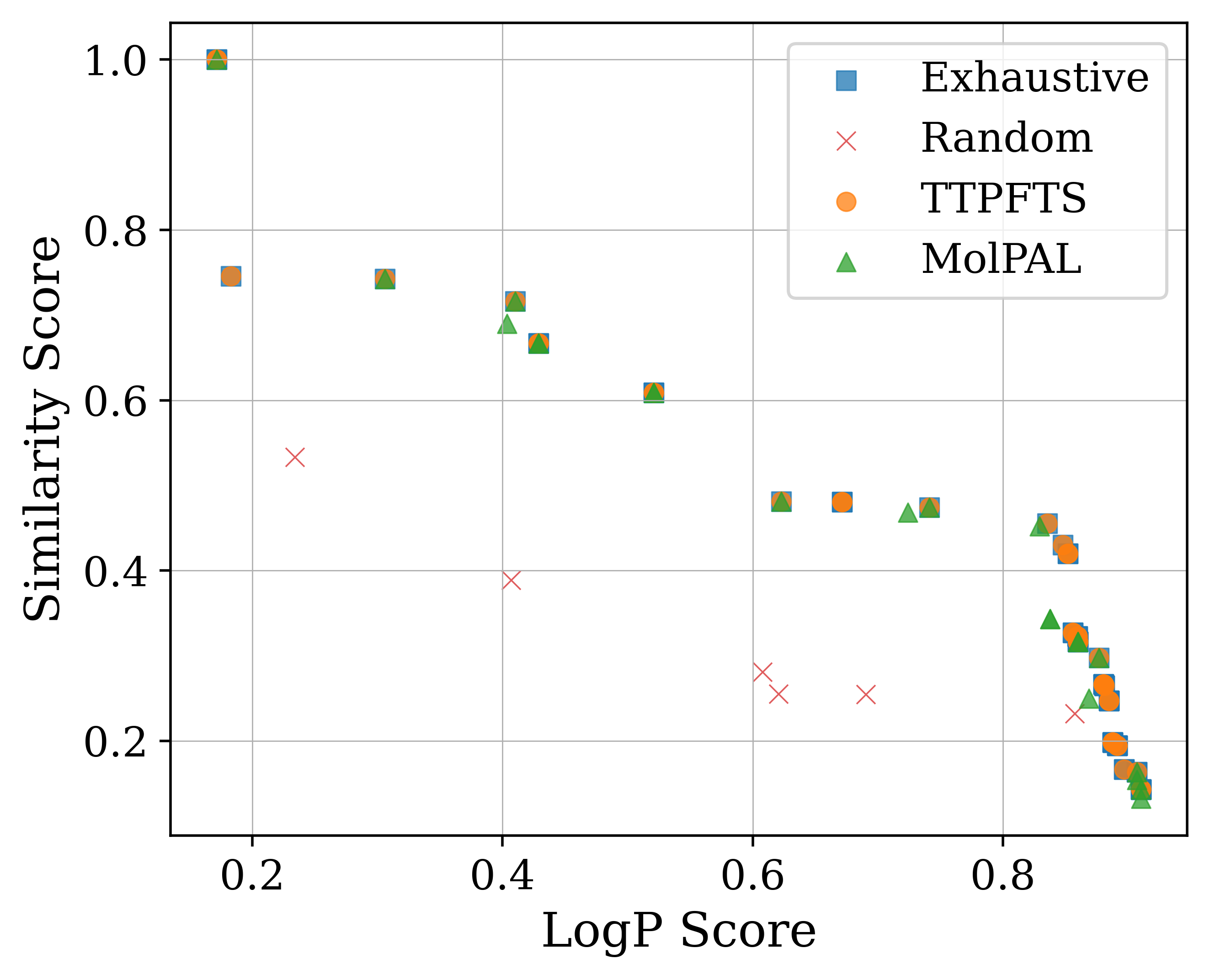}
        \caption{Run 3}
    \end{subfigure}\hfill
    \begin{subfigure}{0.45\linewidth}
        \centering
        \includegraphics[width=\linewidth]{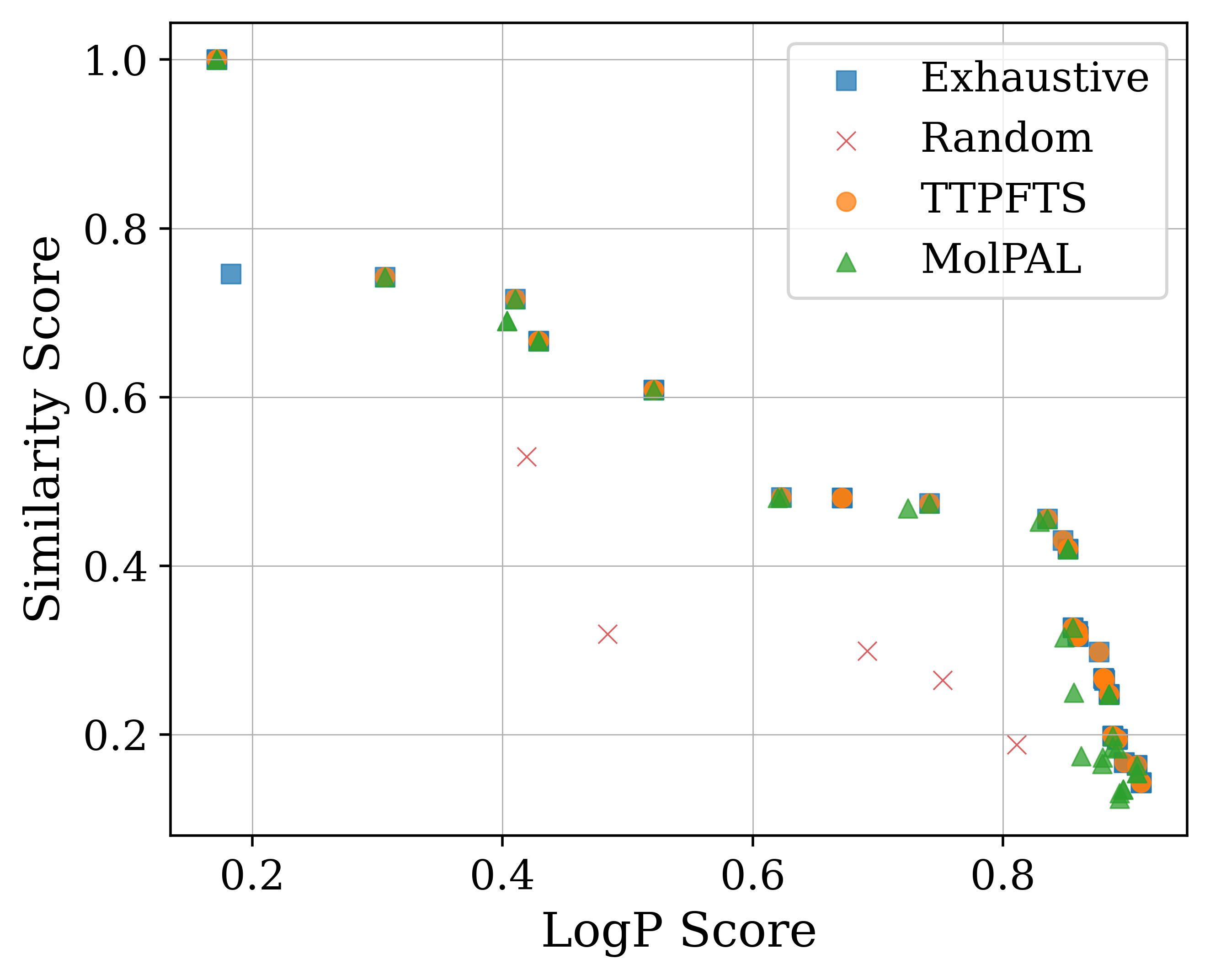}
        \caption{Run 4}
    \end{subfigure}

    \begin{subfigure}{0.45\linewidth}
        \centering
        \includegraphics[width=\linewidth]{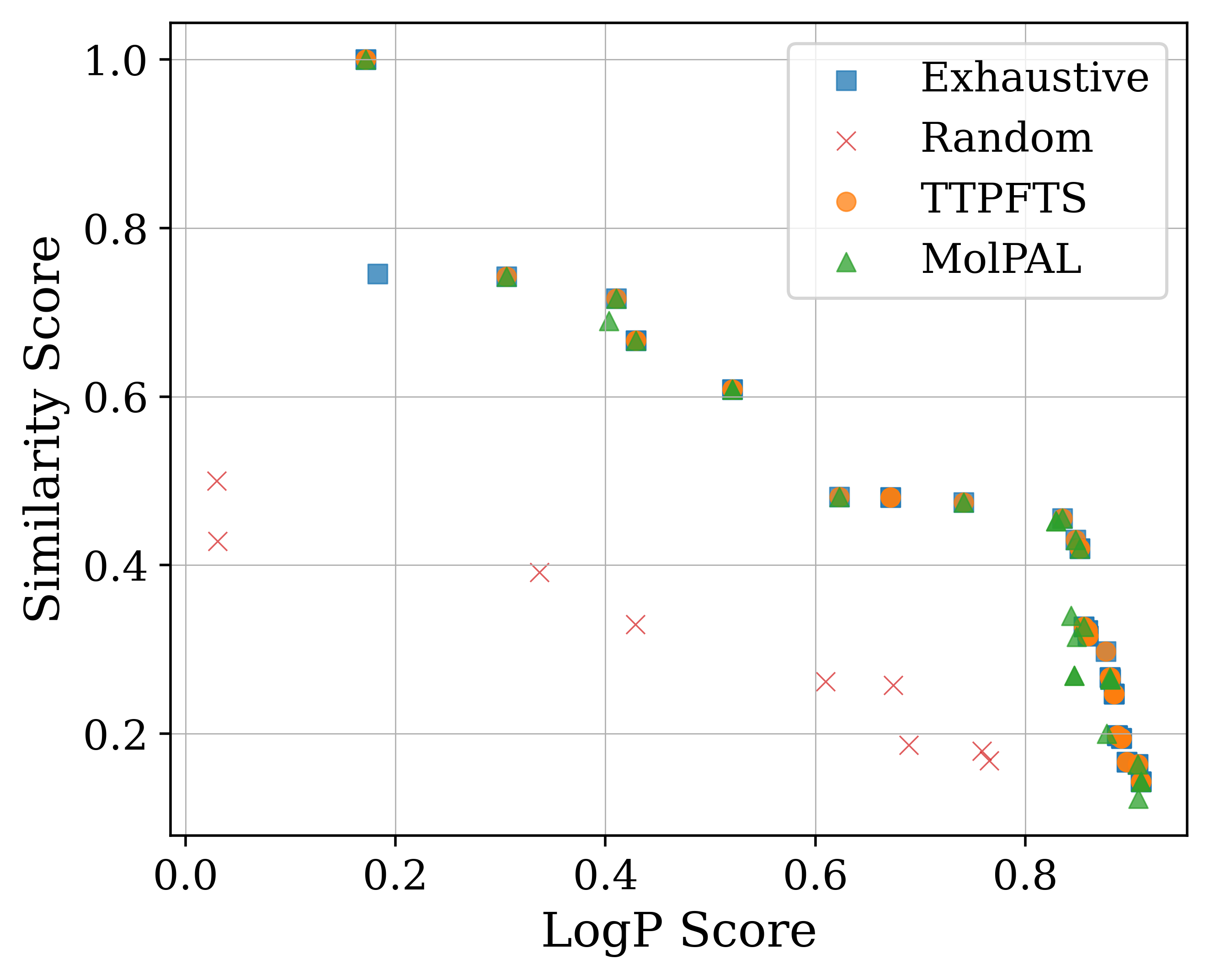}
        \caption{Run 5}
    \end{subfigure}\hfill
    \begin{subfigure}{0.45\linewidth}
        \centering
        \includegraphics[width=\linewidth]{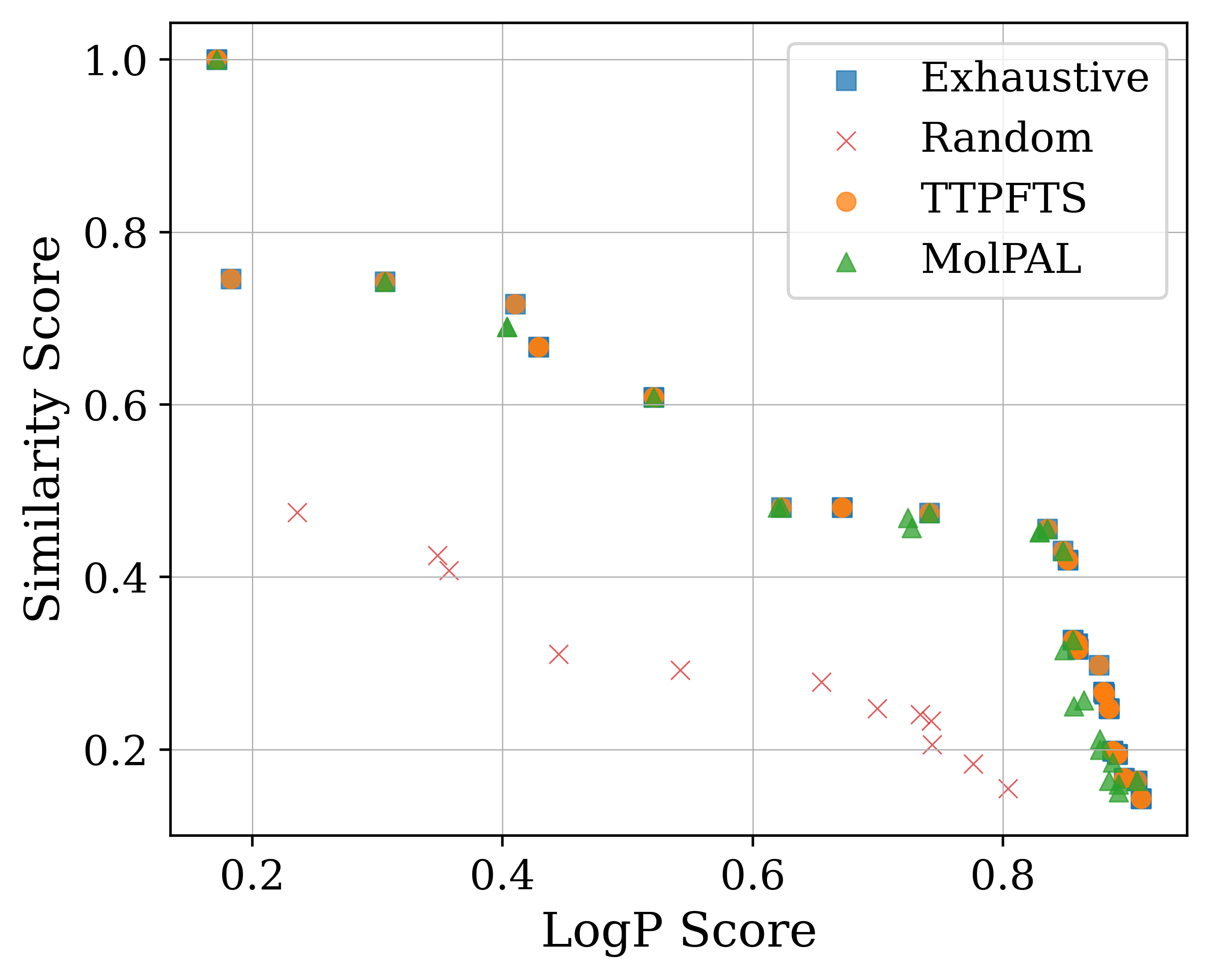}
        \caption{Run 6}
    \end{subfigure}

    \caption{Pareto optimal molecules identified by the Random Search baseline, MolPAL, and TTPFTS, compared to the ground-truth Pareto set of optimal molecules obtained through exhaustive virtual screening of the synthesis-on-demand molecular library. (Runs 1--6)}
    \label{fig:pareto_scatter_runs1_6}
\end{figure}

\begin{figure}[htbp]
    \centering

    \begin{subfigure}{0.45\linewidth}
        \centering
        \includegraphics[width=\linewidth]{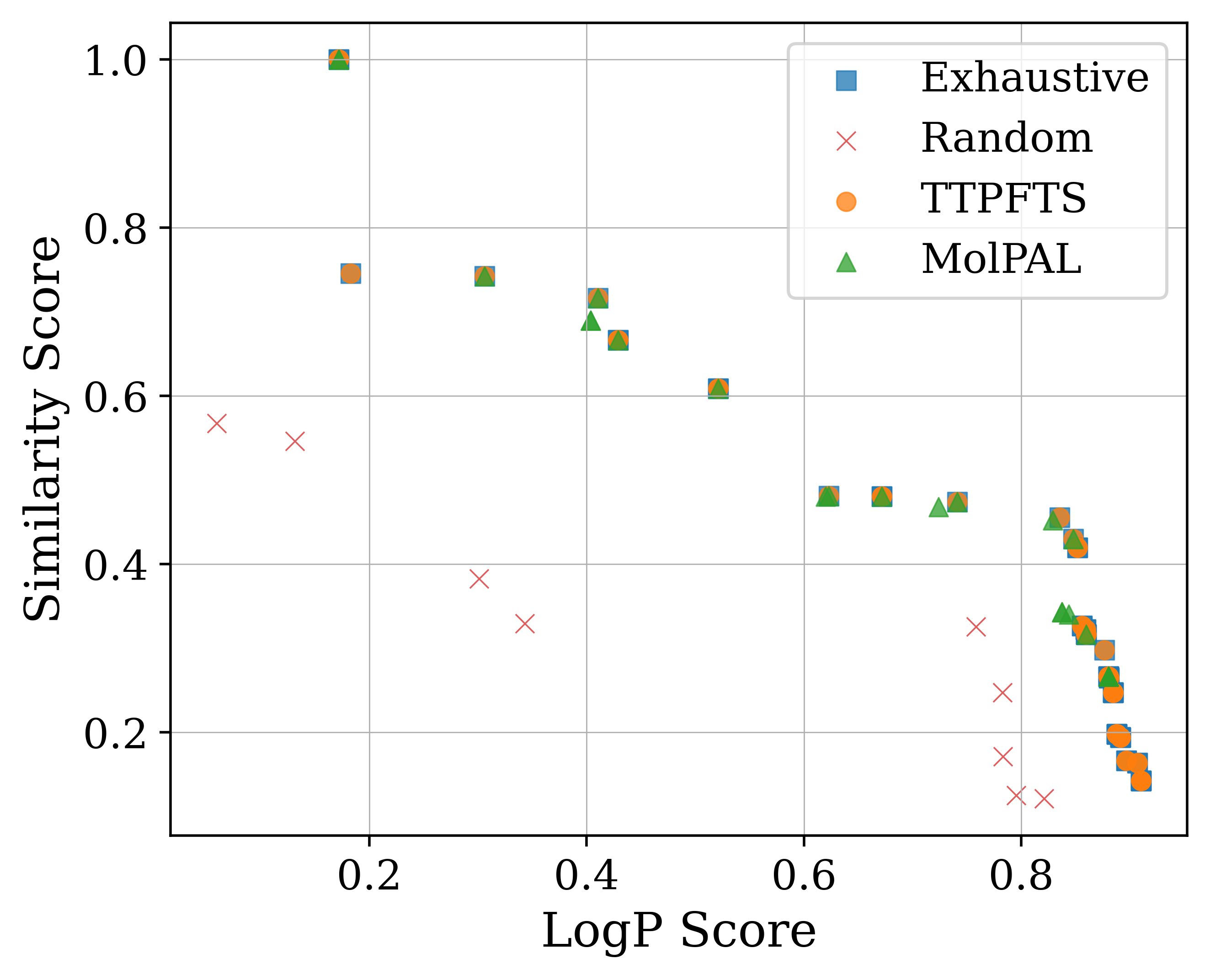}
        \caption{Run 7}
    \end{subfigure}\hfill
    \begin{subfigure}{0.45\linewidth}
        \centering
        \includegraphics[width=\linewidth]{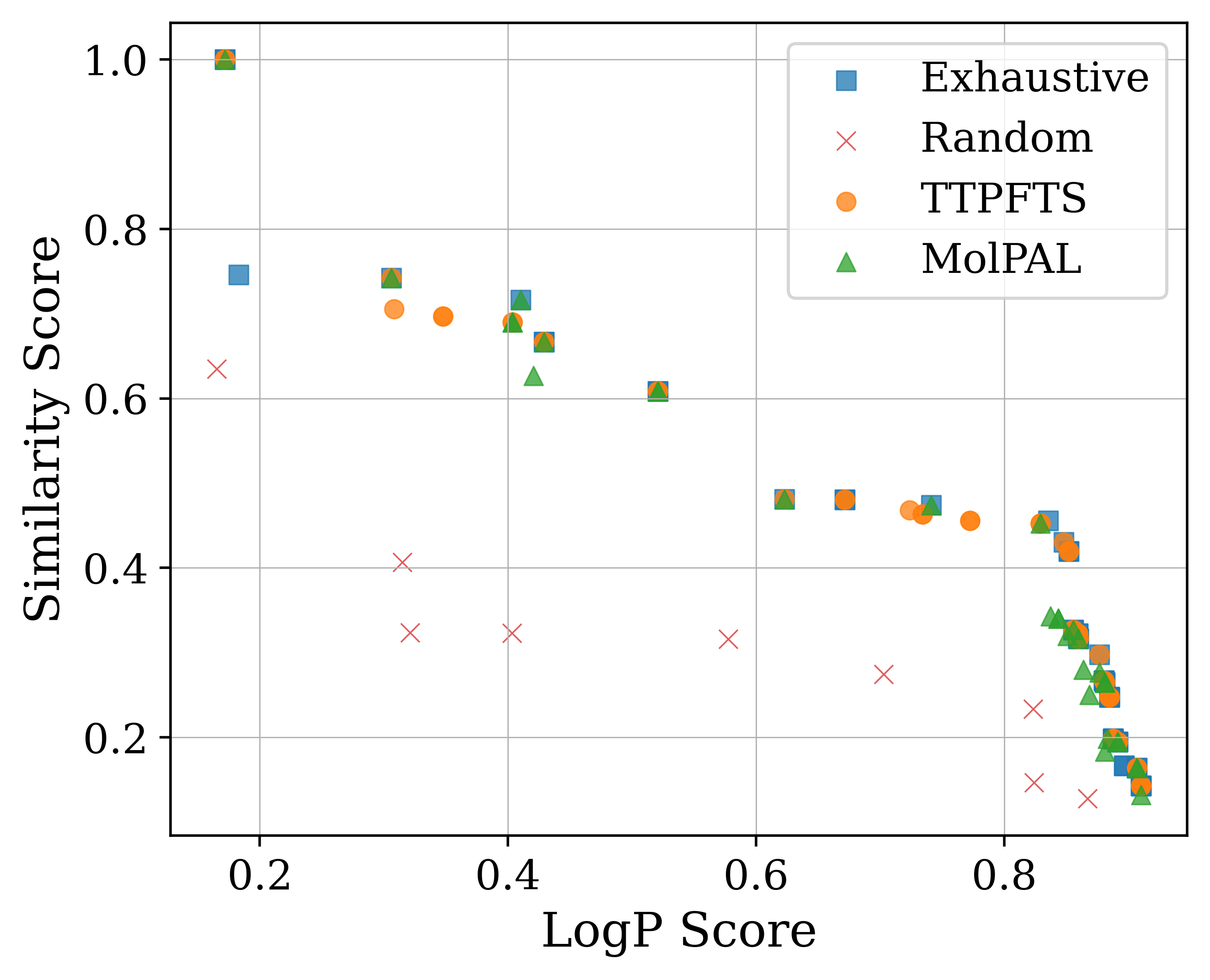}
        \caption{Run 8}
    \end{subfigure}

    \begin{subfigure}{0.45\linewidth}
        \centering
        \includegraphics[width=\linewidth]{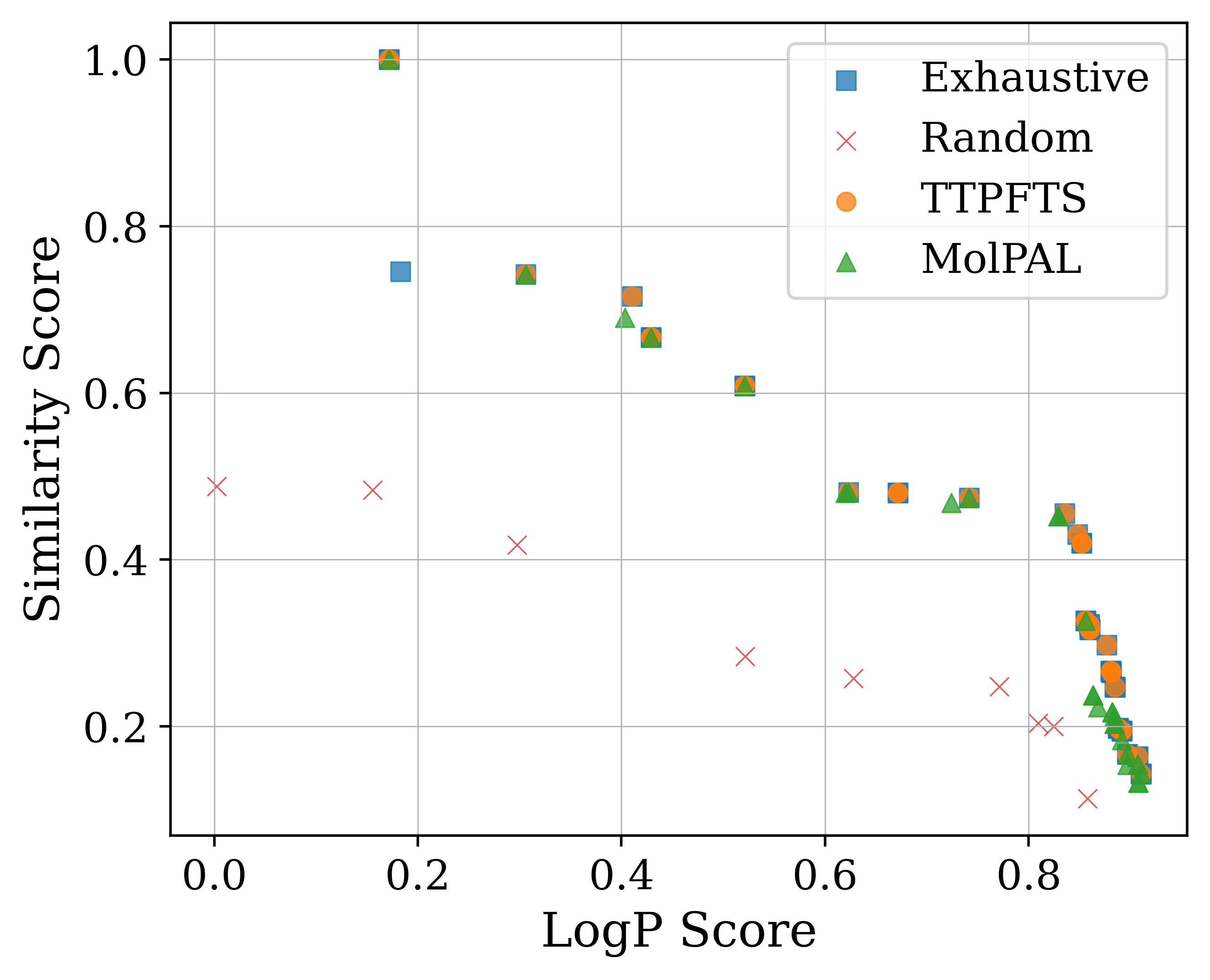}
        \caption{Run 9}
    \end{subfigure}\hfill
    \begin{subfigure}{0.45\linewidth}
        \centering
        \includegraphics[width=\linewidth]{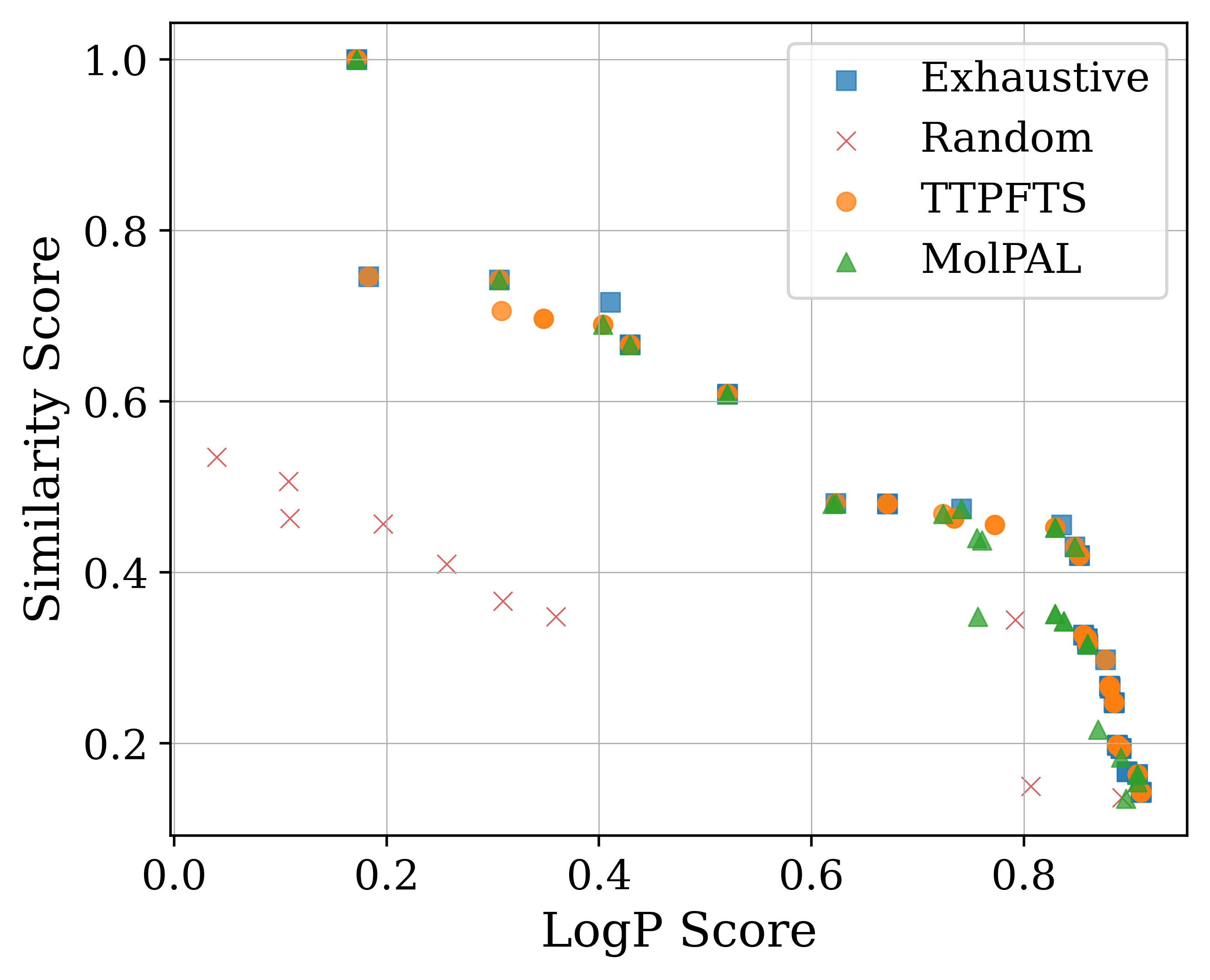}
        \caption{Run 10}
    \end{subfigure}

    \caption{Pareto optimal molecules identified by the Random Search baseline, MolPAL and TTPFTS, compared to the ground-truth Pareto set of optimal molecules obtained through exhaustive virtual screening of the synthesis-on-demand molecular library. (Runs 7--10)}
    \label{fig:pareto_scatter_runs7_10}
\end{figure}

Across all additional runs, TTPFTS consistently recovers the vast majority of the ground-truth Pareto optimal molecules. Notably, the algorithm achieves perfect identification of the entire Pareto front in Runs 3, 6, and 7, capturing every optimal trade-off between structural similarity and LogP. In Runs 4, 5, and 9, the algorithm misses only a very small number (one or two) of the 52 optimal molecules. This high success rate across diverse random seeds confirms that the performance illustrated in the main text is representative of the algorithm's stable behavior.

In the remaining runs where TTPFTS does not fully recover the exact ground-truth set (e.g., Runs 2, 8, and 10), the algorithm exhibits a graceful approximation behavior. The molecules that are identified in place of the missing optimal solutions lie in the immediate vicinity of the true Pareto front in the objective space. This indicates that even in slightly sub-optimal trajectories, TTPFTS successfully focuses its sampling on the most promising regions of the chemical space. Consequently, the resulting estimated Pareto set remains highly valuable for decision-making, providing drug discovery scientists with candidate molecules that are nearly indistinguishable from the optimal set in terms of their property profiles. Consistent with the findings in the main text, the Random Search baseline fails to identify a single ground-truth Pareto optimal molecule across all ten additional runs.

\section{Uncertainty Quantification Metric}
\label{sec:supp_uq_calc}
In this section, we provide the detailed mathematical formulation used to calculate the uncertainty quantification metric $\UQ$ introduced in the main text. The metric is defined as the average Bhattacharyya coefficient (overlap)~\cite{Bhattacharyya1943} between the posterior distributions of arms in the estimated first Pareto front $\ttpEmpFront{1}$ and the estimated second Pareto front $\ttpEmpFront{2}$.

\subsection{Posterior Assumptions}
Let the bandit instance $\nu$ consist of $K$ arms and $D$ objectives. At any time step $t$, the algorithm maintains a posterior distribution over the mean reward vector for each arm $a \in [K]$, conditioned on the history $\mathcal{H}^{(t)}$. In the main text, for brevity, we refer to the overlap between arms simply as $\Bhat(a_i, a_j)$. Here, we define this quantity based on the underlying posterior distributions. We model the posterior belief for any arm $a$ as a multivariate Gaussian distribution:
\begin{equation}
    p_a(\mathbf{x} \mid \mathcal{H}^{(t)}) = \mathcal{N}(\mathbf{x}; \boldsymbol{\mu}_a, \boldsymbol{\Sigma}_a),
\end{equation}
where $\boldsymbol{\mu}_a \in \mathbb{R}^D$ is the posterior mean vector and $\boldsymbol{\Sigma}_a \in \mathbb{R}^{D \times D}$ is the posterior covariance matrix. In our experiments, objectives are modeled independently, resulting in diagonal covariance matrices $\boldsymbol{\Sigma}_a = \text{diag}(\sigma_{a,1}^2, \dots, \sigma_{a,D}^2)$, though the derivation below holds for general covariance matrices~\cite{Fukunaga1990}.

\subsection{Bhattacharyya Distance and Coefficient}
To quantify the separability between two probability distributions $p_i$ and $p_j$, we utilize the \textit{Bhattacharyya distance} $D_B(p_i, p_j)$. For two multivariate Gaussian distributions $p_i = \mathcal{N}(\boldsymbol{\mu}_i, \boldsymbol{\Sigma}_i)$ and $p_j = \mathcal{N}(\boldsymbol{\mu}_j, \boldsymbol{\Sigma}_j)$, the Bhattacharyya distance is derived in~\cite{Fukunaga1990} as:

\begin{equation}
    D_B(p_i, p_j) := \underbrace{\frac{1}{8} (\boldsymbol{\mu}_i - \boldsymbol{\mu}_j)^\top \boldsymbol{\Sigma}^{-1} (\boldsymbol{\mu}_i - \boldsymbol{\mu}_j)}_{\text{Mahalanobis term}} + \underbrace{\frac{1}{2} \ln \left( \frac{\det(\boldsymbol{\Sigma})}{\sqrt{\det(\boldsymbol{\Sigma}_i) \det(\boldsymbol{\Sigma}_j)}} \right)}_{\text{Determinant term}},
\end{equation}

where $\boldsymbol{\Sigma}$ represents the average covariance matrix of the two distributions:
\begin{equation}
    \boldsymbol{\Sigma} := \frac{\boldsymbol{\Sigma}_i + \boldsymbol{\Sigma}_j}{2}.
\end{equation}

The \textit{Bhattacharyya coefficient} $\Bhat$ between two arms $a_i$ and $a_j$, which represents the geometric overlap between their underlying posterior distributions $p_i$ and $p_j$, is derived directly from the Bhattacharyya distance:
\begin{equation}
    \Bhat(a_i, a_j) \triangleq \exp\left( -D_B(p_i, p_j) \right).
\end{equation}
The coefficient ranges from $0$ to $1$, where $0$ implies no overlap, i.e.\ perfect separation, and $1$ implies identical distributions, indicating maximum uncertainty regarding the arms' ordering.

\subsection{Aggregate Uncertainty Measure}
The total uncertainty metric $\UQ$ is computed by averaging the Bhattacharyya coefficients over all pairwise combinations of arms between the first and second Pareto fronts.

Let $\ttpEmpFront{1}$ denote the set of arms in the estimated first Pareto front, based on posterior means, and $\ttpEmpFront{2}$ denote the set of arms in the estimated second Pareto front, also based on posterior means. The uncertainty measure is calculated as:
\begin{equation}
    \UQ\left(\ttpEmpFront{1},\ttpEmpFront{2}\right) = 
 \frac{1}{|\ttpEmpFront{1}| |\ttpEmpFront{2}|} \sum_{a_i \in \ttpEmpFront{1}} \sum_{a_j \in \ttpEmpFront{2}} \Bhat(a_i, a_j)
\end{equation}
This aggregation captures the average overlap between the currently believed best set, i.e.\ the first front, and the challenger set, i.e.\ the second front. A high value indicates that the posterior masses of the optimal and suboptimal arms are heavily intertwined, requiring further exploration to resolve the decision boundary. Conversely, a low value indicates that the probability masses have separated, signifying low uncertainty regarding the current estimate of the Pareto set. The source code containing the full implementation of the proposed uncertainty quantification metric can be accessed in our GitHub repository\footnote{\url{https://github.com/LennertSaerens/TTPFTS}}.

\subsection{Comparison of Uncertainty Quantification Metrics}
\label{sec:uq_comparison}
In the main text, we introduce an uncertainty quantification metric based on the Bhattacharyya coefficient to serve as a ground-truth-independent proxy for algorithmic convergence. To contextualize and justify this design choice, this section presents a comparative analysis of our proposed metric against two standard information-theoretic alternatives: Symmetric Kullback-Leibler (KL) divergence and Posterior Entropy.

All three metrics evaluate the posterior distributions of the arms estimated to be on the top-two Pareto fronts (denoted as $\ttpEmpFront{1}$ and $\ttpEmpFront{2}$). The evaluation was conducted across the eight synthetic benchmarking environments. For each environment, we executed 100 independent runs of the TTPFTS algorithm (using an uninformative $t$-distributed prior) for $H(\nu)$ arm pulls, saving and evaluating the posterior states at predefined intervals of 100 arm pulls.

\begin{itemize}
    \item \textbf{Bhattacharyya Coefficient (BC):} Our proposed metric computes the average pairwise overlap between distributions of arms in $\ttpEmpFront{1}$ and arms in $\ttpEmpFront{2}$. It is bounded in $[0, 1]$, where $0$ indicates no overlap and $1$ indicates identical distributions. \cite{Bhattacharyya1943}
    \item \textbf{Symmetric KL-Divergence:} This metric computes the average pairwise symmetric KL-divergence, $\frac{1}{2}\left(D_{\text{KL}}(p \parallel q) + D_{\text{KL}}(q \parallel p)\right)$, between the same $\ttpEmpFront{1} \times \ttpEmpFront{2}$ arm pairs. It is unbounded and grows as the posteriors separate. \cite{Kullback1951Information}
    \item \textbf{Posterior Entropy:} This metric calculates the average differential entropy over all arms in the union $\ttpEmpFront{1} \cup \ttpEmpFront{2}$. Unlike BC and KL-divergence, it measures the overall posterior uncertainty rather than the pairwise separation between the top-two fronts. \cite{MacKay1992ActiveData}
\end{itemize}

To determine how well each metric acts as a proxy for actual performance, we calculated the Pearson correlation coefficient ($r$) between the metric's time series and the ground-truth Jaccard similarity across the experimental runs. A reliable UQ metric should exhibit a strong, consistent correlation with the Jaccard score across diverse environments. The results of this analysis are visualized in Figure~\ref{fig:uq_comparison_combined}.

\begin{figure}[ht]
    \centering
    \includegraphics[width=1\linewidth]{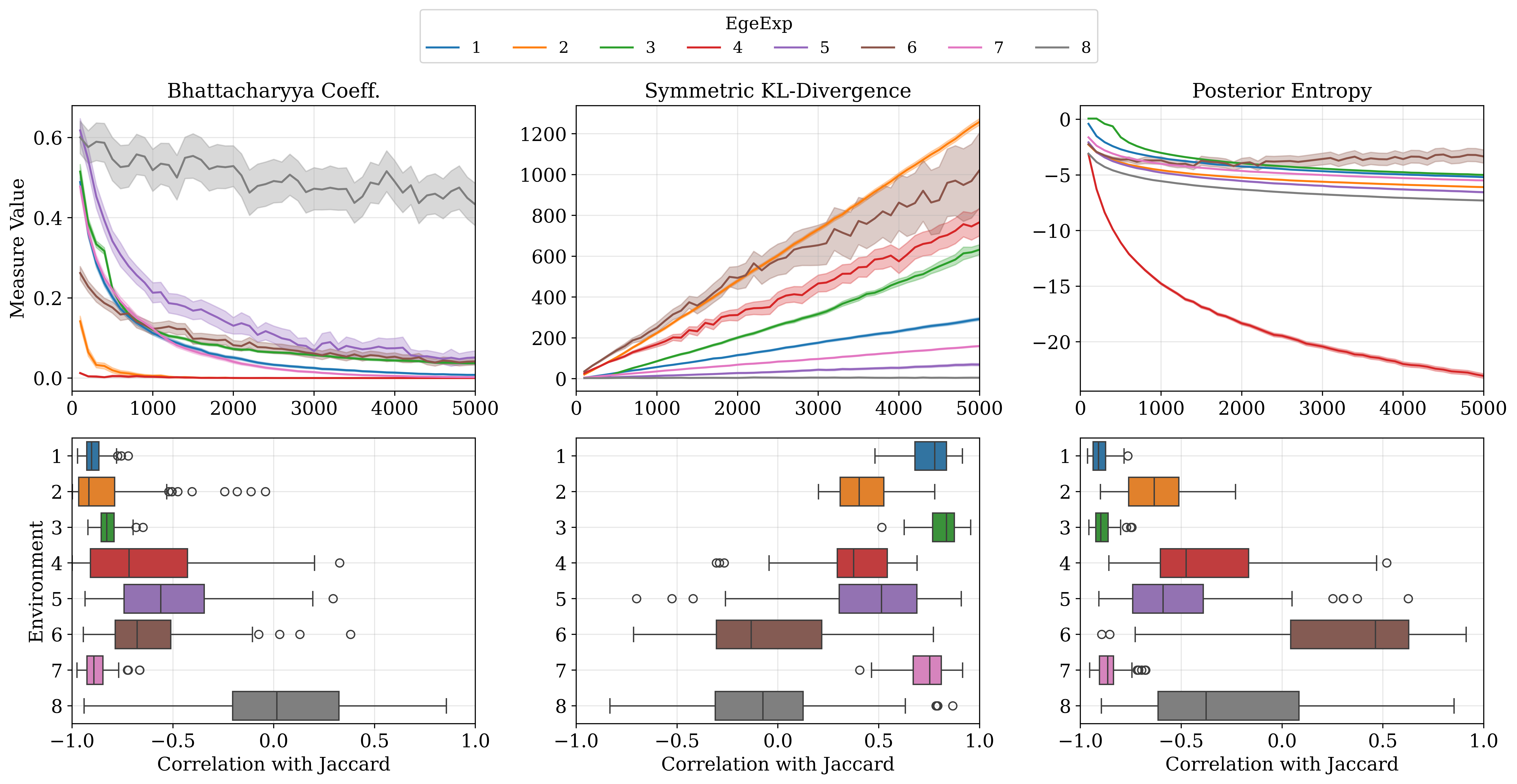}
    \caption{Comparison of alternative uncertainty quantification metrics. \textbf{Top row:} Evolution of the Symmetric KL-Divergence and Posterior Entropy over the normalized budget (mean $\pm$ 95\% CI over 100 runs). \textbf{Bottom row:} Distribution of per-run Pearson correlation coefficients between each metric and the ground-truth Jaccard similarity, grouped by environment.}
    \label{fig:uq_comparison_combined}
\end{figure}

The Bhattacharyya coefficient proves to be the strongest and most consistent tracker of convergence. It achieves the highest median absolute correlation ($r = -0.80$) and maintains the expected negative correlation, meaning that as overlap decreases, identification accuracy increases, across seven of the eight environments. In contrast, Symmetric KL-Divergence is the weakest overall predictor (median $r = +0.55$). Posterior Entropy demonstrates a strong median correlation ($r = -0.70$) but exhibits high variance across different environments (standard deviation of $0.48$).

The limitations of Posterior Entropy and KL-Divergence become apparent in specific edge cases, most notably in the Environment 6 environment, where all arms are Pareto-optimal. Because entropy measures absolute posterior spread rather than the separation between fronts, it naturally decreases as the algorithm pulls arms and variances shrink, even if the fronts have not been accurately separated. This structural flaw leads to a problematic sign flip in Environment 6 ($r = +0.46$). KL-Divergence similarly struggles in this environment and in Environment 8, yielding near-zero or incorrect-sign correlations.

\subsection{Additional Analysis Environment 8}
\label{sec:sup:UQ_ENV8}
In this section we present an additional analysis into the results observed for Environment 8 in Section 6 of the main text. We hypothesize that rather than failing, the UQ metric functions exactly as intended: faithfully signaling that the original horizon of $H(\nu)$ arm pulls was insufficient to resolve the geometrically vanishing sub-optimality gaps unique to this environment. To empirically validate this, we conducted two additional experiments:

\begin{enumerate}
    \item First, we extended the sampling budget for EgeExp8 to 500,000 arm pulls, logging metrics every 500 steps across 100 runs.
    \item Second, we performed an ablation study sweeping the variance of the reward distributions from high to low while maintaining the original 5,000-pull budget.
\end{enumerate}

\begin{figure}[ht]
    \centering

    \begin{subfigure}{0.33\linewidth}
        \centering
        \includegraphics[width=\linewidth]{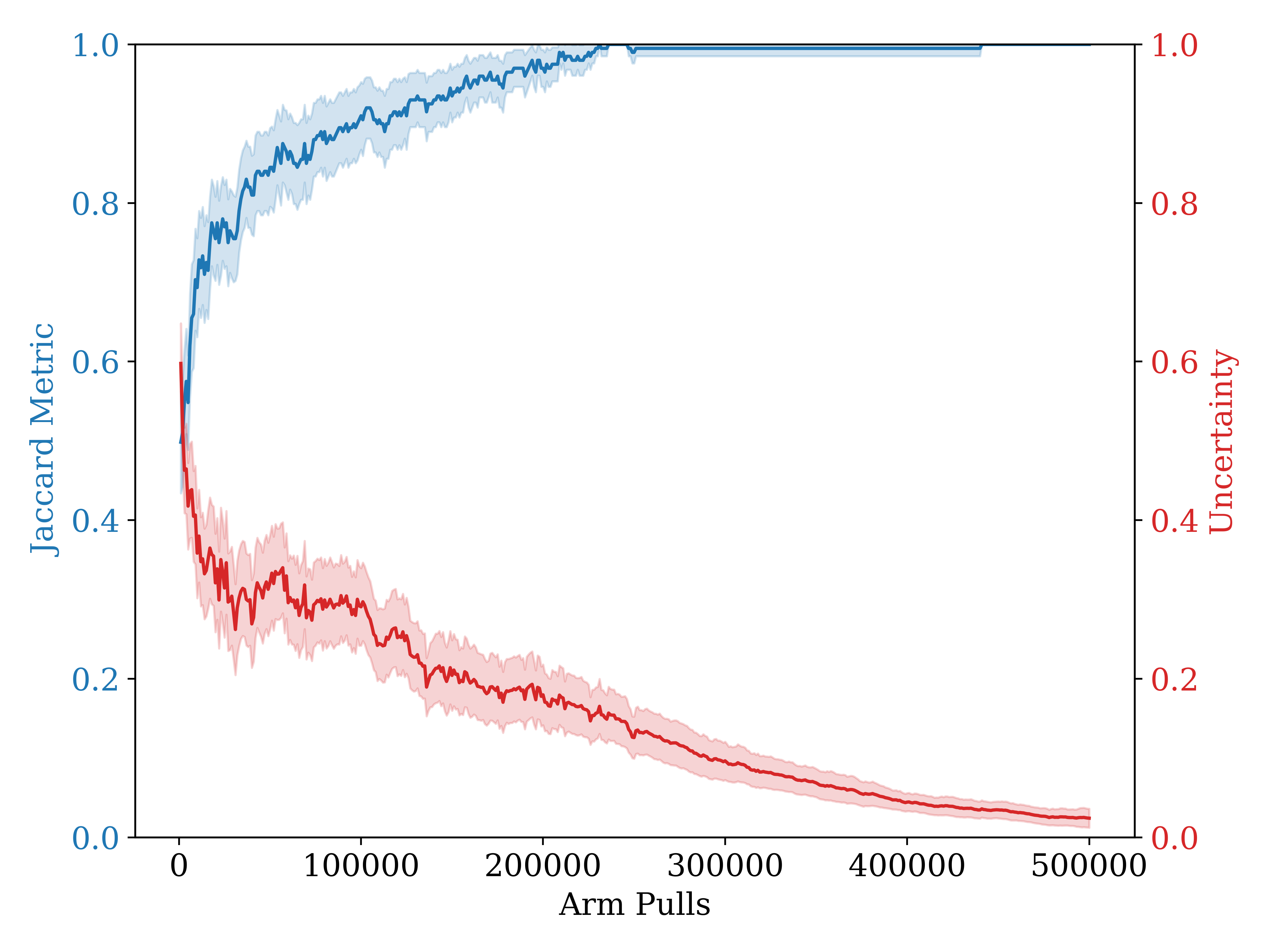}
        \caption{Environment 8 evolution of the Jaccard and Uncertainty metrics.}
        \label{fig:UQ_ENV8_EXTENDED}
    \end{subfigure}\hfill
    \begin{subfigure}{0.66\linewidth}
        \centering
        \includegraphics[width=\linewidth]{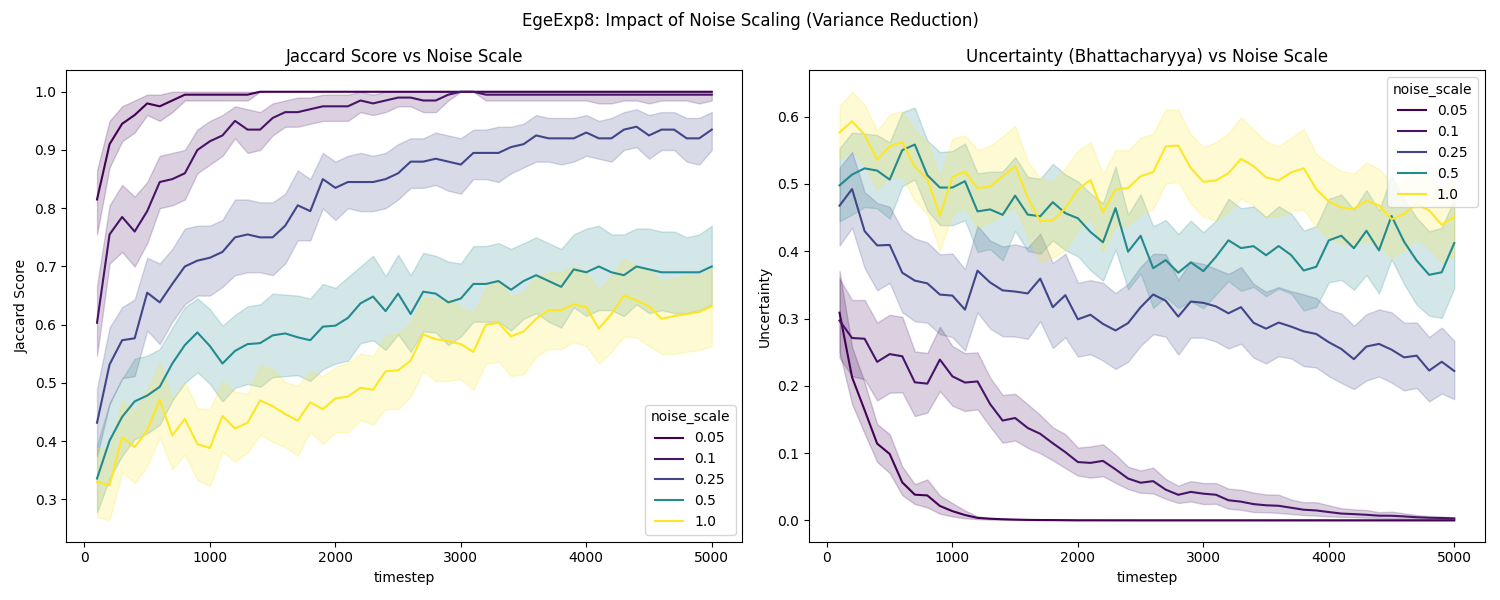}
        \caption{Environment 8 noise scaling.}
        \label{fig:UQ_ENV8_NOISESCALE}
    \end{subfigure}

    \caption{Additional analysis of Environment 8: (a) evolution of the Jaccard and uncertainty quantification metrics over an extended horizon of 500,000 arm pulls, and (b) correlation and performance across varying reward noise levels.}
    \label{fig:sup:UQ_ENV8}
\end{figure}

In both scenarios, once the algorithm is afforded either sufficient time or an adequate signal-to-noise ratio to overcome the vanishing gaps, it successfully identifies the single optimal arm. Moreover, as the Jaccard metric rises to 1.0, the UQ metric exhibits the expected strong correlation and decay. These findings demonstrate that the metric does not break down on hard problems. Instead, it serves as a robust, conservative indicator that the algorithm has not yet converged. This further validates the metric's reliability not only across diverse Pareto front geometries, as established in the original manuscript, but also across a spectrum of problem difficulties.

\end{document}